\documentclass[twoside,11pt]{article}

\usepackage{comment}
\usepackage{amsmath}
\usepackage{thmtools}
\usepackage{thm-restate}
\usepackage{enumitem}
\usepackage{jmlr2e}
\usepackage{tikz}

\usepackage{subfig}

\newcommand{\NN}{\mathbb{N}}
\newcommand{\RR}{\mathbb{R}}
\newcommand{\ZZ}{\mathbb{Z}}
\newcommand{\limn}{\lim_{n\rightarrow\infty}}
\newcommand{\cov}{\operatorname{Cov}}

\newcommand{\limD}{\lim_{D\rightarrow\infty}}
\newcommand{\calN}{\mathcal{N}}
\newcommand{\eps}{\varepsilon}

\newcommand{\lmat}{\begin{bmatrix}}
\newcommand{\rmat}{\end{bmatrix}}

\defcitealias{matthews2018gaussian}{Matthews et~al., 2018}
\defcitealias{lee2017deep}{Lee et~al., 2017}

\makeatletter
\renewcommand*{\@opargbegintheorem}[3]{\trivlist
	\item[\hskip \labelsep{\bfseries #1\ #2}] \textbf{(#3)}\ \itshape}
\makeatother

\jmlrheading{0}{2020}{1-57}{01/20}{10/00}{meila00a}{Devanshu Agrawal, Theodore Papamarkou, Jacob Hinkle}

\ShortHeadings{Wide Neural Networks with Bottlenecks}{Agrawal et al.}
\firstpageno{1}

\begin{document}

\title{Wide Neural Networks with Bottlenecks are Deep Gaussian Processes}

%\author{Anonymous}
%\begin{comment}
\author{\name Devanshu Agrawal \email dagrawa2@vols.utk.edu \\
       \addr The Bredesen Center\\
       University of Tennessee\\
       Knoxville, TN 37996-3394, USA
       \AND
       \name Theodore Papamarkou \email papamarkout@ornl.gov \\
       \addr Computational Sciences and Engineering Division\\
       Oak Ridge National Lab\\
       Oak Ridge, TN 37830-8050, USA
       \AND
       \name Jacob Hinkle \email hinklejd@ornl.gov \\
       \addr Computational Sciences and Engineering Division\\
       Oak Ridge National Lab\\
       Oak Ridge, TN 37830-8050, USA}
%\end{comment}

% \editor{To be selected}

\maketitle

\begin{abstract}%   <- trailing '%' for backward compatibility of .sty file
There has recently been much work on the ``wide limit'' of neural networks, where Bayesian neural networks (BNNs) are shown to converge to a Gaussian process (GP) as all hidden layers are sent to infinite width. 
However, these results do not apply to architectures that require one or more of the hidden layers to remain narrow. 
In this paper, we consider the wide limit of BNNs where some hidden layers, called ``bottlenecks'', are held at finite width. 
The result is a composition of GPs that we term a ``bottleneck neural network Gaussian process'' (bottleneck NNGP). 
Although intuitive, the subtlety of the proof is in showing that the wide limit of a composition of networks is in fact the composition of the limiting GPs. 
We also analyze theoretically a single-bottleneck NNGP, finding  that the bottleneck 
induces dependence between the outputs of a multi-output network that persists through extreme post-bottleneck depths, and 
prevents the kernel of the network from losing discriminative power at extreme post-bottleneck depths. 
%We demonstrate the potential utility of bottleneck NNGPs empirically on multiple datasets, showing that bottlenecks boost model likelihood and allow the network to operate at greater depths.
\end{abstract}


%\tableofcontents

%\listoftheorems

\section{Introduction}

Deep neural networks have found great empirical success, achieving state-of-the-art performance on a variety of tasks such as those in computer vision and natural language understanding~\citep{krizhevsky2012imagenet, antipov2015learned, liang2017text}.
There is considerable interest in understanding the theoretical aspects of deep neural networks both to establish guarantees on the behavior of 
these models on certain classes of problems as well as to guide architecture design and optimization.
One avenue of pursuit in this endeavor leads to the study of Bayesian neural networks (BNNs), where the parameters of the network are random variables following some probability distributions.
BNNs thus bring the formalism and machinery of probability theory to bear on neural networks.

It is a foundational result that a BNN converges to a Gaussian process (GP) in the ``wide limit''---i.e., as the widths of all hidden layers are sent to infinity while the prior distributions on weights are sharpened accordingly~\citep{neal1996priors}. 
The resulting GP is called a ``neural network Gaussian process'' (NNGP). 
Although NNGP limits have been derived from various BNN architectures, they cannot be obtained from architectures requiring some hidden layers to remain narrow, such as certain autoencoders. 
It seems intuitive that the wide limit of a BNN with some hidden layers restricted to finite-width ``bottlenecks'' is a composition of NNGPs,
but until now this claim has not been proven.
Such a composition of GPs is called a ``deep Gaussian process'' (DGP) in the literature~\citep{damianou2013deep}.
Although DGPs were inspired by the compositional structure of deep neural networks, their connection to BNNs has not been established formally.

In this paper, we give a formal proof of the convergence of BNNs with bottleneck layers to a DGP in the wide limit, 
where the DGP is a composition of NNGPs. 
In doing so, we unify the two major approaches to making GPs ``deeper''---NNGPs and DGPs, 
thus allowing NNGPs to be examined in the DGP framework. 
We will refer to the limiting DGP as a ``bottleneck NNGP''.
Even though the result is intuitive, the proof is nontrivial as it requires us to formally justify that the limit of a composition of BNNs equals the composition of the limiting NNGPs.
%This exchangeability of limit and composition holds when the post-bottleneck component of a bottleneck BNN converges with sufficient uniformity; showing this is the main work of this paper.

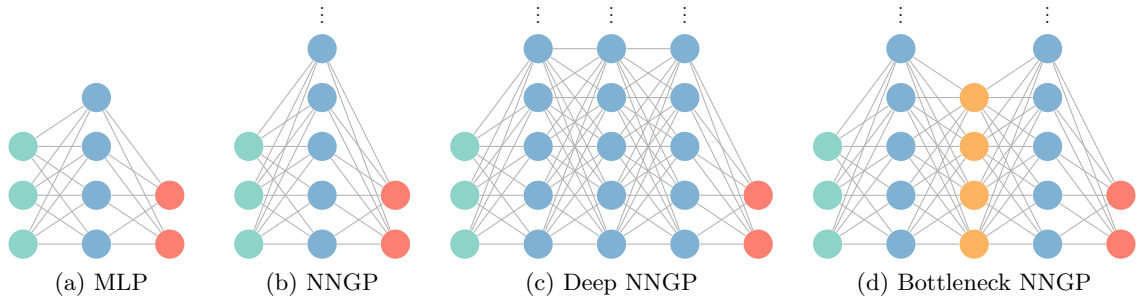
\begin{figure}
\centering
    \def\figscale{.65}
\def\numinputs{3}
\def\numoutputs{2}
\def\layersep{1.5cm}

% first pass
%\def\inputcolor{green!70!blue}
%\def\hiddencolor{blue!70!black}
%\def\infhiddencolor{blue!50}
%\def\bottleneckcolor{blue!70!black}
%\def\outputcolor{red!80!black}

% colorbrewer Paired
%\def\inputcolor{rgb,255: red,51;green,160;blue,44}
%\def\hiddencolor{rgb,255: red,31;green,120;blue,180}
%\def\infhiddencolor{rgb,255: red,166; green,206; blue,227}
%\def\bottleneckcolor{rgb,255: red,31;green,120;blue,180}
%\def\outputcolor{rgb,255: red,227;green,26;blue,28}

% colorbrewer Set3
%141,211,199
%128,177,211
%253,180,98
%251,128,114
\def\inputcolor{rgb,255: red,141;green,211;blue,199}
\def\hiddencolor{rgb,255: red,128;green,177;blue,211}
\def\infhiddencolor{rgb,255: red,128;green,177;blue,211}
\def\bottleneckcolor{rgb,255: red,253;green,180;blue,98}
\def\outputcolor{rgb,255: red,251;green,128;blue,114}

\def\edgecolor{black!30}
% \subfloat[Shallow MLP]{
\subfloat[MLP]{
    \begin{tikzpicture}[shorten >=0pt,-,draw=\edgecolor, node
        distance=\layersep, scale=\figscale, every node/.style={transform shape}]
    \tikzstyle{every pin edge}=[<-,shorten <=0pt]
    \tikzstyle{neuron}=[circle,fill=black!25,minimum size=17pt,inner sep=0pt]
    \tikzstyle{input neuron}=[neuron, fill=\inputcolor]
    \tikzstyle{output neuron}=[neuron, fill=\outputcolor]
    \tikzstyle{hidden neuron}=[neuron, fill=\hiddencolor]%\bottleneckcolor];
    \tikzstyle{annot} = [text width=4em, text centered]

    % Draw the input layer nodes
    \foreach \name / \y in {1,...,\numinputs}
    % This is the same as writing \foreach \name / \y in {1/1,2/2,3/3,4/4}
        \node[input neuron] (I-\name) at (0,-\y) {};

    % Draw the hidden layer nodes
    \foreach \name / \y in {1,...,4}
        \path[yshift=0.5cm]
        node[hidden neuron] (H-\name) at (\layersep,-\y.5 cm + 1cm) {};

    % Draw the output layer nodes
        \foreach \name / \y in {1,...,\numoutputs}
        \node[output neuron] (O-\name) at (2*\layersep, -\y.5cm - 0.5cm) {};

    % Connect every node in the input layer with every node in the
    % hidden layer.
    \foreach \source in {1,...,3}
        \foreach \dest in {1,...,4}
            \path (I-\source) edge (H-\dest);

    % Connect every node in the hidden layer with the output layer
    \foreach \source in {1,...,4}
    \foreach \dest in {1,...,\numoutputs}
            \path (H-\source) edge (O-\dest);

    % Annotate the layers
    %\node[annot,above of=H-1, node distance=1cm] (hl) {Hidden layer};
    %\node[annot,left of=hl] {Input layer};
    %\node[annot,right of=hl] {Output layer};
\end{tikzpicture}
\label{fig-seq-nn-mlp}
} \hfill
% \subfloat[Shallow NNGP~\citep{neal1996priors}]{
\subfloat[NNGP]{
    \begin{tikzpicture}[shorten >=0pt,-,draw=\edgecolor, node
        distance=\layersep, scale=\figscale, every node/.style={transform shape}]
    \tikzstyle{every pin edge}=[<-,shorten <=0pt]
    \tikzstyle{neuron}=[circle,fill=black!25,minimum size=17pt,inner sep=0pt]
    \tikzstyle{input neuron}=[neuron, fill=\inputcolor]
    \tikzstyle{output neuron}=[neuron, fill=\outputcolor]
    \tikzstyle{hidden neuron}=[neuron, fill=\infhiddencolor]
    \tikzstyle{annot} = [text width=4em, text centered]

    % Draw the input layer nodes
    \foreach \name / \y in {1,...,3}
    % This is the same as writing \foreach \name / \y in {1/1,2/2,3/3,4/4}
        \node[input neuron] (I-\name) at (0,-\y) {};

    % Draw the hidden layer nodes
    \foreach \name / \y in {1,...,5}
        \path[yshift=0.5cm]
        node[hidden neuron] (H-\name) at (\layersep,-\y cm + 1.5cm) {};
    \node[annot, above of=H-1,yshift=-0.7cm] {$\mathbf{\vdots}$};

    % Draw the output layer nodes
    \foreach \name / \y in {1,...,\numoutputs}
        \node[output neuron] (O-\name) at (2*\layersep, -\y.5cm - 0.5cm) {};

    % Connect every node in the input layer with every node in the
    % hidden layer.
    \foreach \source in {1,...,3}
        \foreach \dest in {1,...,5}
            \path (I-\source) edge (H-\dest);

    % Connect every node in the hidden layer with the output layer
    \foreach \source in {1,...,5}
        \foreach \dest in {1,...,\numoutputs}
            \path (H-\source) edge (O-\dest);

    % Annotate the layers
    %\node[annot,above of=H-1, node distance=1cm] (hl) {Hidden layer};
    %\node[annot,left of=hl] {Input layer};
    %\node[annot,right of=hl] {Output layer};
\end{tikzpicture}
\label{fig-seq-nn-nngp}
} \hfill
% \subfloat[Deep NNGP~\citep{lee2017deep,matthews2018gaussian}]{
\subfloat[Deep NNGP]{
    \begin{tikzpicture}[shorten >=0pt,-,draw=\edgecolor, node
        distance=\layersep, scale=\figscale, every node/.style={transform shape}]
    \tikzstyle{every pin edge}=[<-,shorten <=0pt]
    \tikzstyle{neuron}=[circle,fill=black!25,minimum size=17pt,inner sep=0pt]
    \tikzstyle{input neuron}=[neuron, fill=\inputcolor]
    \tikzstyle{output neuron}=[neuron, fill=\outputcolor]
    \tikzstyle{hidden neuron}=[neuron, fill=\infhiddencolor]
    \tikzstyle{annot} = [text width=4em, text centered]

    % Draw the input layer nodes
    \foreach \name / \y in {1,...,3}
    % This is the same as writing \foreach \name / \y in {1/1,2/2,3/3,4/4}
        \node[input neuron] (I-\name) at (0,-\y) {};

    % Draw the hidden layer nodes
    \foreach \name / \y in {1,...,5}
        \path[yshift=0.5cm]
        node[hidden neuron] (H-\name) at (\layersep,-\y cm + 1.5cm) {};
    \node[annot, above of=H-1,yshift=-0.7cm] {$\mathbf{\vdots}$};

    % Draw the hidden layer nodes
    \foreach \name / \y in {1,...,5}
        \path[yshift=0.5cm]
        node[hidden neuron] (H2-\name) at (\layersep*2,-\y cm + 1.5cm) {};
    \node[annot, above of=H2-1,yshift=-0.7cm] {$\mathbf{\vdots}$};

    % Draw the hidden layer nodes
    \foreach \name / \y in {1,...,5}
        \path[yshift=0.5cm]
        node[hidden neuron] (H3-\name) at (\layersep*3,-\y cm + 1.5cm) {};
    \node[annot, above of=H3-1,yshift=-0.7cm] {$\mathbf{\vdots}$};

    % Draw the output layer nodes
    \foreach \name / \y in {1,...,\numoutputs}
        \node[output neuron] (O-\name) at (4*\layersep, -\y.5cm - 0.5cm) {};

    % Connect every node in the input layer with every node in the
    % hidden layer.
    \foreach \source in {1,...,3}
        \foreach \dest in {1,...,5}
            \path (I-\source) edge (H-\dest);

    \foreach \source in {1,...,5}
        \foreach \dest in {1,...,5}
            \path (H-\source) edge (H2-\dest);

    \foreach \source in {1,...,5}
        \foreach \dest in {1,...,5}
            \path (H2-\source) edge (H3-\dest);

    % Connect every node in the hidden layer with the output layer
    \foreach \source in {1,...,5}
        \foreach \dest in {1,...,\numoutputs}
            \path (H3-\source) edge (O-\dest);

    % Annotate the layers
    %\node[annot,above of=H-1, node distance=1cm] (hl) {Hidden layer};
    %\node[annot,left of=hl] {Input layer};
    %\node[annot,right of=hl] {Output layer};
\end{tikzpicture}
\label{fig-seq-nn-deep-nngp}
} \hfill
% \subfloat[Bottleneck NNGP (ours)]{
\subfloat[Bottleneck NNGP]{
    \begin{tikzpicture}[shorten >=0pt,-,draw=\edgecolor, node
        distance=\layersep, scale=\figscale, every node/.style={transform shape}]
    \tikzstyle{every pin edge}=[<-,shorten <=0pt]
    \tikzstyle{neuron}=[circle,fill=black!25,minimum size=17pt,inner sep=0pt]
    \tikzstyle{input neuron}=[neuron, fill=\inputcolor]
    \tikzstyle{output neuron}=[neuron, fill=\outputcolor]
    \tikzstyle{hidden neuron}=[neuron, fill=\infhiddencolor]
    \tikzstyle{bottleneck neuron}=[neuron, fill=\bottleneckcolor]
    \tikzstyle{annot} = [text width=4em, text centered]

    % Draw the input layer nodes
    \foreach \name / \y in {1,...,3}
    % This is the same as writing \foreach \name / \y in {1/1,2/2,3/3,4/4}
        \node[input neuron] (I-\name) at (0,-\y cm) {};

    % Draw the hidden layer nodes
    \foreach \name / \y in {1,...,5}
        \path[yshift=0.5cm]
        node[hidden neuron] (H-\name) at (\layersep,-\y cm + 1.5cm) {};
    \node[annot, above of=H-1,yshift=-0.7cm] {$\mathbf{\vdots}$};

    % Draw the hidden layer nodes
    \foreach \name / \y in {1,...,4}
        \path[yshift=0.5cm]
        node[bottleneck neuron] (H2-\name) at (\layersep*2,-\y cm + 0.5cm) {};
    %\node[annot, above of=H2-1,yshift=-0.7cm] {$\mathbf{\vdots}$};

    % Draw the hidden layer nodes
    \foreach \name / \y in {1,...,5}
        \path[yshift=0.5cm]
        node[hidden neuron] (H3-\name) at (\layersep*3,-\y cm + 1.5cm) {};
    \node[annot, above of=H3-1,yshift=-0.7cm] {$\mathbf{\vdots}$};

    % Draw the output layer nodes
    \foreach \name / \y in {1,...,\numoutputs}
        \node[output neuron] (O-\name) at (4*\layersep, -\y.5cm - 0.5cm) {};

    % Connect every node in the input layer with every node in the
    % hidden layer.
    \foreach \source in {1,...,3}
        \foreach \dest in {1,...,5}
            \path (I-\source) edge (H-\dest);

    \foreach \source in {1,...,5}
        \foreach \dest in {1,...,4}
            \path (H-\source) edge (H2-\dest);

    \foreach \source in {1,...,4}
        \foreach \dest in {1,...,5}
            \path (H2-\source) edge (H3-\dest);

    % Connect every node in the hidden layer with the output layer
    \foreach \source in {1,...,5}
        \foreach \dest in {1,...,\numoutputs}
            \path (H3-\source) edge (O-\dest);

    % Annotate the layers
    %\node[annot,above of=H-1, node distance=1cm] (hl) {Hidden layer};
    %\node[annot,left of=hl] {Input layer};
    %\node[annot,right of=hl] {Output layer};
\end{tikzpicture}
\label{fig-seq-nn-bottleneck-nngp}
}
\caption{\label{fig-seq-nn} %
Depiction of various NNGP architectures with three predictors (green nodes) and two response variables (red nodes). 
Blue nodes indicate hidden layers, with ellipses indicating layers that increase in width toward infinity.
In our bottleneck NNGP model, one or more bottlenecks (finite-width hidden layers with orange nodes) are surrounded by infinite-width hidden layers. 
In the historical development of NNGP architectures,
MLPs (sub-figure \protect\subref{fig-seq-nn-mlp}) have been succeeded by
shallow NNGPs (sub-figure \protect\subref{fig-seq-nn-nngp}, see \citet{neal1996priors}), 
which in turn have been succeeded by 
deep NNGPs (sub-figure \protect\subref{fig-seq-nn-deep-nngp}, see \citetalias{lee2017deep,matthews2018gaussian}). 
In our paper, we propose bottleneck NNGPs (sub-figure \protect\subref{fig-seq-nn-bottleneck-nngp}).
}
\end{figure}

In the bottleneck NNGP limit, we consider a sequence of BNNs all having the same architecture except that some hidden layers are growing to infinite width (Fig.~\ref{fig-seq-nn}). 
We call the hidden layers held to finite widths ``bottleneck layers'' or simply ``bottlenecks'', 
and we call each network in the sequence a ``bottleneck BNN''. 
We use the term ``component'' to refer to any subnetwork that is either (1) between the input layer and the first bottleneck layer, (2) between two bottleneck layers with no bottlenecks in 
between, or (3) between the last bottleneck layer and the output layer; 
each BNN is thus a composition of components, 
and each component maintains constant input and output dimensions with only its hidden layers growing in width over the sequence of networks.

We know that each sequence of corresponding components converges to an NNGP in the wide limit. 
It is therefore intuitive to expect that the sequence of bottleneck BNNs (each BNN being a composition of components) converges to the composition of NNGPs---i.e., a bottleneck NNGP. 
%The proof of convergence thus reduces to showing that
%the limit of a composition of components is the composition of the limits of components
%in the cases of BNNs.
However, this fact is not immediate, and care must be taken to verify that the limit procedure can be exchanged with the composition of components. 
In particular, we find that this exchangeability holds if each post-bottleneck component converges to an NNGP with a sufficient amount of uniformity with respect to its inputs.

We demonstrate the utility of bottleneck NNGPs and their link to no-bottleneck NNGPs empirically,
showing that restricting a hidden layer of an NNGP to a bottleneck can boost its model likelihood on three example datasets\footnote{Code for our simulations and experiments is available at
%\url{https://anonymized}.
\url{https://code.ornl.gov/d0a/bottleneck_nngp}.
}.
% by allowing it to operate at greater post-bottleneck depths.

We also characterize the effect of a bottleneck layer theoretically by analyzing an example multi-output single-bottleneck NNGP with rectified linear unit (ReLU) activation. 
We find that the bottleneck induces dependence between distinct response variables and derive a closed-form expression for the correlation between the squares (i.e., quadratic correlation) of the response variables. 
We show that in the deep post-bottleneck limit (infinitely many infinite-width hidden layers after the bottleneck), the quadratic correlation tends to $0$ when the network is in the ``disordered phase''---so that response variables decouple---but remains a nontrivial function of the inputs in the ``ordered phase''---so that information about the inputs can be recovered. 
We identify the prior variance of the network weights as the order parameter responsible for the phase transition.

Similarly, in the deep post-bottleneck limit, we obtain a closed-form expression for the quadratic correlation of outputs of a single response variable given two inputs. 
We find that the quadratic correlation is 100\% in the ``disordered phase''---so that the network has lost all discriminative power at infinite depth---but is surprisingly a nontrivial function of the inputs in the ``ordered phase''.
This behavior  in the ordered phase stands in stark contrast to that of no-bottleneck NNGPs and indicates that bottleneck layers are essential for a very deep network to maintain discriminative power.

\section{Preliminaries}

In this section, we review prior work on DGPs and NNGPs to contextualize and motivate the bottleneck NNGP model.
We also review the main theorem of \citetalias{matthews2018gaussian} in Sec.~\ref{subsection-no-bottleneck}, introducing notation that will be essential to stating our main result in Sec.~\ref{subsection-bottleneck-nngp-theorem}.

\subsection{Deep Gaussian processes}
\label{subsection-dgp}

Compositions of GPs are known as deep Gaussian processes (DGPs) in the literature and were originally motivated by the success of deep neural networks and the hope to obtain similar success on small data sets where Bayesian methods generally shine~\citep{damianou2013deep}.
DGPs have indeed been shown to outperform shallow GPs on a variety of regression and classification tasks~\citep{damianou2013deep, salimbeni2017doubly}.
\citet{damianou2013deep} implement DGPs in a sparse variational inducing points framework based on the work of 
\citet{titsias2009variational} in order to simplify the composition of GPs to a set of separate but coupled GPs, but their
implementation is restricted to small data sets with only a few hundred entries.
Much of the DGP literature has therefore been dedicated to developing more efficient and scalable implementations~\citep{hensman2014nested, dai2015variational, bui2016deep, wang2016sequential, salimbeni2017doubly}.
\citet{salimbeni2017doubly} in particular show that DGPs can be put into a stochastic variational framework as in \citet{hensman2013gaussian}, allowing the models to be applied to much larger data sets.

DGPs offer the additional advantage that they can capture correlation between multiple outputs~\citep{wang2016sequential}.
In contrast, the outputs of a shallow multi-output GP are by default independent, which can limit predictive performance for multi-output problems.
Although methods have been proposed to model correlation in a shallow multi-output GP, such as through linear mixing of latent outputs~\citep{bonilla2008multi, alvarez2009sparse}, DGPs capture correlation naturally through shared feature representations in the latent ``bottleneck space''---similar to the approach taken with multi-task deep neural networks~\citep{ruder2017overview}.
DGPs have therefore been applied to problems that can benefit from modeling correlation between multiple outputs, such as multi-task regression~\citep{alaa2017deep} and tasks involving partially observed multivariate outputs---i.e., missing values~\citep{wang2016sequential}.

The mechanism by which the outputs of a DGP are made dependent predates the DGP model itself, as it was first introduced in the context of the Bayesian Gaussian process latent variable model (GP-LVM), 
where a Gaussian prior is placed on the latent inputs of a GP~\citep{lawrence2004gaussian, titsias2010bayesian}; 
for regression, each input is concatenated with such a latent random variable before it is fed into the GP~\citep{dutordoir2018gaussian}. 
The outputs then become dependent through their dependence on the common set of latent random variable inputs, 
which are analogous to the bottleneck activations of a DGP. 
However, it still remains quantitatively unclear how the introduction of a bottleneck layer---or in the case of an NNGP, the restriction of a layer to finite width---induces correlation between multiple outputs under the prior and how this translates to correlation under the posterior. 
Moreover, although this mechanism is well-established for the GP-LVM and DGP models, it is conspicuously absent in the NNGP literature and thus its implications for NNGPs are not fully understood.

The DGP prior has been studied by
\citet{lu2019interpretable}, who show that for a single-bottleneck DGP with a single response variable, the prior has heavy tails, in contrast to shallow GPs.
Their calculation of the prior kurtosis is similar to that of the quadratic correlation between distinct response variables of a multi-output DGP, but this connection is not discussed.
Moreover, they only consider a bottleneck of width one and primarily focus on stationary kernels that do not arise from NNGP limits with common activation functions.

There is considerable interest in understanding the ``deep limit'' of DGPs---i.e. when arbitrarily many GPs are composed together. 
\citet{duvenaud2014avoiding} and \citet{dunlop2018deep} show that DGPs with a certain class of kernels have trivial, pathological, or convergent deep limits, 
meaning that increasing the depth of a DGP beyond some point is either detrimental to performance or diminishingly beneficial. 
However, they do not consider NNGP kernels and thus do not analyze deep limits of architectures with both bottlenecks and infinitely many infinite-width hidden layers.

Although DGPs were inspired by deep neural networks, there is little literature concretely establishing their connection.
\citet{duvenaud2014avoiding} discuss the connection between DGPs and neural networks at a high level to motivate studying the deep limit of DGPs with radial basis function (RBF) kernels, but the implication for neural networks is not treated formally. 
\citet{gal2015dropout} consider a DGP where the kernel of each GP layer is an integral as in \citet{williams1997computing}.
They show that a Monte Carlo estimation of the kernels leads to a BNN approximation of the DGP, where the width of a hidden layer corresponds to the size of the Monte Carlo sample.
However, they do not formally verify convergence of a BNN to a DGP in the limit of infinite width.
Moreover, their bottleneck layers have no activation function and are not scaled to allow an NNGP to be recovered as the bottlenecks are sent to infinite width.

\subsection{Wide neural networks as GPs}

A foundational result in the study of BNNs came when \citet{neal1996priors} showed that a BNN with one hidden layer converges to a GP in the ``wide limit''---i.e., as the number of hidden neurons is sent to infinity.
Shortly after, \citet{williams1997computing} derived analytic expressions for the kernel of the GPs corresponding to neural networks with sigmoidal and Gaussian hidden units.
These works connected neural networks to the world of Bayesian nonparametrics and kernel methods and thus offered a new perspective to interrogate and probe the behavior of neural networks.
In particular, while training neural networks is challenging since it requires the optimization of highly non-convex objective functions, GPs are nonparametric models that admit exact Bayesian inference, where the predictive posterior distribution can be written in  closed form~\citep{rasmussen2006gaussian}.
%Even though GP inference (without resorting to approximations) becomes infeasible on data sets with more than a few thousand entries, it 
%remains a valuable theoretical tool since it enables us to in principle write down the optimal neural network fitted to a data set (this is of course assuming 
%that the network parameters are random variables and the number of hidden neurons is in the wide limit regime).

Since the works of \citet{neal1996priors, williams1997computing}, new insights into BNNs have steadily emerged.
\citet{cho2009kernel} interpreted a BNN as a feature embedding map and derived the equations for the propagation of a kernel through the layers of a deep neural network with rectified polynomial unit activations.
Subsequent works built upon these findings to elucidate key theoretical aspects of neural networks including expressivity~\citep{poole2016exponential}, generalization power~\citep{hazan2015steps}, initialization~\citep{daniely2016toward}, and trainability~\citep{schoenholz2016deep}.
More recently, the original result by \citet{neal1996priors} has been extended to deep architectures by showing that a deep BNN 
converges to a GP as the widths of all hidden layers are simultaneously sent to infinity~\citepalias{lee2017deep, matthews2018gaussian}.
We refer to GPs that arise from such a limit as ``neural network Gaussian processes'' (NNGPs).
As the work of \citetalias{matthews2018gaussian} illustrates, this extension is nontrivial; the proof by \citet{neal1996priors} relies on the Central Limit Theorem, 
but the assumption of independent and identically distributed (IID) random variables necessary for the Central Limit Theorem does not hold for deep 
architectures.
The proof by \citetalias{matthews2018gaussian} for deep architectures is instead based on a more exotic central limit theorem as given in \citet{blum1958central}.

Since the extension of the NNGP limit to deep architectures, there have been a number of works establishing and analyzing analogous wide-limit results for more modern neural network architectures that are used in practice today.
These include convolutional neural networks~\citep{garriga-alonso2019deep, novak2019bayesian}, weight-tied autoencoders~\citep{li2018on}, and most generally any network that can be represented as a ``tensor program''---including recurrent neural networks and attention networks among others~\citep{yang2019scaling}. 
Alongside these works, new insights into the trainability and generalization power of neural networks have continued to emerge, based on the tractable learning dynamics of neural networks in the wide limit~\citep{jacot2018neural, lee2019wide, arora2019exact}.

One application of the NNGP limit that is of particular note is that it can make the analysis of the ``deep limit''---i.e., as the number of hidden layers is sent to infinity---tractable~\citep{poole2016exponential, schoenholz2016deep, yang2017mean, lee2017deep}.
%The covariance function---or kernel---of an NNGP is defined recursively over depth and can be viewed in terms of an iterative map or as a discrete dynamical system.
\citet{poole2016exponential} and \citet{schoenholz2016deep} show that the correlation between two inputs transformed through an NNGP with sigmoidal activation function has a fixed point at 100\% in the deep limit that transitions from stable to unstable (i.e., ordered to chaotic) when the variances of the Gaussian weights and biases cross a certain phase boundary; 
the network is shown to be highly expressive in the chaotic phase and optimally trainable near the phase boundary.
In contrast, for networks with rectified linear unit (ReLU) activations, %there is no order-to-chaos phase transition;
the correlation between transformed inputs has a stable fixed point at 100\% regardless of the weight variance, implying that an NNGP with ReLU activation has no discriminative power at infinite depth~\citep{lee2017deep}.

The works described above all consider BNNs in the wide limit, and thus the results and insights therein do not apply to neural network 
architectures that require one or more finite width or ``bottleneck'' layers.
One of the most important classes of neural networks that frequently require bottleneck layers is that of autoencoders~\citep{hinton2006reducing, kingma2014auto}.
Another example is neural networks with a word embedding layer, which is currently key to the successful application of neural networks to natural language understanding~\citep{mikolov2013distributed}.
Both word embedding layers and many autoencoder models aim to find dense feature representations and therefore depend on low-dimensional spaces.
Even for neural network architectures that are not directly meant for dense representation learning, it has still been argued and demonstrated that bottleneck layers perform data compression and therefore help to boost generalization power~\citep{tishby2015deep}.
Particularly for fully-connected architectures, which is what we consider in this work, it has been shown that the insertion of linear bottleneck layers between two linear ReLU layers boosts predictive performance by reducing sparsity and improving gradient flow~\citep{lin2015far}. 
This prompts the question:
How can insights based on very wide BNNs be generalized to networks in which one or more hidden layers are held fixed to a finite width?
The first step in addressing this question is to understand what happens if we let all but finitely many hidden layers of a BNN grow to infinite width.
We call these finite-width hidden layers ``bottleneck layers''.
It is already established that the component networks between consecutive bottleneck layers converge to GPs, and thus we intuitively expect a BNN with bottleneck layers to converge in the wide limit to a composition of GPs.
In Sec.~\ref{subsection-bottleneck-nngp-theorem}, we formally verify that this is the case.

\subsection{The no-bottleneck NNGP limit}
\label{subsection-no-bottleneck}

The bottleneck NNGP limit is a generalization of the (no-bottleneck) NNGP result proved by \citetalias{matthews2018gaussian}. 
Moreover, one component of our proof is verifying that BNNs converge in distribution uniformly on compact sets, and our approach to 
proving this closely follows the proof of \citetalias{matthews2018gaussian}. 
In this section, we state the NNGP limit result by \citetalias{matthews2018gaussian}, which also allows us to introduce key concepts and notation along the 
way.

We consider a traditional fully-connected network mapping $\RR^M$ to $\RR^L$ with $D$ hidden layers and nonlinearity $\phi$. 
Let $H_{\mu}$ be the width of the $\mu$-th hidden layer. 
The propagation of an input $x$ through the network is then governed by a recursion with initial step
\begin{equation} \label{eq-f1}
f_i^{(1)}(x) = b_i^{(1)} + \sum_{j=1}^{M} w_{ij}^{(1)} x_j,
\end{equation}
and for $\mu\in\{1,\ldots,D\}$,
\begin{align}
g_i^{(\mu)}(x) &= \phi[f_i^{(\mu)}(x)], \label{eq-gmu} \\
f_i^{(\mu+1)}(x) &= b_i^{(\mu+1)} + \sum_{j=1}^{H_{\mu}} w_{ij}^{(\mu+1)} g_j^{(\mu)} (x). \label{eq-fmu}
\end{align}
In Eq.~\eqref{eq-gmu}, $i$ ranges from $1$ to $H_{\mu}$. 
In Eq.~\eqref{eq-fmu}, $i$ ranges from $1$ to $H_{\mu+1}$ for $\mu=1,\ldots,D-1$, 
and from $1$ to $L$ for $\mu=D$. 
We refer to $f^{(\mu)}(x)$ and the $g^{(\mu)}(x)$ as the preactivations into and activations out of the $\mu$-th hidden layer, respectively. 
The top-most preactivations $f^{(D+1)}(x)$ are the outputs of the network.

We require mild assumptions on the nonlinearity $\phi$ for our main theorem to hold; these are the same assumptions made by \citetalias{matthews2018gaussian}, namely the linear envelope condition.

\begin{definition}[Linear envelope condition]
    \label{defn-linenv}
A nonlinearity $\phi:\RR\mapsto\RR$ is said to satisfy the linear envelope condition if it is continuous and there exist positive constants 
$C$ and $M$ such that
\begin{equation*}
|\phi(x)| < C + M|x| \mbox{ for all } x\in\RR.
\end{equation*}
\end{definition}

Many popular activation functions such as tanh, ReLU, and leaky ReLU satisfy the linear envelope condition, and thus our result is quite general with 
regards to the choice of nonlinearity.

We now turn the above network into a random network by placing IID normal distributions on the weights $w^{(\mu)}$ and biases $b^{(\mu)}$ of the 
network:
\begin{align}
w_{ij}^{(\mu)} &\sim \mathcal{N}\left(0, \frac{v_w^{(\mu)}}{H_{\mu-1}}\right), \label{eq-v_w} \\
b_i^{(\mu)} &\sim \mathcal{N}(0, v_b^{(\mu)}), \label{eq-v_b}
\end{align}
where we set $H_0 = 1$ for the purpose of defining these  distributions. 
The variance of the weights after the first layer are scaled inversely to the preceding hidden layer width so that the Central Limit Theorem can be applied to 
the convergence of BNNs to a GP. 
With a slight abuse of terminology, we will call the constants $v_w^{(\mu)}$ and $v_b^{(\mu)}$ ``weight and bias variance hyperparameters'' even though $v_w^{(\mu)}$ is not the actual variance of the weights.

The output $f^{(D+1)}(x)$ is now a random vector of dimension $L$ for each input $x$, and we therefore understand a BNN as an instance of a stochastic process. 
We give a formal definition next, after we introduce some notation.
If $\Omega$ is a probability space and $E$ is a measurable space, then an $E$-valued stochastic process $F$ with index set $X$ is a function $F: X\times\Omega\mapsto E$ such that $F(x, \cdot)$ is a measurable function for each $x\in X$. 
By the notation $F(x)$, we refer to the random variable $F(x): \Omega\mapsto E$ defined by $F(x)(\omega) = F(x, \omega)$. 
%When the probability space $\Omega$ and state space $E$ are left implicit, we will often denote the stochastic process $F$ by $\{F(x): x\in X\}$; this notation emphasizes that a stochastic process is an indexed collection of random variables. 

\begin{definition}[Bayesian neural network] \label{defn-bnn}
A Bayesian neural network (BNN) $F$ mapping $\RR^M$ to $\RR^L$ with $D$ hidden layers of widths $H_{\mu}$, $\mu\in\{1,\ldots,D\}$, and nonlinearity 
$\phi:\RR\mapsto\RR$ is a stochastic process $F: \RR^M\times\Omega\mapsto\RR^L$ defined such that $F(x) = f^{(D+1)}(x)$, where $f^{(D+1)}$ is 
the neural network output defined through the recursion of preactivations and activations of Eqs.~\eqref{eq-gmu}-\eqref{eq-fmu}.
\end{definition}

\citetalias{matthews2018gaussian} prove the following theorem concerning the convergence of BNNs with no bottleneck layers.

\begin{theorem}[NNGP theorem,~\citetalias{matthews2018gaussian}] \label{thm-matthews}
Let $\{F[n]\}_{n=1}^\infty$ be a sequence of BNNs mapping $\RR^M$ to $\RR^L$ each with $D$ hidden layers of widths $H_{\mu}[n]$, $\mu\in\{1,\ldots,D\}$, and nonlinearity $\phi:\RR\mapsto\RR$ that satisfies the linear envelope condition. 
If $H_{\mu}[n]$ is strictly increasing in $n$ for each $\mu$, then $\{F[n]\}_{n=1}^\infty$ restricted to a countable index set $\mathcal{X}\subset \RR^M$ 
converges in distribution to the Gaussian process $\mathcal{GP}(0, K)$, where $K$ is a kernel
defined recursively by Eqs.~\eqref{eq-K1}-\eqref{eq-Kmu}.
\end{theorem}

Note that we use the suffix $[n]$ instead of a subscript to index a sequence of stochastic processes.
Convergence in distribution is defined in the measurable space $((\RR^L)^{\infty}, \mathcal{A})$ of $\RR^L$-valued sequences; 
details are provided in Appendix~\ref{appendix-section-convergence}. 
The limiting GP in Thm.~\ref{thm-matthews} is called a neural network Gaussian process (NNGP). 
If $f^{(\mu)}_i$ is the limiting NNGP of $\{f^{(\mu)}_i[n]\}_{n=1}^\infty$, then the NNGP kernel is defined through a recursion with initial step
\begin{equation} \label{eq-K1}
K^{(1)}_{ij}(x_1, x_2) = \operatorname{E}[f^{(1)}_i(x_1) f^{(1)}_j(x_2)] 
= \delta_{ij}(v_b^{(1)} + v_w^{(1)} x_1\cdot x_2),
\end{equation}
and for $\mu\in\{1,\ldots,D\}$,
\begin{align} \label{eq-Kmu}
\begin{aligned}
K^{(\mu+1)}_{ij}(x_1, x_2) &= \operatorname{E}[f^{(\mu+1)}_i(x_1) f^{(\mu+1)}_j(x_2)] \\
&= \delta_{ij}\left(v_b^{(\mu+1)} + v_w^{(\mu+1)}\operatorname{E}_{z_1,z_2\sim \mathcal{N}(0, C^{(\mu)})}[\phi(z_1)\phi(z_2)]\right),
\end{aligned}
\end{align}
where $C^{(\mu)}$ is the $2\times 2$ matrix with entries $c^{(\mu)}_{ab} = K^{(\mu)}_{11}(x_a, x_b)$; 
here we could have used $K^{(\mu)}_{ii}$ in place of $K^{(\mu)}_{11}$ for any $i$ since the NNGP preactivations $f^{(\mu)}_i(x)$ into the $\mu$-th hidden layer are IID over $i$. 
The (countably infinite) kernel matrix $K^{(\mu)}(\mathcal{X}, \mathcal{X})$ is therefore block-diagonal
with the $(i, j)$-th block $K^{(\mu)}_{ij}(\mathcal{X}, \mathcal{X})$. 

The kernel $K: \RR^M\times\RR^M\mapsto \RR^{L\times L}$ in Thm.~\ref{thm-matthews} is given by $K = K^{(D+1)}$. 
Observe that the $L$ outputs of the BNNs converge to IID GPs so that all correlations between the outputs of the networks are lost in the infinite width limit. 
We will see  that bottleneck layers help to preserve some correlations between outputs.

\section{The bottleneck NNGP theorem}
%\subsection{The bottleneck NNGP limit}
\label{section-bottleneck-limit}

In this section, we state and prove the bottleneck NNGP theorem,
we show that a single-bottleneck NNGP approximates a no-bottleneck NNGP as the bottleneck width is increased,
and we explore the effect of depth and width on bottleneck NNGPs using three example datasets.

\subsection{Statement of main theorem}
\label{subsection-bottleneck-nngp-theorem}

We now state our main theorem, which is a direct generalization of Thm.~\ref{thm-matthews} to compositions of BNNs. 
Given two stochastic processes $F^{(1)}: X\times\Omega^{(1)}\mapsto Y$ and $F^{(2)}: Y\times\Omega^{(2)}\mapsto Z$, we define the composition $F^{(2)}\circ F^{(1)}$ as the stochastic process $F^{(2)}\circ F^{(1)}: X\times (\Omega^{(1)}\times \Omega^{(2)})\mapsto Z$ with
\begin{equation*}
(F^{(2)}\circ F^{(1)})(x, (\omega_1, \omega_2)) = F^{(2)}(F^{(1)}(x, \omega_1), \omega_2).
\end{equation*}

\begin{theorem}[Bottleneck NNGP theorem] \label{thm-bottleneck}
Let $\{B_d\in\NN\}_{d=0}^D$ for $D\in\NN$ with $B_0=M$ and $B_D=L$. 
For each $d\in\{1,\ldots,D\}$, let $\{F^{(d)}[n]\}_{n=1}^\infty$ be a sequence of BNNs mapping $\RR^{B_{d-1}}$ to $\RR^{B_d}$ with $D_d$ hidden layers of widths $H^{(d)}_{\mu}[n]$ , $\mu\in\{1,\ldots,D_d\}$, and nonlinearity $\phi$ that satisfies the linear envelope condition. 
If $H^{(d)}_{\mu}[n]$ is strictly increasing in $n$ for each $d\in\{1,\ldots,D\}$ and $\mu\in\{1,\ldots,D_d\}$, then the sequence of bottleneck random neural 
networks $\{F^{(D)}[n]\circ\cdots\circ F^{(1)[n]}\}_{n=1}^\infty$ restricted to a countable index set $\mathcal{X}\in\RR^M$ converges in distribution in 
$((\RR^L)^{\infty}, \mathcal{A})$ to $F^{(D)}\circ\cdots\circ F^{(1)}$, where $F^{(d)}$ is the limiting NNGP of $\{F^{(d)}[n]\}_{n=1}^\infty$.
\end{theorem}

\begin{remark}[Nonlinear bottleneck] \label{rem-nonlinbottleneck}
Theorem~\ref{thm-bottleneck} as stated above assumes no nonlinearity on the bottleneck layers. 
However, the theorem also holds when we replace $F^{(d)}[n]$ and $F^{(d)}$ with $F^{(d)}[n]\circ\left(\frac{1}{\sqrt{B_{d-1}}}\phi\right)$ and 
$F^{(d)}\circ\left(\frac{1}{\sqrt{B_{d-1}}}\phi\right)$ respectively for $d\in\{2,\ldots,D\}$, that is when 
we scale the weights after each bottleneck layer by layer width in the same way as all other weights after a hidden layer and
and we place nonlinearities $\phi$ on the bottleneck layers.
The proof is nearly identical to the proof of Thm.~\ref{thm-bottleneck} (see Remark~\ref{rem-nonlinbottleneck-appendix} in Appendix~\ref{appendix-section-main} for details).
\end{remark}

\begin{remark}[Discontinuous nonlinearity] \label{rem-discont}
Theorem~\ref{thm-bottleneck} holds even if the nonlinearity $\phi:\RR\mapsto\RR$ is continuous only almost everywhere, as long as $\phi$ is continuous at $0$ or $v_b > 0$. 
Each of these two conditions ensures that the Continuous Mapping Theorem is still applicable in Lemmas~\ref{lemma-condition-variance}-\ref{lemma-condition-3} (see Remark~\ref{rem-discont-appendix} in Appendix~\ref{appendix-section-main} for details). 
This extends the class of allowable nonlinearities to include such prominent examples as the Heaviside step function used in the first perceptron model~\citep{rosenblat1958perceptron}.
\end{remark}

\begin{remark}[Converse of the main theorem] \label{rem-converse}
The converse of Thm.~\ref{thm-bottleneck}---that every DGP is the bottleneck NNGP limit of a BNN with IID priors and a nonlinearity satisfying the linear envelope condition---does not hold. 
A simple counterexample is a no-bottleneck single-hidden-layer NNGP with the rectified polynomial unit activation $\phi(x) = \operatorname{max}(0, x)^n$ for $n\geq 2$~\citep{cho2009kernel}; 
this is a GP that can only result from a wide limit if the linear envelope condition is violated. 
A more trivial counterexample is any GP with a linear kernel $k(x_1, x_2) = x_1^\top G x_2$ where $G$ is a symmetric positive-semidefinite matrix not proportional to the identity matrix. 
By Eq.~\eqref{eq-K1}, the NNGP kernel depends on its inputs only through their dot product and is thus invariant under rotations. 
The kernel $k$ with metric $G$ can therefore only arise from a wide limit if the ``IID priors'' condition is violated.
%Finding conditions under which the converse holds or strengthening Thm.~\ref{thm-bottleneck} to hold for a larger class of activation functions are left as open problems.
\end{remark}

Here we consider a sequence of BNNs with $D-1$ bottleneck layers of widths $B_1,\ldots,B_{D-1}$. 
As all hidden layers except the bottleneck layers tend to infinite width, each component network converges to an NNGP by Thm.~\ref{thm-matthews}, but it is less obvious that the composition of components tends to the composition of the limiting NNGPs. 
Our proof depends on several original lemmas (Lemmas~\ref{lemma-Fn-Xn}-\ref{lemma-tildeFn} in Appendix~\ref{appendix-section-main}).
However, Lemma~\ref{lemma-tildeFn} is a simple generalization of Lemma 12 in \citetalias{matthews2018gaussian}, and its proof therefore runs in parallel to that in \citetalias{matthews2018gaussian}. 
The complete proof of the main theorem as well as proofs for all supporting lemmas can be found in Appendix~\ref{appendix-section-bottleneck}; 
we recommend readers to start at the introduction of Appendix~\ref{appendix-section-bottleneck}, 
where we provide a detailed sketch of the proof and discuss the high-level function of each lemma. 
Definitions and properties of the various modes of convergence of stochastic processes pertinent to the proof are discussed in Appendix~\ref{appendix-section-convergence}.

\subsection{Correspondence to the no-bottleneck NNGP}
\label{subsection-correspondence}

We expect that in the limit as bottlenecks are sent to infinite width, the bottleneck NNGP converges to the (no-bottleneck) NNGP with the same number of hidden layers. 
The next theorem gives this result for the case of a bottleneck NNGP with one bottleneck layer.

\begin{theorem}[Wide bottleneck correspondence] \label{thm-correspondence}
Let $\{F^{(H)}\}_{H=1}^\infty$ be a sequence of single-bottleneck NNGPs mapping $\RR^M$ to $\RR^L$ with $D_1$ hidden layers and $D_2$ hidden layers before and after the bottleneck of width $H$ and with nonlinearity $\phi:\RR\mapsto\RR$ that satisfies the linear envelope condition. 
Suppose the nonlinearity $\phi$ is also applied to the bottleneck and that its activations are scaled by $\frac{1}{\sqrt{H}}$, in accordance with Remark~\ref{rem-nonlinbottleneck}. 
Suppose also that IID Gaussian noise $\calN(0, v_n)$ is added to the networks for any $v_n > 0$. 
Then
\begin{enumerate}[label={(\alph*)}]
\item $\{F^{(H)}\}_{H=1}^\infty$ restricted to a countable index set $\mathcal{X}\subset \RR^M$ converges in distribution in $((\RR^L)^{\infty}, \mathcal{A})$ to an NNGP $F$ with $D_1+D_2+1$ hidden layers and Gaussian noise $\calN(0, v_n)$. \label{thm-correspondence:a}
\item For every finite set of inputs $X\subset\RR^M$, the sequence of probability density functions (PDFs) of $\{F^{(H)}\}_{H=1}^\infty$ converges pointwise to the PDF of the NNGP $F$. \label{thm-correspondence:b}
\end{enumerate}
\end{theorem}

\begin{remark} \label{rem-correspondence}
Statement~\ref{thm-correspondence:a} of Thm.~\ref{thm-correspondence} holds even if there is no additive Gaussian noise ($v_n = 0$); 
the proof uses the technique in the proof of Lemma~\ref{lemma-Fn-Xn}, 
where the function $X\rightarrow \Pr(F(X)\in U)$ is shown to be continuous for $F$ an NNGP and $U$ a continuity set.
\end{remark}

The proof of Thm.~\ref{thm-correspondence} is given in Appendix~\ref{appendix-section-correspondence}. 
It is based on the observation that the activations in the bottleneck layer of a single-bottleneck NNGP are IID. 
Since the post-bottleneck NNGP depends on these activations through their Gram matrix only and since the activations are inversely scaled by the bottleneck width, then the post-bottleneck NNGP is a function of the sample covariance of bottleneck activations, 
which converges to the pre-bottleneck NNGP kernel in the limit of infinite bottleneck width by the Law of Large Numbers. 
Extending this result to the case of multiple bottlenecks is left for future work.

\subsection{Experiments}
\label{subsection-nngp-experiments}

Statement~\ref{thm-correspondence:b} of Thm.~\ref{thm-correspondence} implies that the marginal log-likelihood (MLL) of a single-bottleneck NNGP architecture given data $(X, Y)$ and fixed variance hyperparameters $v_n, v_b, v_w$ converges to the MLL of the corresponding NNGP as the bottleneck is sent to infinite width. 
The MLL of the bottleneck NNGP is just the logarithm of the PDF given in Eq.~\eqref{eq-pH} evaluated at the dataset. 
This formally validates the intuition that a bottleneck NNGP with a sufficiently wide bottleneck is a similar model to a no-bottleneck NNGP. 
However, the utility of bottleneck NNGPs with narrower bottlenecks as measured by MLL is less clear. 
We investigate this question on a simulated dataset that we call Rings and on two publicly available datasets---Fisher's Iris data set~\citep{anderson1935, fisher1936} and the US Census Boston housing prices dataset~\citep{harisson1978}.

The Rings dataset consists of $120$ points lying on two interlocked cylindrical bands or ``rings'' orthogonal to one another and passing through one another's centers (Fig.~\ref{fig-rings-vis}).
We took a regular $12\times 5$ lattice of 60 points in $[0, 2\pi)\times\left[-\frac{1}{2}, \frac{1}{2}\right]$ and mapped it onto one of the rings in $\RR^3$ by
\[ (\theta, z) \mapsto (\cos\theta, \sin\theta, z). \]
To generate the second ring, we rotated a copy of the first by $90^\circ$ in the $xz$-plane and translated it by $1$ along the $y$-axis. 
We assigned a label of $0$ to the 60 points on the first ring and a label of $1$ to the 60 points on the second.
The Rings dataset is therefore a non-linearly separable binary classification problem where the dimension of the data manifold is less than that of the linear span of the data points.

\begin{figure}
\centering
\includegraphics[width=0.35\textwidth]{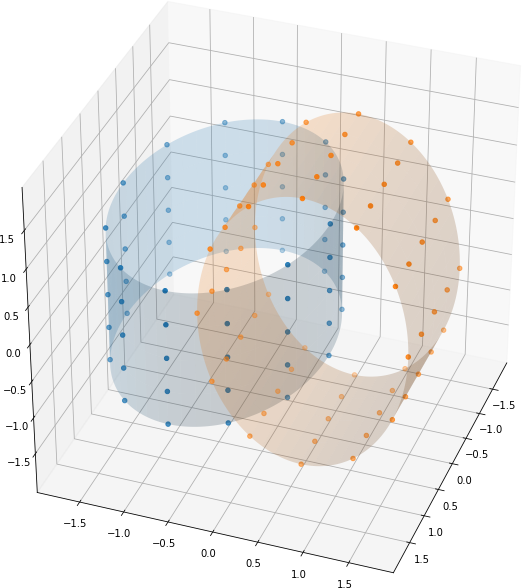}
\caption{\label{fig-rings-vis} %
Visualization of the simulated Rings dataset. 
We take 60 regularly spaced points from each ring. 
We assign a binary label to each point based on the ring to which it belongs.
}
\end{figure}

We also considered the Iris and Boston House-Prices datasets. 
Like Rings, Iris is a classification problem, but we one-hot encoded its labels to include a multivariate dataset with strongly correlated labels. 
For simplicity, following~\citet{lee2017deep}, we implemented all three problems as regression tasks. 
We standardized both the input and target sets of all three datasets to help place the three problems on similar scales, as well as to set a reasonable scale for the variance hyperparameters $v_b$ and $v_w$. 
%After standardization, we added IID Gaussian noise with variance $10^{-4}$ to the labels of the Rings dataset.

%We used a single-bottleneck NNGP with one infinitely wide hidden layer before the bottleneck and various numbers of infinitely wide hidden layers afterwards. 
%We set the noise variance $v_n=10^{-4}$ to match the true variance in the labels of the Rings dataset. 
%We used the same noise variance for all three datasets; 
%since it is small, then it allows us to do model comparison for bottleneck NNGP architectures without having to account for variability in noise. 
%We set the bias and weight variances to $(v_b, v_w) = (0.09, 1.1)$. 

We used a single-bottleneck NNGP with one infinitely wide hidden layer before the bottleneck, bottleneck width $H$, and post-bottleneck depth $D$ (i.e., $D$ infinitely wide hidden layers after the bottleneck) for various values of $H$ and $D$. 
We equipped all hidden neurons with the normalized ReLU activation
\begin{equation} \label{eq-relu}
\phi(x) = \sqrt{2}\operatorname{max}(0, x).
\end{equation}
The propagation of the NNGP kernel through the hidden layers given in Eq.~\eqref{eq-Kmu} admits a closed form for the normalized ReLU activation and is given by~\citep{cho2009kernel}:
\begin{align}
K^{(\mu+1)}_{ij}(x_1, x_2) 
& = \delta_{ij}\left[v_b + v_w\cdot\frac{1}{\pi}\sqrt{K^{(\mu)}_{11}(x_1, x_1)K^{(\mu)}_{11}(x_2, x_2)}J_1(\theta^{(\mu)})\right], \label{eq-Kmu-relu} \\
\theta^{(\mu)}
& = \cos^{-1}\left[\frac{K^{(\mu)}_{11}(x_1, x_2)}{\sqrt{K^{(\mu)}_{11}(x_1, x_1)K^{(\mu)}_{11}(x_2, x_2)}}\right], \nonumber
\end{align}
where the function $J_1$ is defined as
\begin{equation*} %\label{eq-J1}
J_1(\theta) = \sin\theta + (\pi-\theta)\cos\theta.
\end{equation*}

Our goal is to find the bottleneck NNGP architecture (i.e., combination of bottleneck width $H$ and post-bottleneck depth $D$) with the greatest likelihood given a dataset $(X, Y)$ of $N$ observations. 
We calculated the MLL of each bottleneck NNGP architecture $(H, D)$ using
\begin{equation} \label{eq-experiments-mll}
\operatorname{MLL}(H, D; X, Y) = \log p(Y; X, H, D, v_{b*}(H, D), v_{w*}(H, D), v_{n*}(H, D)),
\end{equation}
where on the right-hand side, $p$ is the PDF in Eq.~\eqref{eq-pH} of the data outputs given the data inputs and network architecture, 
and where the variance hyperparameters are set to their maximum likelihood estimates. 
We found the optimal variance hyperparameters iteratively through gradient descent. 
During the forward pass through the network in each iteration, we estimated the integral in Eq.~\eqref{eq-pH} by drawing $100$ IID Monte Carlo (MC) samples---each an $N\times H$ matrix with IID columns---from the pre-bottleneck NNGP. 
We did so using the local reparameterization trick~\citep{kingma2015variational}, so that each sample is actually a transformation of a draw from the $(N\times H)$-dimensional standard normal distribution. 
We used the Adam optimizer~\citep{kingma2014adam} to take advantage of the gradient noise generated by MC sampling during optimization; 
we set the initial learning rate to $0.1$. 
In order to ensure that the noise observed in the learning curves was due only to MC sampling and not a large learning rate, we decayed the learning rate as follows: 
After the backward pass of each iteration, we re-evaluated the MLL using the same draw from the $(N\times H)$-dimensional standard normal distribution for the MC samples; 
if the new MLL was less than the value obtained from the initial forward pass of the iteration, then we multiplied the learning rate by $0.9$. 
We iterated the optimization procedure until convergence of the MLL learning curves; 
once complete, we evaluated Eq.~\eqref{eq-experiments-mll} once more---this time with $1000$ MC samples---to obtain the final MLL estimate for each network architecture. 

On all three datasets, the maximum MLL is attained at a finite bottleneck width and post-bottleneck depth (
$H_*=1024$ and $D_*=1$ for Rings; 
$H_*=8$ and $D_*=5$ for Iris; 
$H_*=64$ and $D_*=7$ for Boston), 
thus demonstrating the utility of bottleneck layers in NNGP models (Fig.~\ref{fig-heatmaps}). 
On Rings and Boston, we also observe that the optimal post-bottleneck depth conditional on a bottleneck width roughly decreases as the bottleneck width increases. 
since the no-bottleneck NNGP kernel with ReLU activation is known to degenerate to a constant kernel with no discriminative power~\citep{lee2017deep}, 
it makes sense that a deeper network may require a narrower bottleneck to help information propagate through the network. 
although not conclusive from the figures alone, we believe this observation at least warrants further investigation. 
In Sec.~\ref{section-bottleneck-dep}, we do just that and find that when the variance hyperparameters are fixed, then the bottleneck width and post-bottleneck depth are indeed intimately related.

\begin{figure}
\centering
\subfloat[Rings]{\includegraphics[height=0.21\textwidth]{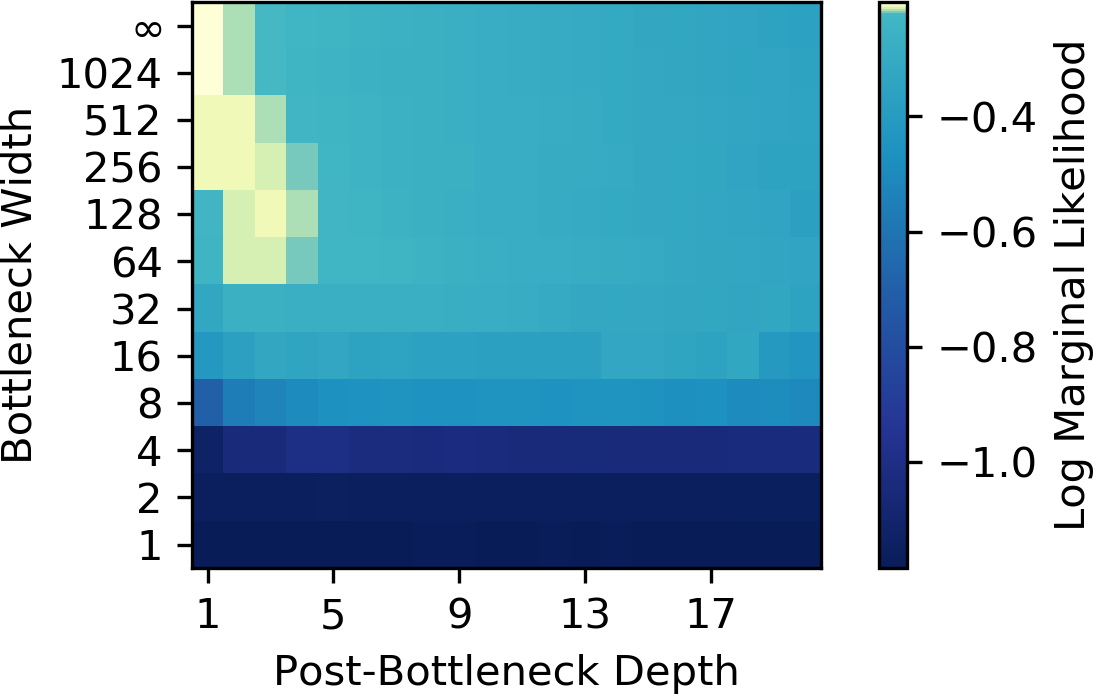}} \hfill
\subfloat[Iris]{\includegraphics[height=0.21\textwidth]{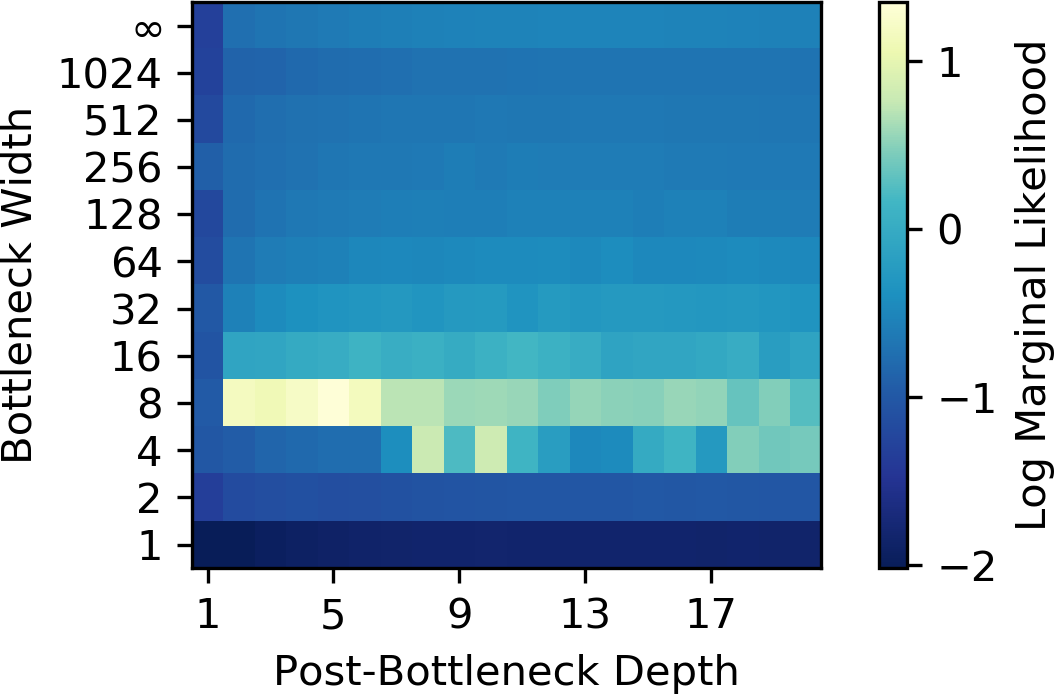}} \hfill
\subfloat[Boston]{\includegraphics[height=0.21\textwidth]{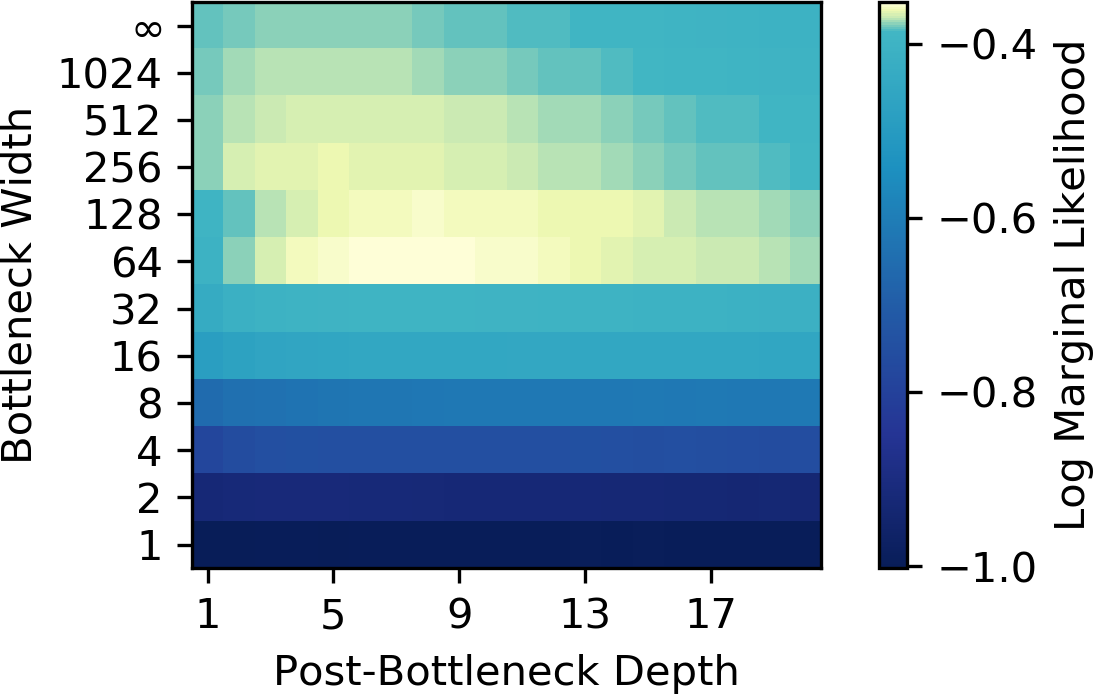}}
\caption{\label{fig-heatmaps} %
Marginal log-likelihoods (MLL) of three datasets normalized by number of observations (data points) under a bottleneck NNGP for various bottleneck widths and post-bottleneck depths. 
Infinite bottleneck width corresponds to the limiting no-bottleneck NNGP. 
On all three datasets, the maximum MLL is attained at some finite bottleneck width and post-bottleneck depth. 
}
\end{figure}

\section{Bottleneck layers induce dependence}
\label{section-bottleneck-dep}

In Sec.~\ref{subsection-nngp-experiments}, we showed empirically that the model likelihood of an NNGP is improved if one of the hidden layers is restricted to a finite-width bottleneck and speculated that the optimal post-bottleneck depth may increase as the bottleneck is narrowed. 
In this section, we investigate possible mechanisms underlying these observed trends in model performance. 
We start by showing that although linear correlations between distinct response variables or ``output neurons'' of an NNGP remain zero even in the presence of bottlenecks, 
the corresponding \emph{quadratic} correlations are often non-zero. 
We also analyze the behavior of this quadratic correlation in the deep post-bottleneck limit---i.e., as the post-bottleneck depth is sent to infinity. 
This deep limit is distinct from the ones typically considered for DGPs, where the number of GP components is sent to infinity, and for NNGPs, where the depth of a single GP component is sent to infinity. 
Proposition~\ref{prop-q} provides a striking result, which implies that bottleneck layers help a network retain discriminative power even at extreme post-bottleneck depths.

Note that in this section, we primarily consider the ReLU activation defined in Eq.~\eqref{eq-relu} as it is by far the most common nonlinearity used in deep learning today. 
In Sec.~\ref{subsection-other-nonlin}, we briefly consider other nonlinearities and contrast their deep limit behaviors with that of the ReLU activation, 
thereby highlighting the peculiarities of the ReLU activation.

\subsection{An exact formula for quadratic correlation}
\label{subsection-cor-exact-formula}

The outputs of a multi-output GP prior are IID, and it follows that the outputs of the corresponding posterior remain independent (though not necessarily identically distributed). 
This is a limitation of GPs for multi-task learning applications, since information cannot be shared across tasks. 
One method to solve this problem was proposed by \citet{bonilla2008multi} who introduce a coupling matrix hyperparameter through which distinct output neurons or tasks can interact.
This method, if applied to finite-width neural networks, would be superfluous since tasks could share information through a common set of features learned in the earlier layers of the network.

The key to correlating tasks in neural networks is clearly not depth alone, since the outputs of a (no-bottleneck) NNGP---however deep---are independent. 
Rather, following from the DGP framework as discussed in Sec.~\ref{subsection-dgp}, the outputs of an NNGP become correlated if bottleneck layers are introduced, 
so that bottleneck NNGPs support multi-task learning out of the box. 
Correlation arising from finite-width bottleneck layers is exactly the type of behavior we expect in neural network architectures such as word embedding layers and many kinds of autoencoders, where the bottleneck width forces dense feature representations (i.e., feature representations that capture correlation) to be learned. 
The correlation structure that is induced in an NNGP prior through bottleneck layers is, however, subtle; 
distinct outputs of a bottleneck NNGP prior remain linearly uncorrelated (i.e., have zero covariance) but can be quadratically correlated (i.e., the squares of the outputs have non-zero covariance).

The expression for the quadratic correlation of outputs in a single-bottleneck NNGP prior can be obtained in closed form. 
Consider a bottleneck NNGP $F: \RR^M\times\Omega\mapsto\RR^2$ with one bottleneck layer of width $H$, any number of infinitely wide hidden layers before the bottleneck, and $D-1$ infinitely wide hidden layers after the bottleneck. 
Suppose all hidden neurons (including in the bottleneck) are equipped with the normalized ReLU activation defined in Eq.~\eqref{eq-relu}.
We also scale the bottleneck activations by $\frac{1}{H}$ in accordance with Remark~\ref{rem-nonlinbottleneck}. 
Let $v_b$ and $v_w$ be the bias variance and weight variance hyperparameters of both the pre-bottleneck and post-bottleneck components of the bottleneck NNGP defined as in Eqs.~\eqref{eq-v_w}-\eqref{eq-v_b}. 
Suppose the bottleneck NNGP is fed two inputs $x_1,x_2\in\RR^M$. 
Then the preactivations into the $H$ neurons in the bottleneck layer are IID with common 2D normal distribution $\mathcal{N}(0, C)$ for some covariance matrix
\[ C = \begin{bmatrix} c_{11} & c_{12} \\ c_{21} & c_{22} \end{bmatrix}. \]
As is commonly done with DGPs~\citep{damianou2013deep, salimbeni2017doubly}, we assume that IID Gaussian noise $\mathcal{N}(0, v_n)$ (with $v_n$ arbitrarily small) is added to the preactivations of the bottleneck layer; 
We include the variance of the noise in the covariance matrix $C$, so that $C$ is invertible. 
We also add IID Gaussian noise $\mathcal{N}(0, v_n)$ to each of the two outputs of the bottleneck NNGP.

%The example bottleneck NNGP $F$ introduced here is the only architecture used throughout the rest of Sec.~\ref{section-bottleneck-dep}; 
%any reference to a bottleneck NNGP $F$ in the rest of Sec.~\ref{section-bottleneck-dep} pertains to this architecture, without being explicitly stated thereafter.

It is easy to verify that the two outputs $F_1(x_a)$ and $F_2(x_b)$ (for any $a,b\in\{1,2\}$) are linearly uncorrelated; 
this immediately follows from the conditional independence of outputs given the bottleneck activations. 
The relationship between the squares of the outputs is, however, less trivial. 
Let $Q^{\times}$ denote the matrix of correlations between the squares of the two outputs of $F$ (the superscript $\times$ is used to emphasize that the correlation is between distinct output neurons). 

\begin{proposition}[Quadratic correlation between outputs]\label{prop-quad-cor}
Consider the single-bottleneck NNGP $F$ with normalized ReLU activation defined above. 
Then
\begin{align}
\label{eq-square-cov}
\begin{aligned}
\cov[F_1(x_a)^2, F_2(x_b)^2]
&= \frac{v_w^{2D}}{H} \cov_{(z,z^\prime)\sim\mathcal{N}(0, K)}[\phi(z)^2, \phi(z^\prime)^2] \\
&= \frac{v_w^{2D}c_{aa}c_{bb}}{H}\left(\frac{2}{\pi}J_2(\beta)-1\right),
\end{aligned}
\end{align}
where $\beta = \frac{c_{ab}}{\sqrt{c_{aa}c_{bb}}}$ and the function $J_2$ is defined as~\citep{cho2009kernel}:
\begin{equation*}
J_2(\beta) = 3\sin\beta\cos\beta + (\pi-\beta)(1 + 2\cos^2\beta).
\end{equation*}
The corresponding correlation is
\begin{align}
\label{eq-square-cor}
\begin{aligned}
q^{\times}_{ab}
&= \frac{\cov[F_1(x_a)^2, F_2(x_b)^2]}{\sqrt{\operatorname{V}[F_1(x_a)^2] \operatorname{V}[F_2(x_b)^2]}} \\
% &= \frac{\left(\frac{2}{\pi}J_2(\beta)-1\right)}{\sqrt{\left[15 + 2H\left(\frac{r_D}{c_{aa}} + 1\right)^2\right]\left[15 + 2H\left(\frac{r_D}{c_{bb}} + 1\right)^2\right]}}, 
&= \frac{\left(\frac{2}{\pi}J_2(\beta)-1\right)}{\sqrt{\left[15 + 2H\left(\frac{r_D}{c_{aa}} + 1\right)^2\right]\left[15 + 2H\left(\frac{r_D}{c_{bb}} + 1\right)^2\right]}},
\end{aligned}
\end{align}
where
\begin{equation} \label{eq-rD}
r_D =
\begin{cases}
v_n + D v_b & \mbox{ if } v_w = 1 \\
\frac{v_n}{v_w^D} + \frac{v_b}{1-v_w}\left(\frac{1}{v_w^D}-1\right) & \mbox{ otherwise.}
\end{cases}
\end{equation}
\end{proposition}

\begin{remark}[Quadratic correlation for stationary kernels] \label{rem-quad-cor-stationary}
In Prop.~\ref{prop-quad-cor}, if we instead consider a single-bottleneck DGP $F$ where the post-bottleneck GP has a stationary kernel (such as the RBF kernel), 
then the quadratic correlation between outputs is $q^{\times}_{ab} = 0$. 
The non-stationarity of the NNGP kernel is therefore key to capturing some amount of correlation.
\end{remark}

The proof of Prop.~\ref{prop-quad-cor} as well as the proofs of all other propositions in Sec.~\ref{section-bottleneck-dep} are given in Appendix~\ref{appendix-section-cor}. 
The significance of Prop.~\ref{prop-quad-cor} is two-fold. 
First, quadratic correlation under the prior---although subtle---may translate to stronger dependence (such as linear correlation) under the posterior. 
Indeed, as stated in Remark~\ref{rem-quad-cor-stationary}, a DGP with RBF kernel captures less correlation under its prior than does a bottleneck NNGP with ReLU activation, 
and yet the former has been shown to be useful in modeling dependence in practice~\citep{alaa2017deep, wang2016sequential}. 
This suggests that a bottleneck NNGP with ReLU activation may be just as useful in modeling dependence. 
Second, the ability of the network to capture quadratic correlation is closely linked to its ability to operate effectively at extreme depths; 
we discuss this in more detail in Sec.~\ref{subsection-nondegenerate}; 

The pre-bottleneck component of the bottleneck NNGP is a map that sends input vectors in $\RR^M$ to normally distributed real-valued random variables (preactivations of bottleneck neurons). 
By understanding covariance as an inner product on the space of finite-variance random variables, we can see that the covariance matrix $C$ is the Gram matrix of bottleneck preactivations, 
and the angle $\beta$ appearing in Prop.~\ref{prop-quad-cor} is the angle between the preactivations of two inputs at one bottleneck neuron; 
we will call $\beta$ the bottleneck angle.

The quadratic correlation (Eq.~\eqref{eq-square-cor}) varies with the bottleneck width $H$ roughly as $\frac{1}{H}$ and thus vanishes in the limit of infinite bottleneck width, recovering the independence of outputs of an NNGP with no bottlenecks. 
We empirically verified Eq.~\eqref{eq-square-cor} for a bottleneck NNGP prior with one hidden layer before the bottleneck, one hidden layer after the bottleneck ($D=2$), and with variance parameters $v_b=v_w=1$ and $v_n=10^{-4}$. 
We fed the example bottleneck NNGP two inputs $x_1=(1, 0)$ and $x_2=(0, 1)$. 
Then for each bottleneck width $H\in\{1,\ldots,10\}$, we generated $10^6$ IID samples of $(F_1(x_1), F_2(x_2))$ and used them to estimate $q^{\times}_{12}$. 
We repeated this simulation $10$ times, and we report the mean estimate of $q^{\times}_{12}$ along with its standard deviation for each bottleneck width (Fig.~\ref{fig-quadcorrempirical}). 
The empirical quadratic correlation estimates are very close to the theoretical values predicted by Eq.~\eqref{eq-square-cor}, with standard deviations all on the order of $10^{-3}$.

%\begin{table}
%\centering
%\input{tables/cor_over_width_2.tex}
%\caption{\label{table-cor-over-width} %
%Empirical estimates of the quadratic correlation of the outputs of an example bottleneck NNGP compared to theory (Eq.~\eqref{eq-square-cor}) for various widths of the bottleneck layer.
%}
%\end{table}
\begin{figure}
    \centering
    \includegraphics[width=.7\textwidth]{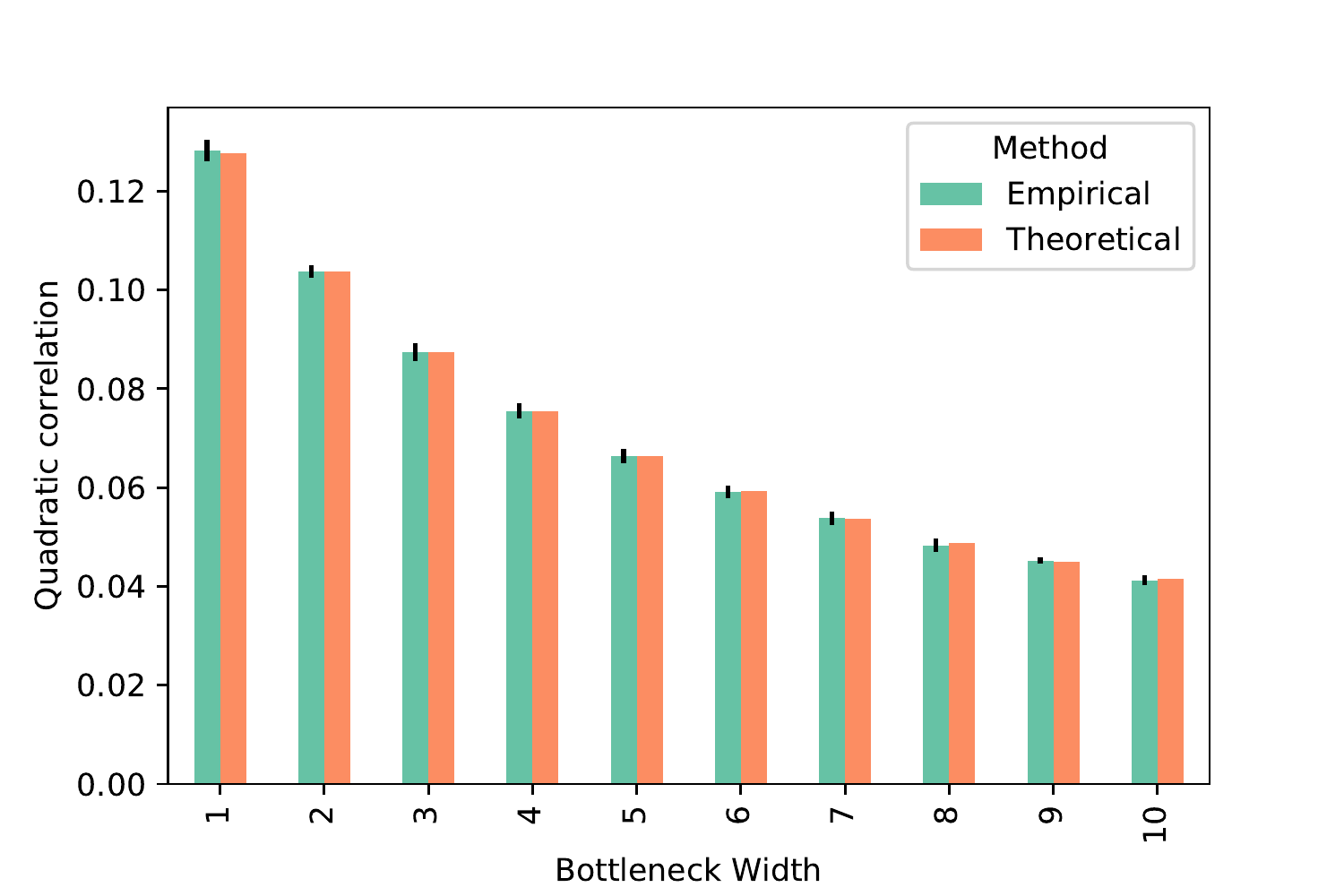}
    \caption{\label{fig-quadcorrempirical}
    Empirical estimates with standard errors of the quadratic correlation of the outputs of an example bottleneck NNGP with one bottleneck surrounded by two infinite hidden layers, compared to theory (Eq.~\eqref{eq-square-cor}) for various widths of the bottleneck layer. The theoretical values are all within one standard error of the empirical mean value.}
\end{figure}

We performed additional simulations to understand how multiple bottleneck layers affect the correlation of outputs, as we found this to be intractable theoretically.
We still consider a bottleneck NNGP with $v_b=v_w=1$ and $v_n=10^{-4}$ that is fed two 2D inputs $x_1=(1, 0)$ and $x_2=(0, 1)$, but now we suppose the bottleneck NNGP has 11 hidden layers (including all bottleneck layers). 
We chose 11 hidden layers since it allows us to restrict zero to three hidden layers to bottlenecks such that the bottlenecks are equally spaced in depth. 
For each of the zero to three bottleneck layers, we ran the experiment described above for a single bottleneck and estimated the quadratic correlation $q^{\times}_{12}$ along with its standard deviation over ten runs for various bottleneck widths (all bottleneck layers have the same width). 
The quadratic correlation tends to zero with increasing bottleneck width regardless of the number of bottleneck layers, as we expect (Fig.~\ref{fig-cor-multi-bottleneck}). 
We additionally observe that for bottleneck widths $H\geq 2$, the quadratic correlation increases with the number of bottlenecks even if the overall depth of the bottleneck NNGP remains the same, 
suggesting that bottlenecking more layers has a similar effect to further narrowing existing bottlenecks.

\begin{figure}
\centering
\includegraphics[width=.7\textwidth]{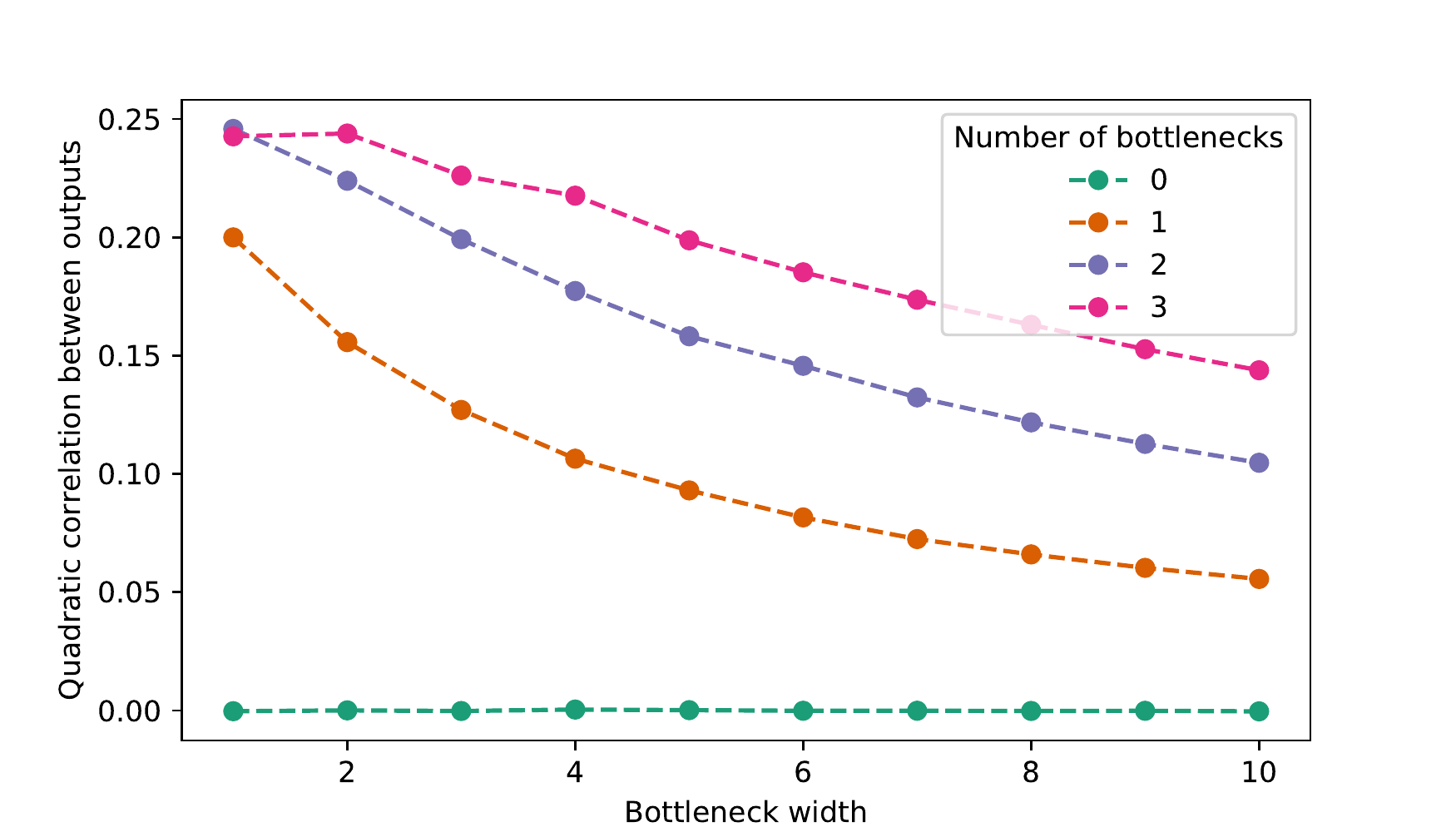}
\caption{\label{fig-cor-multi-bottleneck} %
Empirical quadratic correlation of the outputs of an example bottleneck NNGP with 11 hidden layers, some of which are restricted to regularly spaced bottlenecks. 
Quadratic correlation increases with more numerous and narrower bottlenecks.
}
\end{figure}

%\begin{figure}[t]
%	
%	\begin{minipage}[c]{0.625\textwidth}
%		
%		\centering
%		\includegraphics[width=1\textwidth]{cor_multi_bottleneck.pdf}
%	\end{minipage}\hfill
%	\begin{minipage}[c]{0.375\textwidth}
%		\vspace{-0.05cm}
%		\caption{Empirical quadratic correlation of the outputs of an example bottleneck NNGP with 11 hidden layers, some of which are restricted to regularly spaced bottlenecks.  Quadratic correlation increases with more numerous and narrower bottlenecks.
%		}
%		\label{fig-cor-multi-bottleneck}
%	\end{minipage}
%\end{figure}

\subsection{Quadratic correlation as a function of depth}
\label{subsection-quad-cor-lim}

The inverse dependence of the quadratic correlation on bottleneck width is intuitive since we know that the outputs of an NNGP are independent in the absence of bottlenecks. 
There is also an interesting and less obvious dependence of the quadratic correlation on the post-bottleneck depth $D$ (where there are $D-1$ post-bottleneck hidden layers of infinite width) as well as on the angle $\beta$ between the random bottleneck preactivations of the inputs $x_a$ ($a=1,2$) in the bottleneck layer, as captured by the covariance matrix $C$. 
We denote the quadratic correlation by $q^{\times(D)}_{ab}$ to make explicit its dependence on the post-bottleneck depth $D$ and will sometimes write $q^{\times(D)}_{ab}(\beta)$ to further clarify its dependence on the bottleneck angle $\beta$. 
By Prop.~\ref{prop-quad-cor}, it is easy to verify that $q^{\times(D)}_{ab}(\beta)$ is strictly decreasing in $\beta$ on $[0, \pi]$ with $\beta < \frac{\pi}{2}$ giving positive correlation, $\beta=\frac{\pi}{2}$ giving zero correlation, and $\beta>\frac{\pi}{2}$ giving negative correlation at all depths $D$. 
The quadratic correlation between outputs therefore encodes the correlation of inputs in the bottleneck layer. 
By Eq.~\eqref{eq-rD}, $r_D$ is strictly increasing in $D$ regardless of the values of $v_b,v_w > 0$. 
It follows that the absolute quadratic correlation $|q^{\times(D)}_{ab}(\beta)|$ strictly decreases with $D$ for $\beta\neq 0$ (and remains $0$ otherwise). 
A final property of the quadratic correlation is its range of possible values, which easily follows from the ranges of $J_2$ and $r_D$:
\begin{equation} \label{eq-square-cor-bound}
-\frac{1}{17} < q^{\times(D)}_{ab}(\beta) < \frac{5}{17}.
\end{equation}
These properties are apparent in the plot of $q^{\times(D)}_{ab}(\beta)$ over $D$ for an example single-bottleneck NNGP ($H=2$, $v_b=0.09$, $v_w=1.1$) with no pre-bottleneck hidden layers that is fed two inputs $x_1=(1, 0)$ and $x_2=(\cos\alpha, \sin\alpha)$ for various values of the input angle $\alpha$ (Fig.~\ref{fig-cor-over-depth}); 
Since there are no pre-bottleneck hidden layers, the input angle $\alpha$ and bottleneck angle $\beta$ are related through the equation
\begin{equation*}
\cos\beta = \frac{v_b + v_w\cos\alpha}{v_b+v_w}.
\end{equation*}
Orthogonal bottleneck preactivations ($\beta = \frac{\pi}{2}$) and thus $0$ quadratic correlation in the example bottleneck NNGP are then achieved at an input angle $\alpha\approx 0.526\pi$. 
Observe in general that if $v_b > 0$, then $\beta < \alpha$ and the range of $\beta$ is strictly smaller than $[0, \pi]$, indicating that bias units promote positive quadratic correlation.

\begin{figure}
\centering
\includegraphics[width=.7\textwidth]{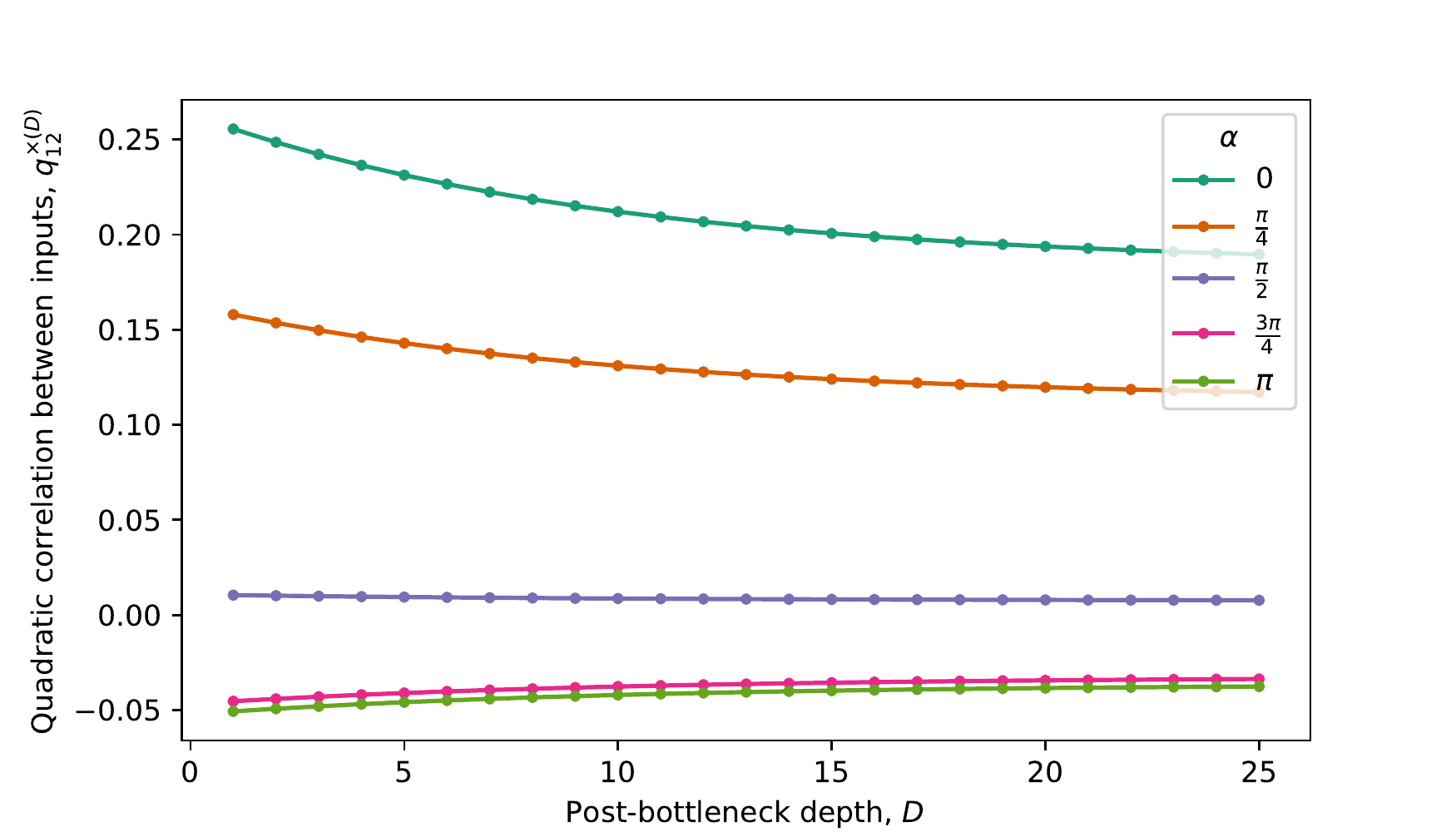}
\caption{\label{fig-cor-over-depth} %
Theoretical quadratic correlation between outputs of a single-bottleneck NNGP over the post-bottleneck depth $D$ for various angles $\alpha$ between the inputs. 
These quadratic correlations are asymptotically non-zero whenever $v_w>1$ and the bottleneck preactivations (as random variables) of the inputs are not orthogonal 
(bottleneck orthogonality occurs at $\alpha\approx 0.526\pi$).
}
\end{figure}

%\begin{figure}[t]
%	
%	\begin{minipage}[c]{0.625\textwidth}
%		
%		\centering
%		\includegraphics[width=1\textwidth]{cor_over_depth.pdf}
%	\end{minipage}\hfill
%	\begin{minipage}[c]{0.375\textwidth}
%	 \vspace{0.375cm}
%		\caption{Theoretical quadratic correlation between outputs of a single-bottleneck NNGP over the post-bottleneck depth $D$ for various angles $\alpha$ between the inputs. These quadratic correlations are asymptotically non-zero whenever $v_w>1$ and the inputs are not orthogonal.
%		}
%  	\label{fig-cor-over-depth}
%	\end{minipage}
%\end{figure}

The behavior of $r_D$ (Eq.~\eqref{eq-rD}) in the limit of infinite post-bottleneck depth ($D\rightarrow\infty$) is easily analyzed. 
We have
\begin{equation} \label{eq-r_infty}
r_{\infty} = \limD r_D =
\begin{cases}
\frac{v_b}{v_w-1} & \mbox{ if } v_w > 1 \\
\infty & \mbox{ otherwise.}
\end{cases}
\end{equation}
This lets us determine what happens to the quadratic correlation of outputs as the number of post-bottleneck hidden layers grows to infinity.

\begin{proposition}[Infinite-depth quadratic correlation between outputs] \label{prop-quad-cor-lim}
Consider the single-bottleneck NNGP $F$ with normalized ReLU activation from Prop.~\ref{prop-quad-cor}, and suppose we send the post-bottleneck depth to infinity.
\begin{enumerate}[label={(\alph*)}]
\item The infinite-depth quadratic correlation matrix $Q^{\times(\infty)}$ has $(a,b)$-th element
$q^{\times(\infty)}_{ab}
= \limD q^{\times(D)}_{ab}$
given by
\begin{equation} \label{eq-square-cor-lim}
q^{\times(\infty)}_{ab}
% = \limD q^{\times(D)}_{ab}
= \begin{cases}
% \frac{\left(\frac{2}{\pi}J_2(\beta)-1\right)}{\sqrt{\left[15 + 2H\left(\frac{v_b}{(v_w-1)c_{aa}} + 1\right)^2\right]\left[15 + 2H\left(\frac{v_b}{(v_w-1)c_{bb}} + 1\right)^2\right]}} & \mbox{ if } v_w > 1 \\
\frac{\left(\frac{2}{\pi}J_2(\beta)-1\right)}{\displaystyle\prod_{c=c_{aa},c_{bb}}\sqrt{15 + 2H\left(\frac{v_b}{(v_w-1)c} + 1\right)^2}} & \mbox{ if } v_w > 1 \\
0 & \mbox{ otherwise.}
\end{cases}
\end{equation} \label{prop-quad-cor-lim:a}
\item Let $G$ be the Gram matrix of inputs with entries $g_{ab} = x_a^\top x_b$. 
If $v_w > 1$, then the mapping $G\mapsto Q^{\times(\infty)}$ is invertible. \label{prop-quad-cor-lim:b}
\end{enumerate}
\end{proposition}

We visualize the infinite-depth quadratic correlation by plotting Eq.~\eqref{eq-square-cor-lim} as a function of $v_w$ for a single-bottleneck NNGP ($H=2$, $v_b=0.09$) with no pre-bottleneck hidden layers that is fed two inputs $x_1=(1, 0)$ and $x_2=(\cos\alpha, \sin\alpha)$ for various values of the input angle $\alpha$ (Fig.~\ref{fig-cor-deep-limit-QX}). 
Note that the covariance matrix $C$ and hence the bottleneck angle $\beta$ are themselves functions of $v_w$, which is why we plot the quadratic correlation for fixed values of $\alpha$ instead of $\beta$. 
The infinite-depth quadratic correlation exhibits interesting behavior around $v_w=1$; 
it is continuous but not differentiable at $v_w=1$. 
The bottleneck NNGP therefore undergoes a phase transition at $v_w=1$. 
The quadratic correlation tends to $0$ in the $v_w\leq 1$ regime---meaning that the outputs of a bottleneck NNGP decouple (up through second order correlations) at infinite depth---while quadratic correlation is maintained through infinitely many, infinitely wide hidden layers in the regime $v_w > 1$, though even then the limiting correlation is weak (Eq.~\eqref{eq-square-cor-bound}).

\begin{figure}
\centering
\includegraphics[width=.7\textwidth]{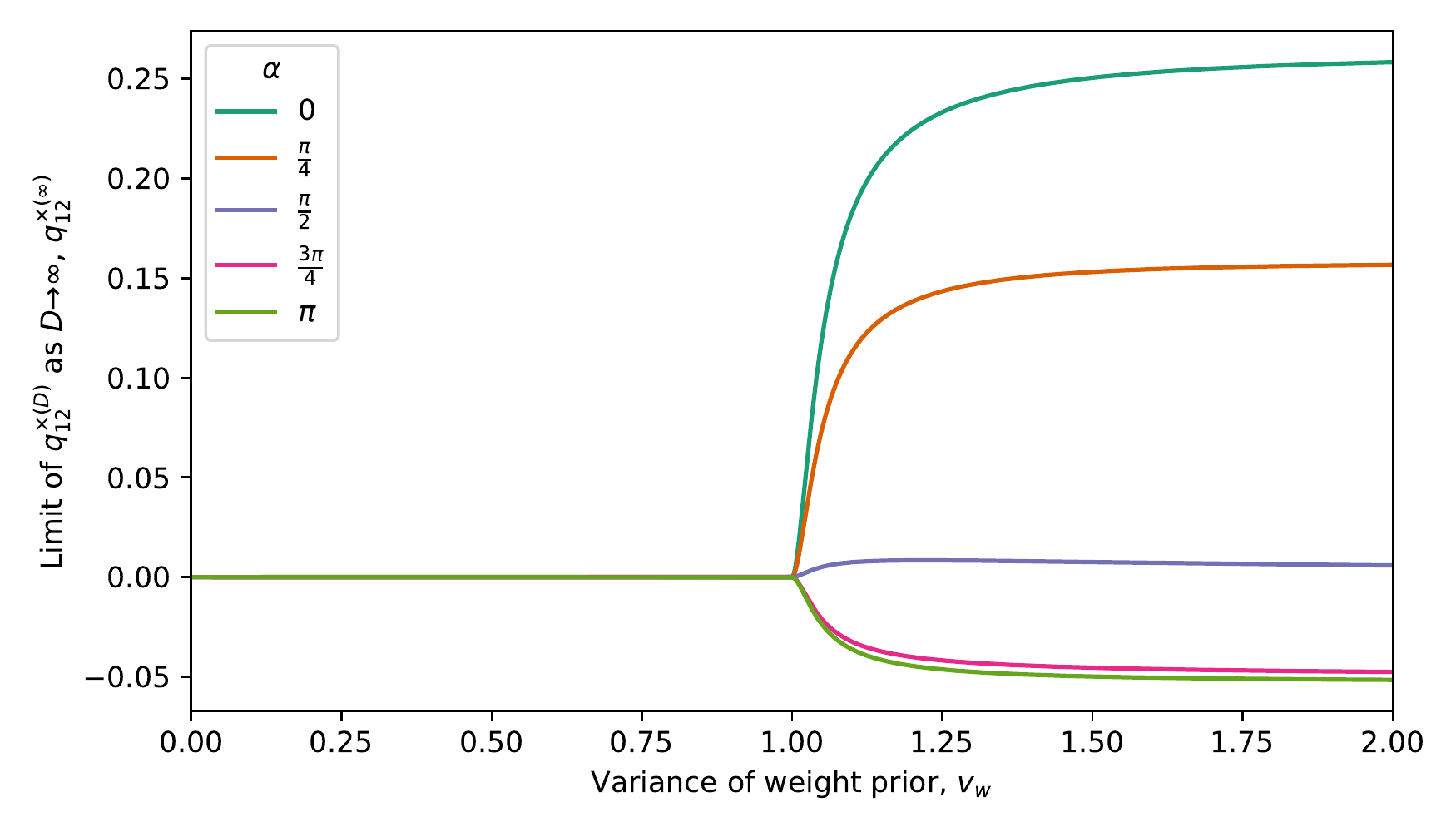}
\caption{\label{fig-cor-deep-limit-QX} %
The infinite-depth quadratic correlation between outputs as a function of the weight-variance hyperparameter $v_w$ for an example bottleneck NNGP for various angles $\alpha$ between the inputs.
}
\end{figure}

%\begin{figure}[t]
%	
%	\begin{minipage}[c]{0.625\textwidth}
%		
%		\centering
%		\includegraphics[width=1\textwidth]{cor_deep_limit_QX.pdf}
%	\end{minipage}\hfill
%	\begin{minipage}[c]{0.375\textwidth}
%		\vspace{-1.45cm}
%		\caption{The infinite-depth quadratic correlation between outputs as a function of the weight-variance hyperparameter $v_w$ for an example bottleneck NNGP for various angles $\alpha$ between the inputs.
%		}
%		\label{fig-cor-deep-limit-QX}
%	\end{minipage}
%\end{figure}

A phase transition at $v_w=1$ has already been noted in the literature in the behavior of no-bottleneck NNGP models at infinite depth~\citep{schoenholz2016deep, poole2016exponential, lee2017deep}. 
Specifically, the kernel of an NNGP with normalized ReLU activation degenerates to a constant kernel at infinite depth with a value of either $\frac{v_b}{1-v_w}$ if $v_w < 1$ and $\infty$ otherwise. 
Bottleneck layers help to reveal a richer structure of this phase transition, as we explain next. 
Drawing an analogy to the classical Ising model in statistical mechanics~\citep{baxter2016exactly}, the hyperparameter $v_w$ operates as an inverse temperature with a critical value at $v_w=1$. 
The quadratic correlation is then an order parameter analogous to magnetization whose derivative contains a discontinuity at the phase boundary $v_w=1$. 
In the $v_w < 1$ phase, the infinite-depth quadratic correlation is $0$ regardless of the bottleneck angle $\beta$; 
information about the inputs into the bottleneck NNGP is therefore lost, analogous to a disordered system at large scale. 
However, as $v_w$ crosses the phase boundary from below, the system undergoes a symmetry breaking with the infinite-depth quadratic correlation taking a distinct value for each bottleneck angle $\beta$ as well as for each input angle $\alpha$. 
This lets us recover information about the inputs from the infinite-depth quadratic correlation, 
indicating that bottlenecks help information propagate to extreme depths (Prop.~\ref{prop-quad-cor-lim}~\ref{prop-quad-cor-lim:b}). 

The symmetry breaking discussed above is not apparent in the phase transition of degenerate NNGP kernels noted in the literature; 
i.e., all information about the inputs are lost in an infinite-depth (no-bottleneck) NNGP in either phase. 
We see, however, that the restriction of just one hidden layer to a bottleneck is sufficient to break this symmetry. 
Moreover, the symmetry can be recovered by sending the bottleneck width to infinity (Thm.~\ref{thm-correspondence}); 
the bottleneck is therefore analogous to an external magnetic field in the Ising model.

\subsection{A divergent depth scale}
\label{subsection-depth-scale}

\citet{schoenholz2016deep} show that the characteristic depth scale on which the kernel of a (no-bottleneck) NNGP degenerates exponentially to its deep limit diverges at a phase boundary in $(v_b, v_w)$-space, and they use this result to argue that values of $(v_b, v_w)$ near criticality or ``at the edge of chaos'' optimize trainability by maximizing the depth to which information can penetrate in an NNGP. 
We show that  the depth scale on which the quadratic correlation $Q^{\times(D)}$ converges to its limit also diverges at the phase boundary $v_w=1$. 

\begin{proposition}[Characteristic depth scale] \label{prop-depth-scale}
Consider the single-bottleneck NNGP $F$ with normalized ReLU activation from Prop.~\ref{prop-quad-cor}, and suppose it is fed two inputs $x_1$ and $x_2$ with $\|x_1\|=\|x_2\|$ and bottleneck angle $\beta\neq \frac{\pi}{2}$. 
Then for all positive $v_w\neq 1$, the quadratic correlation $q^{\times(D)}_{ab}$ is asymptotically exponential in $D$, meaning that the limit
\begin{equation} \label{eq-asymp-exp-L}
L = \limD \frac{q^{\times(D)}_{ab}-q^{\times(\infty)}_{ab}}{e^{-\frac{D}{\lambda}}}
\end{equation}
exists and is finite and non-zero for some $\lambda > 0$, which we find to be
\begin{equation} \label{eq-depth-scale}
\lambda = 
\begin{cases}
\ln\left(\frac{1}{v_w^2}\right)^{-1} & \mbox{ if } v_w < 1 \\
\ln(v_w)^{-1} & \mbox{ if } v_w > 1.
\end{cases}
\end{equation}
In the case $v_w=1$, the limit $L$ is infinite for all finite $\lambda > 0$.
\end{proposition}
%\prooflink{prop-depth-scale}
Proposition~\ref{prop-depth-scale} excludes the case of orthogonal bottleneck preactivations ($\beta=\frac{\pi}{2}$) since it leads to trivial asymptotic behavior ($q^{\times(D)}_{ab}=0$ for all $D$). 
The quantity $\lambda$ given in Eq.~\eqref{eq-depth-scale} is called the characteristic depth scale and is proportional to the ``half-life'' of quadratic correlation at large depth. 
The divergence of the depth scale at $v_w=1$ gives us another perspective on the phase transition in bottleneck NNGP models (Fig.~\ref{fig-cor-depth-scale}). 
Based on this, we expect optimal values of the $v_w$ hyperparameter to be greater than but close to $1$; 
although $v_w=1$ gives the smallest decay rate (i.e., largest depth scale) to the infinite-depth quadratic correlation, this limiting value is $0$. 
Larger values of $v_w$ admit non-zero infinite-depth quadratic correlations, but values that are too large lead to fast decay rates (i.e., small depth scales).
We hypothesize that this tension between large depth scales (near $v_w=1$) and non-zero quadratic correlations ($v_w\gg 1$) is the main driving force determining the optimal value of $v_w$ in the $v_w>1$ phase.

\begin{figure}
\centering
\includegraphics[width=.7\textwidth]{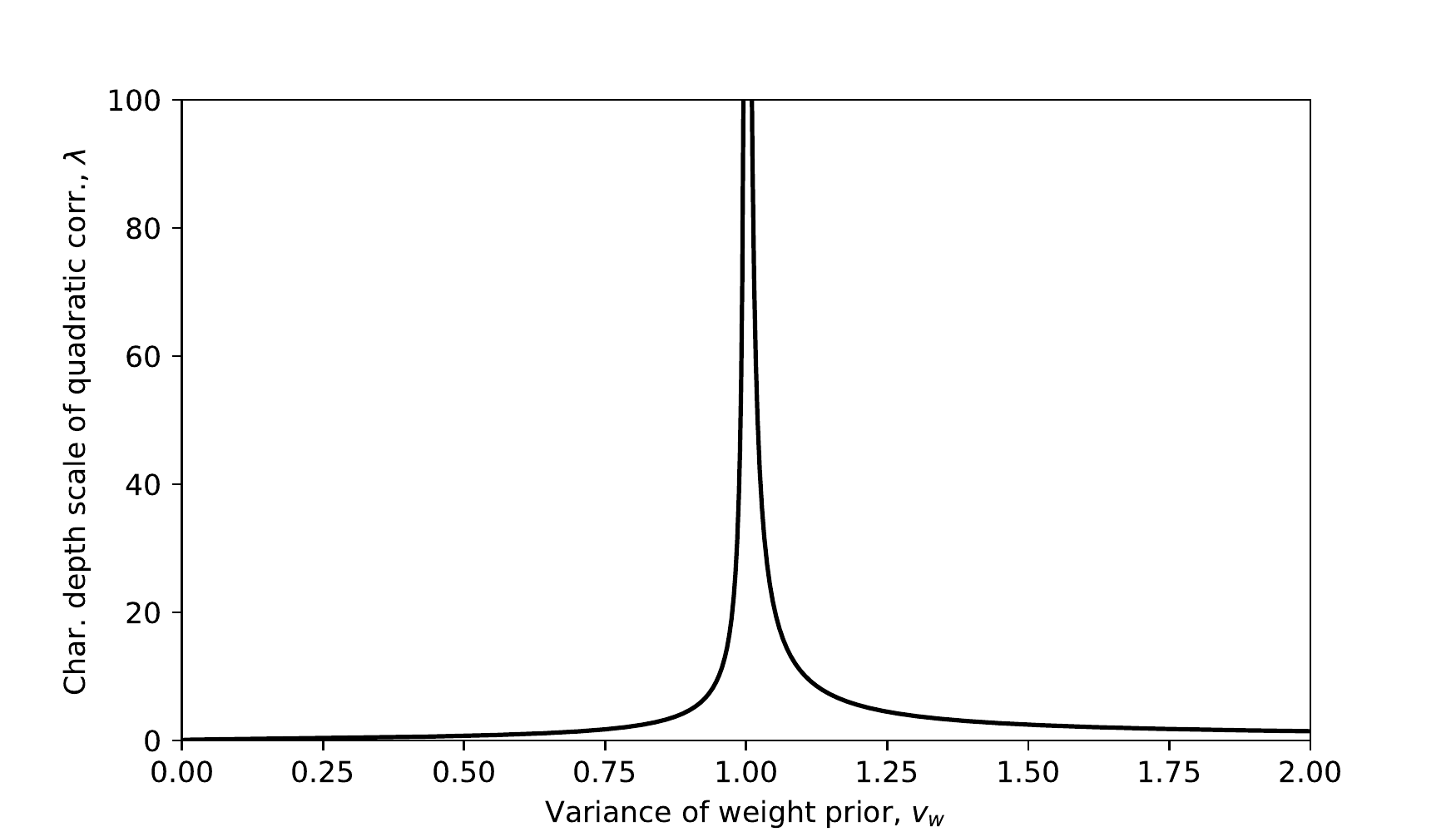}
\caption{\label{fig-cor-depth-scale} %
Characteristic depth scale of the convergence of the quadratic correlation between outputs to its infinite-depth limit as a function of the weight-variance hyperparameter $v_w$.
}
\end{figure}

%\begin{figure}[t]
%	
%	\begin{minipage}[c]{0.625\textwidth}
%		
%		\centering
%		\includegraphics[width=1\textwidth]{cor_depth_scale.pdf}
%	\end{minipage}\hfill
%	\begin{minipage}[c]{0.375\textwidth}
%		\vspace{-1.05cm}
%		\caption{Characteristic depth scale of the convergence of the quadratic correlation between outputs to its infinite-depth limit as a function of the weight-variance hyperparameter $v_w$.
%		}
%		\label{fig-cor-depth-scale}
%	\end{minipage}
%\end{figure}

\subsection{Non-degenerate kernels at extreme depths}
\label{subsection-nondegenerate}

The non-trivial dependence of the infinite-depth quadratic correlation $q^{\times(\infty)}_{ab}(\beta)$ on the bottleneck angle $\beta$ has remarkable implications for the kernel or covariance matrix $K^{(\infty)}$ of individual output neurons of a bottleneck NNGP at infinite depth. 
Consider again the single-bottleneck NNGP $F$ described at the beginning of Sec.~\ref{subsection-cor-exact-formula}, and suppose it is fed two linearly independent inputs $x_1$ and $x_2$. 
Assume that the infinite-depth kernel $K^{(\infty)}$ is degenerate so that the associated correlation matrix $\hat{K}^{(\infty)}$ is a matrix of ones; 
this is indeed the case when there are no bottlenecks, where the correlations in $\hat{K}$ tend to 100\% even when all elements of the kernel $K$ grow to infinity at infinite depth. 
Then the outputs $F_i(x_1)$ and $F_i(x_2)$ of a single output neuron $i$ at infinite depth are linearly dependent and are in fact equal. 
It follows that the elements $q^{\times(\infty)}_{ab}$ of the quadratic correlation matrix $Q^{\times(\infty)}$ between distinct output neurons are all equal, but this is impossible in the $v_w>1$ phase since $q^{\times(\infty)}_{12}(\beta)$ is a one-one function of $\beta$ (Fig.~\ref{fig-cor-deep-limit-QX}) and the diagonal elements assume $\beta=0$ while the off-diagonal elements assume $\beta > 0$ (since $x_1$ and $x_2$ are linearly independent). 
We therefore learn that the kernel of an NNGP does not degenerate to a constant in the $v_w>1$ phase if at least one hidden layer is bottlenecked. 

Unfortunately, the kernel---and thus the associated correlation matrix---of each output neuron of a bottleneck NNGP do not admit closed forms, even in the infinite depth limit. 
The quadratic correlation matrix $Q^{(D)}$ (without a superscript $\times$) for individual outputs is intractable as well, but its infinite depth limit does admit an elegant closed form. 

\begin{proposition}[Infinite-depth quadratic correlation for single output]\label{prop-q}
Consider the single-bottleneck NNGP $F$ with normalized ReLU activation from Prop.~\ref{prop-quad-cor}, and suppose we send the post-bottleneck depth to infinity.
\begin{enumerate}[label={(\alph*)}]
\item The infinite-depth quadratic correlation matrix $Q^{(\infty)}$ of a single output neuron has $(a, b)$-th element
$q^{(\infty)}_{ab} = \limD q^{(D)}_{ab}$ given by
\begin{equation}
\label{eq-prop-q}
q^{(\infty)}_{ab} = 
% \limD q^{(D)}_{ab} =
\begin{cases}
% 3q^{\times (\infty)}_{ab} + \frac{\left(\frac{v_b}{(v_w-1)c_{aa}}+1\right)\left(\frac{v_b}{(v_w-1)c_{bb}}\right)}{\sqrt{\left(\frac{15}{2H} + \left(\frac{v_b}{(v_w-1)c_{aa}} + 1\right)^2\right)\left(\frac{15}{2H} + \left(\frac{v_b}{(v_w-1)c_{bb}} + 1\right)^2\right)}} & \mbox{ if } v_w > 1 \\
3q^{\times (\infty)}_{ab} + \frac{\left(\frac{v_b}{(v_w-1)c_{aa}}+1\right)\left(\frac{v_b}{(v_w-1)c_{bb}}+1\right)}{\displaystyle\prod_{c=c_{aa},c_{bb}}\sqrt{\frac{15}{2H} + \left(\frac{v_b}{(v_w-1)c} + 1\right)^2}} & \mbox{ if } v_w > 1 \\
1 & \mbox{ otherwise.}
\end{cases}
\end{equation} \label{prop-q:a}
\item Let $G$ be the Gram matrix of inputs with entries $g_{ab} = x_a^\top x_b$. 
If $v_w > 1$, then the mapping $G\mapsto (Q^{(\infty)}, \operatorname{diag}(G))$ is invertible. \label{prop-q:b}
\end{enumerate}
\end{proposition}
%\prooflink{prop-q}

The infinite-depth single-output quadratic correlation $Q^{(\infty)}$ carries many of the same properties as the infinite-depth between-output quadratic correlation $Q^{\times(\infty)}$. 
Recalling that $q^{\times(\infty)}_{ab}\rightarrow 0$ as $H\rightarrow\infty$, it is easy to verify that $q^{(\infty)}_{ab}\rightarrow 1$. 
%Although this limit alone does not imply that the kernel degenerates to a constant at infinite depth as we would expect in the absence of bottlenecks, this is likely the case. 
The single-output quadratic correlation also exhibits the same symmetry breaking at the phase boundary $v_w=1$ as the between-output quadratic correlation; 
this is evident in the example plot of $q^{(\infty)}_{12}(\beta)$ as a function of $v_w$, using the same setup as for Fig.~\ref{fig-cor-deep-limit-QX} (see Fig.~\ref{fig-cor-deep-limit-Q}).

\begin{figure}
\centering
\includegraphics[width=.7\textwidth]{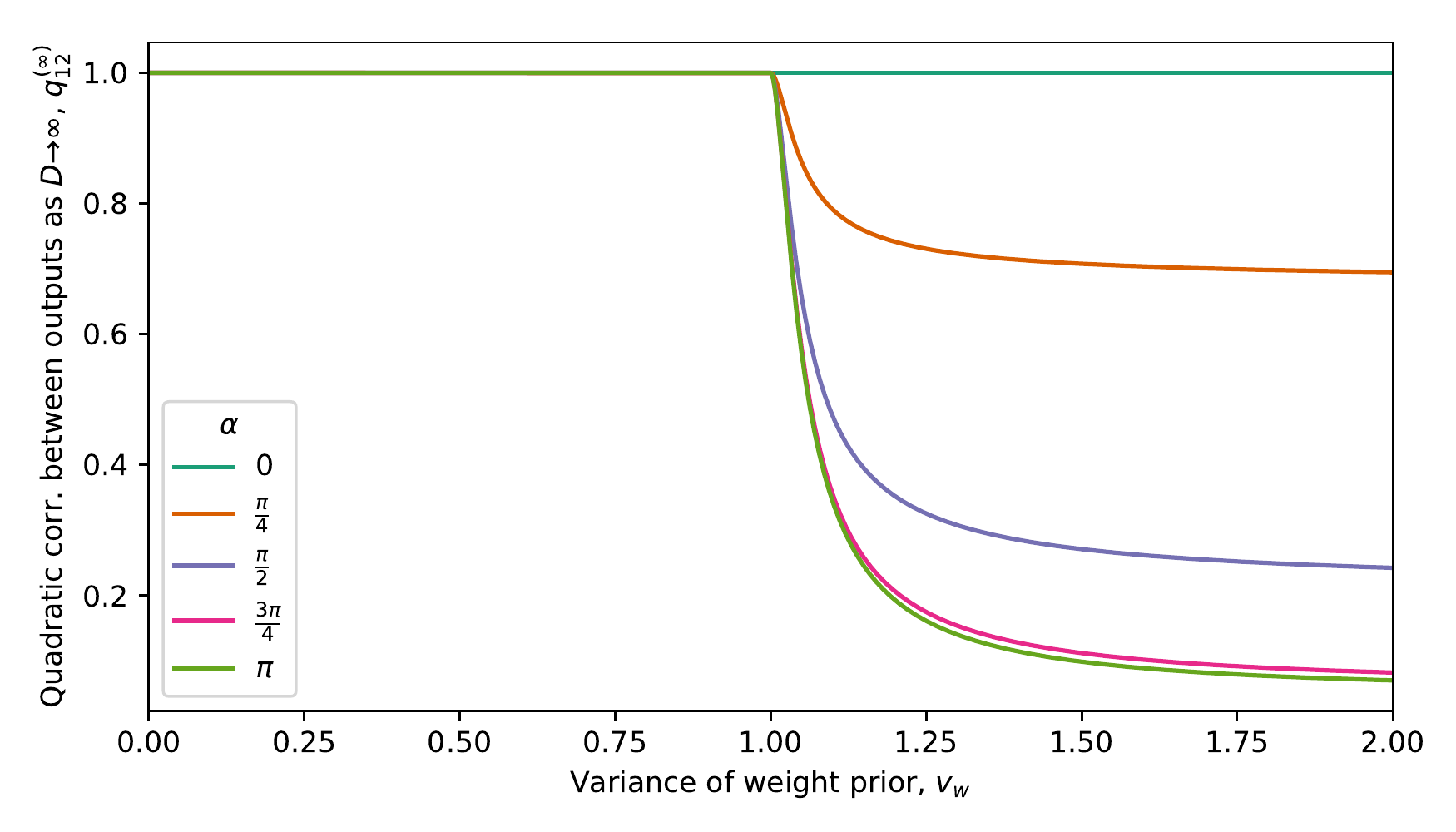}
\caption{\label{fig-cor-deep-limit-Q} %
The infinite-depth quadratic correlation of a single output as a function of the $v_w$ hyperparameter for an example bottleneck NNGP.
}
\end{figure}

%\begin{figure}[t]
%	
%	\begin{minipage}[c]{0.625\textwidth}
%		
%		\centering
%		\includegraphics[width=1\textwidth]{cor_deep_limit_Q.pdf}
%	\end{minipage}\hfill
%	\begin{minipage}[c]{0.375\textwidth}
%		\vspace{-2.25cm}
%		\caption{The infinite-depth quadratic correlation of a single output as a function of the $v_w$ hyperparameter for an example bottleneck NNGP.
%		}
%		\label{fig-cor-deep-limit-Q}
%	\end{minipage}
%\end{figure}

In the phase $v_w < 1$, the single-output quadratic correlation is 100\%, suggesting degeneracy at infinite depth. 
In contrast, in the phase $v_w>1$---where symmetry breaks---$q^{(\infty)}_{ab}$ becomes a strictly increasing function of the input angle $\alpha$, allowing us to recover some information about the inputs (Prop.~\ref{prop-q}~\ref{prop-q:b}). 
In particular, Prop.~\ref{prop-q}~\ref{prop-q:b} implies that in the ordered phase $v_w > 1$, if the norms of the inputs are known (if the inputs are constrained to a sphere, for example), then the input angle can be recovered from the single-output quadratic correlation even at infinite depth. 
This stands in stark contrast to no-bottleneck NNGP models, where all information is lost at infinite depth regardless of the phase, and it suggests that bottleneck layers are vital for the trainability of very deep models.

\subsection{Other nonlinearities}
\label{subsection-other-nonlin}

The behavior described in Prop.~\ref{prop-q} does not extend to other common nonlinearities and is thus all the more striking. 
Two notable examples are the sigmoidal nonlinearity $\phi(x) = \tanh x$ and the sinusoidal nonlinearity 
\begin{equation} \label{eq-sinusoidal}
\phi(x) = \cos x + \sin x.
\end{equation}
The significance of the latter is that the corresponding NNGP has an RBF kernel; 
the random features literature~\citep{rahimi2008random, cutajar2017random} hints at this connection but does not discuss it in the context of NNGPs or neural network nonlinearities. 
We make this connection more precise in the next proposition and subsequent remark 
(see Appendix~\ref{appendix-section-other-nonlin} for the proofs of both Props.~\ref{prop-rbf-nngp},~\ref{prop-rbf-nngp-bottleneck}).

\begin{proposition}[RBF-NNGP kernel recursion] \label{prop-rbf-nngp}
Consider an NNGP mapping $\RR^M$ to $\RR$ with $D$ hidden layers and the sinusoidal nonlinearity in Eq.~\eqref{eq-sinusoidal}. 
\begin{enumerate}[label={(\alph*)}]
\item The NNGP kernel $K^{(\mu)}:\RR^M\times\RR^M\mapsto\RR$ for the $\mu$-th hidden layer is given recursively as
\begin{align}
K^{(1)}(x, x^\prime) &= v_b + v_w x^\top x^\prime \label{eq-rbf-nngp-K1} \\
K^{(\mu+1)}(x, x^\prime) &= v_b + v_w e^{-\frac{1}{2}\left[K^{(\mu)}(x, x)+K^{(\mu)}(x^\prime, x^\prime) - 2K^{(\mu)}(x, x^\prime)\right]}, \label{eq-rbf-nngp-Kmu}
\end{align}
for $\mu=1,\ldots,D$ where $K^{(D+1)}$ is the kernel of the output layer. \label{prop-rbf-nngp:a}
\item In the deep limit $D\rightarrow\infty$, the NNGP kernel converges pointwise to
\begin{equation} \label{eq-rbf-nngp-Kinfty}
K^{(\infty)}(x, x^\prime) = v_*(v_b, v_w)
\begin{cases}
1, & \mbox{ if } x=x^\prime \\
c_*(v_b, v_w), & \mbox{ otherwise,}
\end{cases}
\end{equation}
where
\begin{align}
v_*(v_b, v_w) &= v_b+v_w \label{eq-rbf-nngp-v} \\
c_*(v_b, v_w) &=
\begin{cases}
1, & \mbox{ if } v_w < 1 \\
c^\prime, & \mbox{ if } v_w > 1.
\end{cases} \label{eq-rbf-nngp-c}
\end{align} \label{prop-rbf-nngp:b}
\end{enumerate}
\end{proposition}

\begin{remark}[RBF-DGP as a bottleneck NNGP]
Since the sinusoidal nonlinearity in Eq.~\eqref{eq-sinusoidal} clearly satisfies the linear envelope condition, 
then we can apply Thm.~\ref{thm-bottleneck} to a BNN with the sinusoidal nonlinearity and obtain a bottleneck NNGP limit. 
If no bottlenecks are imposed, then the wide limit of such a BNN is the NNGP described in Prop.~\ref{prop-rbf-nngp}. 
In particular, using Eqs.~\eqref{eq-rbf-nngp-K1}-\eqref{eq-rbf-nngp-Kmu}, the single-hidden-layer NNGP with sinusoidal nonlinearity has kernel
\begin{align*}
K^{(2)}(x, x^\prime)
&= v_b + v_w e^{-\frac{1}{2}[K^{(1)}(x, x) + K^{(1)}(x^\prime, x^\prime) - 2K^{(1)}(x, x^\prime)]} \\
&= v_b + v_w e^{-\frac{1}{2}[v_b+v_w\Vert x\Vert^2+v_b+v_w\Vert x^\prime\Vert^2 - 2(v_b+v_wx\cdot x^\prime)]} \\
%&= v_b + v_w e^{-\frac{v_w}{2}(\Vert x\Vert^2+\Vert x^\prime\Vert^2 - 2x\cdot x^\prime)} \\
&= v_b + v_w e^{-\frac{v_w}{2}\Vert x-x^\prime\Vert^2},
\end{align*}
which we recognize as the RBF kernel. 
More generally, if the BNN with sinusoidal nonlinearity has an odd number of hidden layers $D$ and the $\mu$-th hidden layers are restricted to bottlenecks for $\mu=2,4,6,\ldots,D-1$, 
then the bottleneck NNGP limit is a DGP with $\frac{D+1}{2}$ GP components each with RBF kernel (and sinusoidal nonlinearities applied to the bottleneck layers).
\end{remark}

The deep limit of an NNGP with sinusoidal nonlinearity is described in Prop.~\ref{prop-rbf-nngp}~\ref{prop-rbf-nngp:b}. 
It is identical to the behavior of an NNGP with sigmoidal nonlinearity as described by~\citet{poole2016exponential}, 
except that $v_*(v_b, v_w)$ and $c_*(v_b, v_w)$ take different forms and the phase boundary has a different location. 
In one phase ($v_w<1$ for the sinusoidal nonlinearity), all inputs tend to $100\%$ correlation as depth is increased without bound---similar to the ReLU activation. 
However, in the other phase ($v_w>1$ for the sinusoidal nonlinearity)---the ``chaotic phase''---the infinite-depth correlation for distinct inputs is a constant less than $100\%$; 
thus, unlike the ReLU activation, the sigmoidal and sinusoidal nonlinearities allow distinct inputs to remain distinct through infinite depth, 
although all information about the distince between distinct inputs is lost. 
Moreover, the introduction of a bottleneck does not remove this degeneracy---in sharp contrast to Prop.~\ref{prop-q}. 
We substantiate this with the next proposition.

\begin{proposition}[Deep post-bottleneck limit for sigmoidal and sinusoidal nonlinearities] \label{prop-rbf-nngp-bottleneck}
Consider a single-bottleneck NNGP $F$ mapping $\RR^M$ to $\RR$  with either sigmoidal or sinusoidal nonlinearity $\phi:\RR\mapsto\RR$. 
Suppose IID Gaussian noise $\calN(0, v_n)$ is added to the outputs for any $v_n>0$. 
Given two distinct inputs $x_1,x_2\in\RR^M$, let $p^{(D)}:\RR^2\mapsto [0, \infty)$ be the PDF of $\{F(x_1), F(x_2)\}$, 
where $D$ is the post-bottleneck depth. 
Then in the deep post-bottleneck limit $D\rightarrow\infty$, the PDF converges pointwise to
\begin{equation} \label{eq-rbf-nngp-pinfty}
p^{(\infty)}(y) = \calN(y; 0, K^{(\infty)}(X, X)+v_n I),
\end{equation}
where $K^{(\infty)}(X, X)$ is a $2\times 2$ matrix with entries $K^{(\infty)}(x_a, x_b)$ with $K^{(\infty)}$ given in Eq.~\eqref{eq-rbf-nngp-Kinfty}.
\end{proposition}

Note that Eq.~\eqref{eq-rbf-nngp-pinfty} is independent of bottleneck width; 
in fact, it is exactly the PDF of a no-bottleneck NNGP, 
indicating that the bottleneck has no effect at all on the deep post-bottleneck limit given either the sigmoidal or sinusoidal nonlinearity. 
In contrast to Prop.~\ref{prop-q}~\ref{prop-q:b} for the ReLU activation, 
the only information about the inputs that can be recovered from the deep post-bottleneck limit is whether the inputs are distinct. 
This suggests a fundamental difference between the ReLU activation and the sigmoidal nonlinearity and between the NNGP kernel with ReLU activation and the RBF kernel---namely, that the ReLU activation allows a network to operate at very large depths as long as one bottleneck is present. 
Classifying activation functions based on their deep limit and deep post-bottleneck limit behaviors could help to better understand which activation functions are useful in practice and is a topic for future work.

\subsection{Other deep limits}
\label{subsection-other}

Besides the deep post-bottleneck limit discussed in Secs.~\ref{subsection-quad-cor-lim}-\ref{subsection-nondegenerate}, there are two other limits we could have considered: 
1) The deep pre-bottleneck limit where a single-bottleneck NNGP has infinitely many infinitely wide hidden layers before its bottleneck but only finitely many after, and 
2) the doubly deep bottleneck limit where the single-bottleneck NNGP has infinitely many infinitely wide hidden layers both before and after its bottleneck. 
The quadratic correlation between the activations of distinct output neurons can be analyzed in both of these limits by replacing the pre-bottleneck NNGP covariance matrix $C$ in Prop.~\eqref{prop-quad-cor} with a sequence of such covariance matrices indexed by the pre-bottleneck depth. 
Since the NNGP kernel with ReLU activation tends to a constant in the deep limit, then the bottleneck angle $\beta$ tends to $0$ in both the deep pre-bottleneck limit and doubly deep bottleneck limit. 
Although the quadratic correlation remains nontrivial in both limits and even exhibits a first-order phase transition at $v_w=1$ in the doubly deep bottleneck limit, these results are uninteresting as all information about the original inputs into the network is lost due to vanishing bottleneck angle.

The quadratic correlation between the activations of a single output neuron at two different inputs can also be analyzed in both the deep pre-bottleneck and doubly deep bottleneck limits; 
calculations proceed similarly as in the proof of Prop.~\ref{prop-q}. 
Again, since the bottleneck angle tends to $0$, we find that the quadratic correlation tends to 100\% in both limits at every post-bottleneck depth; 
the bottleneck NNGP kernel therefore degenerates, unlike in the deep post-bottleneck limit.
We conclude that out of the three possible deep limits, only the deep post-bottleneck limit is interesting since it is the only one that admits a phase ($v_w>1$) in which information about network inputs is preserved.

\section{Conclusion}

Our main theorem, Thm.~\ref{thm-bottleneck}, generalizes the result of \citetalias{matthews2018gaussian} concerning deep neural networks whose hidden layer widths are increased without bound to a setting in which some intermediate hidden layers, called bottlenecks, are fixed to a finite width. 
From a theoretical perspective, this result connects the NNGP literature with that of DGPs, as the resulting probability model is in fact a DGP consisting of a composition of NNGP components. 
Additionally, we have explored the effect of these bottleneck layers on the resulting probability model from a practical perspective, showing that model likelihood peaks at a finite bottleneck width and is superior to that of no-bottleneck NNGPs.

Surprisingly, in contrast to no-bottleneck NNGP models, the behavior of a bottleneck NNGP with ReLU activation at extreme post-bottleneck depths is not always degenerate (Props.~\ref{prop-quad-cor-lim},~\ref{prop-q}); 
in particular, the input Gram matrix can be fully recovered from the between-output quadratic correlation matrix of the bottleneck NNGP at infinite post-bottleneck depth, 
and the input angle can be recovered even from the single-output quadratic correlation matrix at infinite depth if the input norms are known. 
Bottleneck layers are therefore fundamental as they allow networks to ``go deeper''. % ; 
%indeed, we observed empirically that not only does the model likelihood decrease as the bottleneck width is increased beyond a certain point but that the optimal depth decreases as well. 
However, this non-degeneracy in the deep limit manifests only when the network weight prior is weaker than a standard normal. 
We have just begun to explore the dependence of the deep post-bottleneck limit on the prior weight variance, showing that convergence to the limit is asymptotically exponential in depth and that the characteristic depth scale diverges at a critical value of the weight prior variance.

So far, we have not directly connected bottleneck NNGPs with BNNs, aside from studying the limits of BNNs in the wide regime. 
However, an interesting special case of bottleneck NNGPs emerges when \emph{every} hidden layer is fixed at finite width. 
The result is a BNN, but our result lends an interesting perspective that suggests that one might approach BNN inference with a non-parametric DGP-based approach. 
Indeed, in follow-up work we intend to explore the practicality of such a method. 
Moreover, we believe that our main theorem can easily be extended to convolutional architectures (by introducing a second index for each hidden layer), 
and thus we plan to explore the implications of our work for convolutional BNNs.

Finally, in this work we did not consider the implications of the bottleneck NNGP for the learning dynamics of Gaussian-initialized deterministic neural networks (DNNs). 
It is now a celebrated result that the evolution over training time of a Gaussian-initialized DNN is described in the wide limit by an exactly solvable linear ODE, 
where the time evolution operator is termed the neural tangent kernel (NTK) and is related to the NNGP kernel~\citep{jacot2018neural}. 
As part of additional follow-up work, we plan to investigate the ``bottleneck NTK'', 
where the wide limit is relaxed to allow for some bottleneck layers. 
The result is a system of coupled ODEs that is more challenging to analyze but carries potential for a more refined description of the evolution of finite-width DNNs.

%\acks{Anonymized}
%\begin{comment}
\acks{
    This work has been supported in part by the Joint Design of Advanced
    Computing Solutions for Cancer (JDACS4C) program established by the U.S.
    Department of Energy (DOE) and the National Cancer Institute (NCI) of
    the National Institutes of Health.
    This work was performed under the auspices of the U.S. Department of
    Energy by Argonne National Laboratory under Contract DE-AC02-06-CH11357,
    Lawrence Livermore National Laboratory under Contract DEAC52-07NA27344,
    Los Alamos National Laboratory under Contract DE-AC5206NA25396, and Oak
    Ridge National Laboratory under Contract DE-AC05-00OR22725.
    Research sponsored by the Laboratory Directed Research and Development
    Program of Oak Ridge National Laboratory, managed by UT-Battelle, LLC, for
    the US Department of Energy under contract DE-AC05-00OR22725.
    This research used resources of the Compute and Data Environment for Science
    (CADES) at the Oak Ridge National Laboratory, which is supported by the
    Office of Science of the U.S. Department of Energy under Contract
    No. DE-AC05-00OR22725.
    The United States Government retains and the publisher, by accepting the
    article for publication, acknowledges that the United States Government
    retains a non-exclusive, paidup, irrevocable, world-wide license to
    publish or reproduce the published form of the manuscript, or allow
    others to do so, for United States Government purposes.
    The Department of Energy will provide public access to these results of
    federally sponsored research in accordance with the DOE Public Access
    Plan (http://energy.gov/downloads/doe-public-access-plan).
}
%\end{comment}

\appendix

% !TeX root = main.tex

\section{Convergence of stochastic processes}
\label{appendix-section-convergence}

Here we define three notions of convergence that are important for formally stating and proving the bottleneck NNGP limit. 
The first notion of convergence is the one that \citetalias{matthews2018gaussian} employ to prove the (no-bottleneck) NNGP limit. 
They consider a sequence of BNNs mapping $\RR^M$ to $\RR^L$ restricted to a countable set of inputs $\mathcal{X}\subset \RR^M$. 
Each BNN is then equivalent to a random sequence of $L$-vectors, i.e. a random variable taking values in $(\RR^L)^{\infty}$.
Following \citet{billingsley1999convergence}, \citetalias{matthews2018gaussian} equip the sequence space $(\RR^L)^{\infty}$ with the metric
\begin{equation*} %\label{eq-rho}
\rho(s, t) = \sum_{k=1}^\infty 2^{-k} \operatorname{min}(1, \|s_k-t_k\|),
\end{equation*}
where $s, t\in (\RR^L)^{\infty}$.
The metric $\rho$ induces the Euclidean topology on $(\RR^L)^{\infty}$, and we endow $(\RR^L)^{\infty}$ with the associated Borel algebra $\mathcal{A}$, giving us a measurable space of sequences. 
Demonstrating convergence in $((\RR^L)^{\infty}, \mathcal{A})$ can prove challenging, but \citetalias{matthews2018gaussian} simplify the task by invoking the following theorem.

\begin{theorem}[\citet{billingsley1999convergence}] \label{thm-marginal}
A sequence of stochastic processes $\{F[n]: \mathcal{X}\times\Omega\mapsto \RR^L\}_{n=1}^\infty$ with countable index set $\mathcal{X}$ converges in distribution to a stochastic process $F: \mathcal{X}\times\Omega\mapsto \RR^L$ on the measurable space $((\RR^L)^{\infty}, \mathcal{A})$ if and only if every sequence of finite-dimensional marginals $\{(F(x_1)[n], \ldots, F(x_T)[n])\}_{n=1}^\infty$ converges in distribution to the corresponding limiting marginal $(F(x_1), \ldots, F(x_T))$.
\end{theorem}

Theorem~\ref{thm-marginal} effectively reduces the task of proving the convergence of a sequence of random sequences to that of a sequence of random vectors. 
When looking at the convergence of $T$-dimensional marginal distributions, it is convenient to introduce the following notation.

\begin{definition}[Batch stochastic process] \label{defn-bsproc}
Let $F: \RR^M\times\Omega\mapsto \RR^L$ be a stochastic process. 
Then the batch stochastic process of size $T\in\NN$ associated with $F$ is the stochastic process $\tilde{F}: (\RR^M)^T\times\Omega\mapsto (\RR^L)^T$ defined by
\begin{equation*}
\tilde{F}(\{x_t\}_{t=1}^T, \omega) = \{F(x_t, \omega)\}_{t=1}^T.
\end{equation*}
\end{definition}

By working with batch stochastic processes, we can think of $T$ inputs as constituting a single input. 
Thus, to show $\{F[n]\}_{n=1}^\infty\rightarrow F$, it is enough to show $\{\tilde{F}(x)[n]\}_{n=1}^\infty\rightarrow \tilde{F}(x)$ for each input $x$ for every batch size $T$. 

Our proof of the bottleneck NNGP limit takes the approach of showing that each component of a BNN after the first bottleneck layer converges to an NNGP in the wide limit with some uniformity. 
We specify the appropriate notion of uniform convergence next.

\begin{definition}[Uniform convergence in distribution~\citep{sweeting1980uniform}] \label{defn-ud}
A sequence of stochastic processes $\{F[n]: 
X\times\Omega\mapsto \RR^L\}_{n=1}^\infty$ is said to converge in distribution to $\{F: 
X\times\Omega\mapsto \RR^L\}$ uniformly on $X$ if for every continuity set $U\subseteq \RR^L$ of $F$, i.e., a set satisfying
\[ \Pr(F(x)\in \partial U) = 0 \mbox{ for all } x\in X, \]
we have the limit
\[ \limn \Pr(F(x)[n]\in U) = \Pr(F(x)\in U) \mbox{ uniformly for all } x \in X. \]
We denote this by $F[n] \xrightarrow{UD} F$.
\end{definition}

Note that uniform convergence in distribution is distinct from and is not a stronger version of convergence in distribution in $((\RR^L)^{\infty}, \mathcal{A})$ since the former notion concerns only the singly-indexed marginals of a stochastic process while the latter deals with the joint distribution of all elements of a stochastic process. 
However, the uniform convergence in distribution of the batch stochastic processes $\{\tilde{F}[n]\}_{n=1}^\infty$ is stronger than the convergence in distribution of the original stochastic processes $\{F[n]\}_{n=1}^\infty$. 
Proving uniform convergence in distribution can be challenging, but fortunately there is another closely related notion of convergence, which we define next.

\begin{definition}[Continuous convergence in distribution~\citep{sweeting1980uniform}] \label{defn-cd}
Let $X$ be a topological space. 
A sequence of stochastic processes $\{F[n]: 
X\times\Omega\mapsto\RR^L\}_{n=1}^\infty$ is said to converge in distribution to $\{F: 
X\times\Omega\mapsto\RR^L\}$ continuously on $X$ if for every $x\in X$ and sequence $\{x_n\in X\}_{n=1}^\infty$ converging to $x$, the sequence of random variables $\{F(x_n)[n]\}_{n=1}^\infty$ converges in distribution to $F(x)$. 
We denote this by $F[n] \xrightarrow{CD} F$.
\end{definition}

Uniform and continuous convergence in distribution are related through the following proposition.

\begin{proposition}[\citet{saikkonen1993continuous}] \label{prop-uc}
Let $X$ be a topological space. 
Let $\{F[n]: 
X\times\Omega\mapsto\RR^L\}_{n=1}^\infty$ be a sequence of stochastic processes and $\{F: 
X\times\Omega\mapsto\RR^L\}$ a stochastic process. 
Then the following are equivalent statements:
\begin{enumerate}[label={(\alph*)}]
\item $F[n] \xrightarrow{CD} F$. \label{prop-uc:a}
\item $F[n] \xrightarrow{UD} F$ on every compact subset of $X$ and $x\mapsto \Pr(F(x)\in U)$ is a continuous function for every continuity set $U$ of $F$. \label{prop-uc:b}
\end{enumerate}
\end{proposition}

We should mention that \citet{sweeting1980uniform} and \citet{saikkonen1993continuous} do not state Defs.~\ref{defn-ud}-\ref{defn-cd} and Prop.~\ref{prop-uc} directly but instead define and work with the equivalent notions of uniform and continuous weak convergence of probability measures. 
The proof of Prop.~\ref{prop-uc} is a simple application of the equivalence of uniform and continuous convergence of real-valued functions on compact sets.

\section{The bottleneck NNGP theorem}
\label{appendix-section-bottleneck}

Here we list the proof of our main theorem (Thm.~\ref{thm-bottleneck}) along with all supporting lemmas.
We give a sketch of the proof next, highlighting the role of each lemma and its position in the general proof strategy.

The first step is to apply Thm.~\ref{thm-marginal} so that it is sufficient to prove convergence of BNNs to a bottleneck NNGP restricted to an arbitrary finite set of inputs. 
Since each component $F^{(d)}[n]$ in Thm.~\ref{thm-bottleneck} is being evaluated at $T$ inputs, then it is convenient to utilize the concept of a batch stochastic process of size $T$ (Def.~\ref{defn-bsproc}). 
By working with batch BNNs, we can think of $T$ inputs as constituting a single ``batch'' input. 
This reduces our task to proving the convergence of batch BNNs to a batch bottleneck NNGP given a single arbitrary input.

The next step is to find sufficient conditions under which the distributional limit of an element-wise composition of a sequence of stochastic processes with a sequence of random variables 
equals the composition of the limiting stochastic process with the limiting random variable. 
This trick can then be iterated via induction to prove that the limit of compositions is the composition of limits for stochastic processes. 
Lemma~\ref{lemma-Fn-Xn} provides such sufficient conditions, which include a notion of uniform convergence in distribution (Def.~\ref{defn-ud}).

Proving Thm.~\ref{thm-bottleneck} now comes down to verifying the conditions of Lemma~\ref{lemma-Fn-Xn}. 
Condition~\ref{lemma-Fn-Xn:a} is given to us by Thm.~\ref{thm-matthews} for single-bottleneck networks and holds by induction for multi-bottleneck networks. 
Condition~\ref{lemma-Fn-Xn:b} is also immediate in the single-bottleneck case, although it is less obvious for multi-bottleneck architectures. 
We verify condition~\ref{lemma-Fn-Xn:b} directly in the proof of Thm.~\ref{thm-bottleneck} but take aid from Lemma~\ref{lemma-unif-bound-nn-var}. 
Lemma~\ref{lemma-cont} establishes condition~\ref{lemma-Fn-Xn:c}, which amounts to showing that the NNGP kernel is a continuous function. 
Condition~\ref{lemma-Fn-Xn:d} is the trickiest to verify; 
it states that the outer sequence of stochastic processes (that is composed with an inner sequence of random variables) must converge in distribution uniformly (Def.~\ref{defn-ud}) on compact sets, meaning that the rate of convergence in distribution should be independent of the input to the stochastic processes. 
Condition~\ref{lemma-Fn-Xn:d} is verified with the help of Lemma~\ref{lemma-tildeFn}.

Lemma~\ref{lemma-tildeFn} is a direct generalization of Lemma 12 in \citetalias{matthews2018gaussian}; 
the latter states that given a fixed finite batch of inputs, BNNs with no bottlenecks converge in distribution to an NNGP in the wide limit. 
In Lemma~\ref{lemma-tildeFn}, we strengthen the mode of convergence to continuous convergence in distribution (Def.~\ref{defn-cd}).
More specifically, Lemma~\ref{lemma-tildeFn} states that a BNN converges in distribution to an NNGP
even if we replace the fixed batch of inputs with a convergent sequence of input batches. 
Continuous convergence in distribution is in fact equivalent to uniform convergence in distribution on compact sets (Prop.~\ref{prop-uc}), thus 
granting condition~\ref{lemma-Fn-Xn:d}.

The proof of Lemma~\ref{lemma-tildeFn} runs in parallel to the proof of Lemma 12 in \citetalias{matthews2018gaussian}. 
It depends on several lemmas (Appendices~\ref{appendix-section-verifying}-\ref{appendix-section-establishing}) that are all simple extensions of (or help 
to extend) the lemmas in \citetalias{matthews2018gaussian}; 
at each step, we simply replace the fixed batch of inputs in \citetalias{matthews2018gaussian} with a convergent sequence of input batches and verify that convergence in distribution still holds. 
Only a few key modifications are made to the lemmas establishing uniform integrability (Appendix~\ref{appendix-section-establishing}).

\subsection{Notation}
\label{appendix-section-notation}

We start with some notation. 
Let $\{F[n]: \RR^M\times\Omega\mapsto\RR^L\}_{n=1}^\infty$ be a sequence of BNNs each with $D$ hidden layers of widths $H_{\mu}[n]$, $\mu\in\{1,\ldots,D\}$, and nonlinearity $\phi:\RR\mapsto\RR$ that satisfies the linear envelope condition. 
Let $F:\RR^M\times\Omega\mapsto\RR^L$ be the limiting NNGP of the sequence of BNNs. 
Let $f^{(\mu)}_i(x)[n]$ (resp. $f^{(\mu)}_i(x)$) and $g^{(\mu)}_i(x)[n]$ (resp. $g^{(\mu)}_i(x)$) be the preactivation and activation of the $i$-th neuron in the $\mu$-th hidden layer of $F[n]$ (resp. $F$). 

For each $n\in\NN$, let $X[n] = \{x_t[n]\in \RR^M\}_{t=1}^T$ be a batch of $T$ inputs, and suppose the sequence of batches $\{X[n]\}_{n=1}^\infty$ converges to some finite $X = \{x_t\}_{t=1}^T$. 
Let $\alpha\in\RR^{T\times|\NN|}$ be a countably infinite block vector whose blocks are indexed by $\NN$ and where each block has $T$ elements. 
Let $\alpha$ have finite support $\{1,\ldots,T\}\times I$, where $I$ is a finite subset of $\NN$; 
i.e., only finitely many blocks indexed by $I$ are permitted to have non-zero elements. 
Let $\alpha_{ti}$ denote the $t$-th element in the $i$-th block. 
For each $\mu\in\{1,\ldots,D+1\}$, define the preactivation projections of a BNN and its limiting NNGP as
\begin{align*}
f^{(\mu)}(X[n], \alpha)[n] &= \sum_{t=1}^T\sum_{i\in I} \alpha_{ti} f_i^{(\mu)}(x_t[n])[n], \\ %\label{eq-proj-rnn} \\
f^{(\mu)}(X, \alpha) &= \sum_{t=1}^T\sum_{i\in I} \alpha_{ti} f_i^{(\mu)}(x_t). %\label{eq-proj-nngp}
\end{align*}

Let $k^{(\mu)}:\RR^M\times\RR^M\mapsto\RR$ be the NNGP kernel of $f^{(\mu)}_i$. 
The kernel $k^{(\mu)}$ relates to the block kernel $K^{(\mu)}$ (Eqs.~\eqref{eq-K1}-\eqref{eq-Kmu}) through the equation
\begin{equation*}
K^{(\mu)}(X, X) =
\begin{cases}
k^{(\mu)}(X, X)\otimes \mathbf{I}_{\infty} & \mbox{ for } \mu\in\{1,\ldots,D\}  \\
k^{(\mu)}(X, X)\otimes \mathbf{I}_L & \mbox{ for } \mu=D+1,
\end{cases}
\end{equation*}
where $\mathbf{I}_L$ is the $L\times L$ identity matrix, $\mathbf{I}_{\infty}$ is the countably infinite identity matrix, and $\otimes$ denotes the Kronecker product. 
We also define
\begin{equation*}
L^{(\mu)}_{ij}(x_1, x_2)
= \cov[g^{(\mu)}_i(x_1), g^{(\mu)}_j(x_2)]
= \delta_{ij} \cov[g^{(\mu)}_1(x_1), g^{(\mu)}_1(x_2)],
\end{equation*}
which satisfies the relation
\begin{equation*}
K^{(\mu+1)}_{ij}(x_1, x_2) = v_b \delta_{ij} + v_w L^{(\mu)}_{ij}(x_1, x_2).
\end{equation*}
We let $L^{(\mu)}(X, X)$ denote a block matrix where $L^{(\mu)}_{ij}(X, X)$ is the $(i, j)$-th block. 
The block matrices $K^{(\mu)}(X, X)$ and $L^{(\mu)}(X, X)$ for $\mu\in\{1,\ldots,D\}$ have infinitely many blocks since the NNGP has infinitely wide hidden layers.
However, given a block vector $\alpha\in\RR^{T\times|\NN|}$ of finite support $\{1,\ldots,T\}\times I$, the quadratic forms $\alpha^\top K^{(\mu)}(X, X) \alpha$ and $\alpha^\top L^{(\mu)}(X, X) \alpha$ are still finite sums:
\begin{align*}
\alpha^\top K^{(\mu)}(X, X)\alpha
&= \sum_{t,u=1}^T\sum_{i,j\in I} \alpha_{ti}\alpha_{uj} K^{(\mu)}(X, X)_{T(i-1)+t,T(j-1)+u} \\
&= \operatorname{V}[\sum_{t=1}^T\sum_{i\in I} \alpha_{ti} f^{(\mu)}_i(x_t)] \\
&= \operatorname{V}[f^{(\mu)}(X, \alpha)],
\end{align*}
% where we recall that $I$ is a finite index set. 
We define a quadratic form for $L^{(\mu)}(X, X)$ similarly.

Next we define the quantities that are at the heart of the proof of Lemma~\ref{lemma-tildeFn}. 
This definition is similar to Definition 7 in \citetalias{matthews2018gaussian}. 
We discuss the purpose of this definition in more detail in Appendix~\ref{appendix-section-verifying} in the context of the Central Limit Theorem. 

\begin{definition}[Projections and summands] \label{defn-proj-sum}
For each $\mu\in\{2,\ldots,D+1\}$ and for each $n\in\NN$ and $j\in\{1,\ldots,n\}$, define the summands
\begin{equation} \label{eq-summand}
\gamma^{(\mu)}_j(X[n], \alpha)[n] = \sqrt{H_{\mu-1}[n]} \sum_{t=1}^T\sum_{i\in I} \alpha_{ti} w^{(\mu)}_{ij} g^{(\mu-1)}_j(x_t[n])[n],
\end{equation}
and the projections
\begin{align}
 \label{eq-projection}
\begin{aligned}
S^{(\mu)}(X[n], \alpha)[n]
&= \sum_{t=1}^T\sum_{i\in I} \alpha_{ti} \left( f^{(\mu)}_i(x_t[n])[n] - b^{(\mu)}_i\right) \\
&= \frac{1}{\sqrt{H_{\mu-1}[n]}} \sum_{j=1}^{H_{\mu-1}[n]} \gamma^{(\mu)}_j[n].
\end{aligned}
\end{align}
\end{definition}

Finally, for $\mu\in\{1,\ldots,D+1\}$, define the variances
\begin{align}
\sigma^2_{(\mu)}(X[n], \alpha)[n]
&= \operatorname{V}[\gamma^{(\mu)}_j(X[n], \alpha)[n]] \label{eq-sigma} \\
\sigma^2_{(\mu)}(X, \alpha)
&= v_w \alpha^\top L^{(\mu-1)}(X, X) \alpha \label{eq-sigma-mu} \\
&= \alpha^\top \left[K^{(\mu)}(X, X) - v_b \mathbf{I}\otimes\mathbf{1}_{T\times T}\right] \alpha\nonumber.
\end{align}

\subsection{Main lemmas and theorem}
\label{appendix-section-main}

This section contains the proof of the main theorem (Thm.~\ref{thm-bottleneck}) as well as all original lemmas supporting it. 
See the proof sketch in Sec.~\ref{subsection-bottleneck-nngp-theorem} for an overview and guide to the logical flow of the lemmas.
We start with a lemma that gives sufficient conditions under which the distributional limit of a sequence of compositions of stochastic processes and random variable indices equals the composition of limits.

\begin{lemma}[Limit of stochastic process compositions] \label{lemma-Fn-Xn}
Let $\{X[n]\}_{n=1}^\infty$ be a sequence of random vectors and $X$ a random vector of dimension $B$. 
Let $\{F[n]:\RR^B\times\Omega\mapsto\RR^L\}_{n=1}^\infty$ be a sequence of stochastic processes and $F:\RR^B\times\Omega\mapsto\RR^L$ a stochastic process with $F(x)\sim\calN(0, \Sigma(x))$. 
If
\begin{enumerate}[label={(\alph*)}]
\item $X[n]$ converges in distribution to $X$, denoted $X[n] \xrightarrow{D} X$, \label{lemma-Fn-Xn:a}
\item $\{\operatorname{E}\{|X[n]|^2\}\}_{n=1}^\infty$ is eventually bounded, \label{lemma-Fn-Xn:b}
\item $\Sigma: \RR^B\mapsto\RR^{L\times L}$ is a continuous function, and \label{lemma-Fn-Xn:c}
\item $F[n] \xrightarrow{UD} F$ on every compact ball in $\RR^B$ centered at $0$, \label{lemma-Fn-Xn:d}
\end{enumerate}
then the sequence of random variables $F(X[n])[n] \xrightarrow{D} F(X)$.
\end{lemma}
\begin{proof}
We first prove the claim for the case that $F[n]$ and $F$ are real-valued stochastic processes ($L=1$) and $\Sigma(x) > 0$ for all $x\in\RR^B$. 
For this case, we will use the notation $\sigma^2(x)$ in place of $\Sigma(x)$ to emphasize that $\Sigma(x)$ is a scalar.

Let $c$ be a continuity point of $F(X)$. 
We want to show that
\[ \limn \Pr(F(X[n])[n] < c) = \Pr(F(X) < c). \]
Let $\eps > 0$. 
We have
\begin{align}
\label{eq-tri-ineq}
\begin{aligned}
&\lvert \Pr(F(X[n])[n] < c) - \Pr(F(X) < c)\rvert \\
&= \lvert \Pr(F(X[n])[n] < c) - \Pr(F(X[n]) < c) 
+ \Pr(F(X[n]) < c) - \Pr(F(X) < c)\rvert \\
&\leq \lvert \Pr(F(X[n])[n] < c) - \Pr(F(X[n]) < c)\rvert
+ \lvert \Pr(F(X[n]) < c) - \Pr(F(X) < c)\rvert.
\end{aligned}
\end{align}
We will show that both terms on the right-hand side of Inequality \eqref{eq-tri-ineq} tend to $0$.

We start with the second term. 
Let $\mu[n]$ and $\mu$ be the probability distributions associated with $X[n]$ and $X$, respectively. 
Then the second term becomes
\[ \lvert \Pr(F(X[n]) < c\} - \Pr(F(X) < c\}\rvert =
\Big\lvert \int_{\RR^B} \Pr(F(x) < c)\,\mathrm{d}\mu(x)[n]
- \int_{\RR^B} \Pr(F(x) < c)\,\mathrm{d}\mu(x)\Big\rvert. \]
Since $F(x)\sim\calN(0, \sigma^2(x))$ with $\sigma^2(x) > 0$, then we have
\[ \Pr(F(x) < c) = \Phi\left(\frac{c}{\sigma(x)}\right), \]
where $\Phi$ is the cumulative distribution function (CDF) of the standard normal distribution. 
Since $\sigma$ is continuous, then the map $x\mapsto \Pr(F(x) < c)$ is immediately seen to be continuous as well. 
Moreover, $\Pr(F(x) < c)$ is clearly bounded. 
Since $X[n] \xrightarrow{D} X$, then $\{\mu[n]\}_{n=1}^\infty$ converges weakly to $\mu$, so that
\[ \limn \int_{\RR^B} \Pr(F(x) < c)\,\mathrm{d}\mu(x)[n] = \int_{\RR^B} \Pr(F(x) < c)\,\mathrm{d}\mu(x). \]
Therefore, there exists an integer $N_2$ such that
\begin{equation} \label{eq-tri-ineq-2}
\lvert \Pr(F(X[n]) < c) - \Pr(F(X) < c)\rvert < \frac{\eps}{2} \mbox{ for all } n > N_2.
\end{equation}

We next bound the first term on the right-hand side of Inequality \eqref{eq-tri-ineq}. 
We have
\begin{align}
\label{eq-integral-bound}
\begin{aligned}
&\lvert \Pr(F(X[n])[n] < c) - \Pr(F(X[n]) < c)\rvert\\
&= \Big\lvert \int_{\RR^B} \Pr(F(x)[n] < c)\,\mathrm{d}\mu(x)[n] - \int_{\RR^B} \Pr(F(x) < c)\,\mathrm{d}\mu(x)[n]\Big\rvert \\
&\leq \int_{\RR^B} \lvert \Pr(F(x)[n] < c) - \Pr(F(x) < c)\rvert\,\mathrm{d}\mu(x)[n].
\end{aligned}
\end{align}
We will bound the integrand. 
Since $\{\operatorname{E}\{|X[n]|^2\}\}_{n=1}^\infty$ is eventually bounded, then there exists $V>0$ and an integer $N_V$ such that
\[ \operatorname{E}\{|X[n]|^2\} < V \mbox{ for all } n > N_V. \]
Define
\[ R_{\eps} = \sqrt{\operatorname{max}\left(0, \frac{2}{\eps}(1+V)-1\right)}. \]
$R_{\eps}$ is defined such that $\|x\| > R_{\eps}$ implies
\[ \frac{\eps}{2}\cdot \frac{1+\|x\|^2}{1+V} > 1. \]
We therefore have
\begin{equation} \label{eq-x-ge-Reps}
\lvert \Pr(F[n](x) < c) - \Pr(F(x) < c)\rvert \leq 1 < \frac{\eps}{2}\cdot \frac{1+\|x\|^2}{1+V} \mbox{ for all } x \mid \|x\| > R_{\eps}.
\end{equation}
Since $F(x)$ follows a normal distribution, then $c$ is trivially a continuity point of $F(x)$ for every $x\in\RR^B$. 
Since $\{F[n]\}_{n=1}^\infty$ converges in distribution to $F$ uniformly on every zero-centered compact ball, then there exists an integer $N_1 > N_V$ such that
\[ \lvert \Pr(F(x)[n] < c) - \Pr(F(x) < c)\rvert < \frac{\eps}{2}\cdot \frac{1}{1+V} \mbox{ for all } n>N_1 \mbox{ and } \|x\| \leq R_{\eps}. \]
Since $\|x\|^2 \geq 0$, then we have the weaker bound
\[ \lvert \Pr(F(x)[n] < c) - \Pr(F(x) < c)\rvert < \frac{\eps}{2}\cdot \frac{1+\|x\|^2}{1+V} \mbox{ for all } n>N_1 \mbox{ and } \|x\| \leq R_{\eps}. \]
Combining this with Eq.~\eqref{eq-x-ge-Reps} gives
\[ \lvert \Pr(F(x)[n] < c) - \Pr(F(x) < c)\rvert < \frac{\eps}{2}\cdot \frac{1+\|x\|^2}{1+V} \mbox{ for all } n>N_1 \mbox{ and } x\in\RR^B. \]
Using this bound in Inequality \eqref{eq-integral-bound}, we get
\begin{align}
\label{eq-tri-ineq-1}
\begin{aligned}
\lvert \Pr(F(X[n])[n] < c) - \Pr(F(X[n]) < c)\rvert
&\leq \int_{\RR^B} \frac{\eps}{2}\cdot \frac{1+\|x\|^2}{1+V}\,\mathrm{d}\mu_n(x) \\
&= \frac{\eps}{2}\cdot \frac{1+\operatorname{E}\{|X[n]|^2\}}{1+V} \\
&\leq \frac{\eps}{2}\cdot \frac{1+V}{1+V} \\
&= \frac{\eps}{2} \mbox{ for all } n > N_1.
\end{aligned}
\end{align}

Let $N = \operatorname{max}(N_1, N_2)$. 
Combining Inequalities \eqref{eq-tri-ineq}, \eqref{eq-tri-ineq-2}, and \eqref{eq-tri-ineq-1}, we obtain the bound
\[ \lvert \Pr(F(X[n])[n] < c) - \Pr(F(X) < c)\rvert < \frac{\eps}{2} + \frac{\eps}{2} = \eps \mbox{ for all } n > N, \]
implying $F(X[n])[n] \xrightarrow{D} F(X)$.

Now consider the more general case where the processes $F[n]$ and $F$ take values in $\RR^L$ for $L\geq 1$ and where the kernel $k$ of $F$ is not necessarily strictly positive definite. 
Consider any $\alpha\in\RR^L$, and define the processes
\begin{align*}
\hat{F}(x)[n] &= Z + \alpha^\top F(x)[n], \\
\hat{F}(x) &= Z + \alpha^\top F(x),
\end{align*}
where $Z\sim \mathcal{N}(0, 1)$ is independent of $F[n]$ and $F$. 
$\hat{F}[n]$ and $\hat{F}$ are real-valued stochastic processes and $\hat{F}(x)$ is normally distributed with variance
\[ \hat{\sigma}^2(x) = 1 + \alpha^\top \Sigma(x) \alpha > 0. \]
By the case already proven above, $\hat{F}(X[n])[n] \xrightarrow{D} \hat{F}(X)$. 
Since the addition of an independent normally distributed random variable $Z$ preserves convergence in distribution, then $\alpha^\top F(X[n])[n] \xrightarrow{D} \alpha F(x)$. 
Since this holds for any vector $\alpha$, then by the Cram\'{e}r-Wold Device~\citep{billingsley1995probability}, we obtain the conclusion $F(X[n])[n] \xrightarrow{D} F(X)$.
\end{proof}

Lemma~\ref{lemma-Fn-Xn} can be applied inductively to show that a sequence of compositions of stochastic processes converges in distribution to the composition of limit processes. 
The next lemma verifies condition~\ref{lemma-Fn-Xn:b} of Lemma~\ref{lemma-Fn-Xn}.

\begin{lemma}[Uniformly bounded neural network variances] \label{lemma-unif-bound-nn-var}
Let $\{F[n]\}_{n=1}^\infty$ be a sequence of BNNs mapping $\RR^M$ to $\RR^L$ with $D$ hidden layers of widths $H_{\mu}[n]$, $\mu\in\{1,\ldots,D\}$, and nonlinearity $\phi^{\mu}$ on the $\mu$-th hidden layer satisfying the linear envelope condition. 
Then for every $x\in\RR^M$, the sequence of second moments $\operatorname{E}[\| F(x)[n]\|^2]\}_{n=1}^\infty$ is uniformly bounded by $A+B\| x\|^2$ for some constants $A,B>0$.
\end{lemma}
\begin{proof}
Let $x\in\RR^M$. 
The claimed uniform bound on $\operatorname{E}[\| F(x)[n]\|^2]$ holds if we can uniformly bound $\{\operatorname{E}[\| F_i(x)[n]\|^2]\}_{n=1}^\infty$ for each $i$. 
Since $F_i(x)[n] = f_i^{(D+1)}(x)[n]$, then we need to establish $\operatorname{E}[f_i^{(D+1)}(x)[n]^2] \leq A+B\|x\|^2$ for sufficiently large $n$. 
We proceed by induction on $\mu$. 
In the case $\mu=1$, we have
\[ \operatorname{E}[f_i^{(1)}(x)[n]] = v_b + v_w \| x\|^2. \]
Taking $A=v_b$ and $B=v_w$, this is clearly bounded by $A+B\|x\|^2$ independently of $n$. 
By exchangeability, this same bound holds for all $i$.

Now suppose for some $\mu$ that the claimed uniform bound holds. 
We then need to establish the bound $\operatorname{E}[f_i^{(\mu+1)}(x)[n]] \leq A+B\|x\|^2$ for some $A,B>0$. 
We have
\begin{align*}
\operatorname{E}[f_i^{(\mu+1)}(x)[n]^2]
&= v_b + \frac{v_w}{H_{\mu}[n]} \sum_{j=1}^{H_{\mu}[n]} \operatorname{E}[g_j^{(\mu)}(x)[n]^2] \\
&= v_b + v_w \operatorname{E}[g_i^{(\mu)}(x)[n]^2] \\
&= v_b + v_w \int_{-\infty}^\infty \phi^{(\mu)}(z)^2\,\mathrm{d}\mu_n(z),
\end{align*}
where $\mu_n$ is the probability distribution of $f_i^{(\mu)}(x)[n]$. 
By the linear envelope condition,
\begin{align*}
\operatorname{E}[f_i^{(\mu+1)}(x)[n]^2]
&\leq v_b + v_w \int_{-\infty}^\infty (C + M |z|)^2\,\mathrm{d}\mu_n(z) \\
&\leq v_b + v_w \int_{-\infty}^\infty 2(C^2 + M^2 |z|^2)\,\mathrm{d}\mu_n(z) \\
&= v_b + 2v_w (C^2 + M^2\operatorname{E}[f_i^{(\mu)}(x)[n]^2]) \\
&\leq v_b + 2v_w (C^2 + M^2 (A + B \|x\|^2)),
\end{align*}
which is clearly bounded by an expression of the form $A^\prime + B^\prime \|x\|^2$ independently of $n$ for some $A^\prime,B^\prime > 0$. 
The claim then follows by induction.
\end{proof}

The next lemma verifies condition~\ref{lemma-Fn-Xn:c} of Lemma~\ref{lemma-Fn-Xn}, which amounts to showing that the NNGP kernel is continuous.

\begin{lemma}[Continuity of batch NNGP kernel] \label{lemma-cont}
Let $F:\RR^M\times\Omega\mapsto\RR^L$ be an NNGP with $D$ hidden layers and nonlinearity $\phi$ that satisfies the linear envelope condition. 
Then the associated batch NNGP $\tilde{F}: (\RR^M)^T\times\Omega\mapsto (\RR^L)^T$ of size $T$ has marginal $\tilde{F}(X)\sim \calN(0, \Sigma(X))$,
where the batch NNGP kernel $\Sigma: (\RR^M)^T\mapsto \RR^{LT\times LT}$ given by $\Sigma(X) = K(X, X)$ is a continuous function.
\end{lemma}
\begin{proof}
All we need to show is that $\Sigma$ is a continuous function. 
Since $\Sigma(X) = K(X, X) = k(X, X)\otimes I_L$, then it is sufficient to show that the NNGP kernel $k: \RR^M\times\RR^M\mapsto\RR$ is continuous.
We do so inductively by showing that $k^{(\mu)}$ is continuous for $\mu\in\{1,\ldots,D+1\}$, where $k = k^{(D+1)}$.

For $\mu=1$, $k^{(1)}(x, x^\prime) = v_b+v_w x\cdot x^\prime$ is clearly continuous. 
Now suppose for some $\mu\in\{1,\ldots,D\}$, that $k^{(\mu)}$ is continuous. 
We then need to show that $k^{(\mu+1)}$ is continuous. 
Let $\{X_n = (x_n, x^\prime_n)\in \RR^M\times\RR^M\}_{n=1}^\infty$ be a convergent sequence of pairs of inputs such that $X_n\rightarrow X = (x, x^\prime)$. 
Since $k^{(\mu)}$ is continuous, then $k^{(\mu)}(X_n, X_n)\rightarrow k^{(\mu)}(X, X)$. 
Let $s_i(X_n)$ (resp. $s_i(X)$) denote the $i$-th column of the symmetric positive semidefinite square root $S(X_n)$ (resp. $S(X)$) of $k^{(\mu)}(X_n, X_n)$ (resp. $k^{(\mu)}(X, X)$). 
Then by the work of \citet{cho2009kernel}, the kernel recursion in Eq.~\eqref{eq-Kmu} can be expressed as
\begin{equation*}
k^{(\mu+1)}(x_n, x^\prime_n) = v_b + v_w \frac{1}{2\pi}\int_{\RR^2} \phi(w^\top s_1(X_n)) \phi(w^\top s_2(X_n))\, e^{-\frac{\| w\|^2}{2}}\,\mathrm{d}w,
\end{equation*}
and $k^{(\mu+1)}(x, x^\prime)$ is given similarly. 
To ensure continuity, we will show that $k^{(\mu+1)}(x_n, x^\prime_n)\rightarrow k^{(\mu+1)}(x, x^\prime)$. 
We do so by verifying the conditions of the Dominated Convergence Theorem. 
First, by the linear envelope condition, there exist positive constants $C$ and $M$ such that
\begin{align*}
\phi(w^\top s_1(X_n)) \phi(w^\top s_2(X_n))
&\leq [C + M w^\top s_1(X_n)][C + M w^\top s_2(X_n)] \\
% &= C^2 + CM [w^\top s_1(X_n) + w^\top s_2(X_n)] + M^2 [w^\top s_1(X_n) w^\top s_2(X_n)] \\
&= C^2 + CM \sum_{i=1}^{2} w^\top s_i(X_n) + M^2 w^\top s_1(X_n) w^\top s_2(X_n) \\
% &\leq C^2 + CM [w^\top s_1(X_n) + w^\top s_2(X_n)] + \frac{M^2}{2} [(w^\top s_1(X_n))^2 + (w^\top s_2(X_n))^2] \\
&\leq C^2 + CM \sum_{i=1}^{2} w^\top s_i(X_n) + \frac{M^2}{2} \sum_{i=1}^{2}(w^\top s_i(X_n))^2 \\
&= \frac{1}{2}\left[ (C + M w^\top s_1(X_n))^2 + (C + M w^\top s_2(X_n))^2 \right] \\
&\leq [C^2 + M^2 (w^\top s_1(X_n))^2] + [C^2 + M^2 (w^\top s_2(X_n))^2] \\
&= 2C^2 + M^2 \| S^\top(X_n) w\|^2 \\
&\leq 2C^2 + M^2 \|S^\top(X_n)\|_2^2 \|w\|^2 \\
&\leq 2C^2 + M^2 \|S^\top(X_n)\|_F^2 \|w\|^2 \\
&= 2C^2 + M^2 \|S(X_n)\|_F^2 \|w\|^2,
\end{align*}
where $\|\cdot\|_F$ denotes the Frobenius norm. 
Since the matrix square root operation is continuous under the Frobenius norm, then $\|S(X_n)\|_F$ is bounded by some $B > 0$. 
We therefore have
\[ \phi(w^\top s_1(X_n)) \phi(w^\top s_2(X_n)) \leq 2C^2 + M^2 B^2 \|w\|^2, \]
hence
\[ \frac{1}{2\pi}\phi(w^\top s_1(X_n)) \phi(w^\top s_2(X_n))e^{-\frac{\|w\|^2}{2}} \leq \frac{1}{2\pi} (2C^2 + M^2 B^2 \|w\|^2) e^{-\frac{\|w\|^2}{2}}, \]
where the bound on the right-hand side is clearly integrable over $w\in\RR^2$. 
Moreover, since the matrix square root operation is continuous and $\phi$ is continuous, then
\begin{equation} \label{eq-lemma-cont-pointwise}
\frac{1}{2\pi}\phi(w^\top s_1(X_n)) \phi(w^\top s_2(X_n))e^{-\frac{\|w\|^2}{2}}\rightarrow \frac{1}{2\pi}\phi(w^\top s_1(X)) \phi(w^\top s_2(X))e^{-\frac{\|w\|^2}{2}} 
\mbox{ pointwise in } w.
\end{equation}
Therefore, by the Dominated Convergence Theorem, $k^{(\mu+1)}(x_n, x^\prime_n)\rightarrow k^{(\mu+1)}(x, x^\prime)$ so that $k^{(\mu+1)}$ is continuous. 
The continuity of the kernel $k$ then follows by induction.
\end{proof}

The next lemma will help to verify condition~\ref{lemma-Fn-Xn:d} of Lemma~\ref{lemma-Fn-Xn}.
It is a generalization of Lemma 12 in \citetalias{matthews2018gaussian} and depends on several additional lemmas (Appendices~\ref{appendix-section-verifying}-\ref{appendix-section-establishing}) similar to those in \citetalias{matthews2018gaussian}.

\begin{lemma}[Continuous convergence in distribution of batch BNNs] \label{lemma-tildeFn}
Consider a sequence $\{F[n]\}_{n=1}^\infty$ of BNNs mapping $\RR^M$ to $\RR^L$ with $D$ hidden layers of widths $H_{\mu}[n]$, $\mu\in\{1,\ldots,D\}$, and nonlinearity $\phi$ that satisfies the linear envelope condition. 	
% Let $\{F[n]\}_{n=1}^\infty$ be a sequence of BNNs mapping $\RR^M$ to $\RR^L$ with $D$ hidden layers of widths $H_{\mu}[n]$, $\mu\in\{1,\ldots,D\}$, and nonlinearity $\phi$ that satisfies the linear envelope condition. 
Let $F$ be the NNGP limit of the BNNs as given by Thm.~\ref{thm-matthews}. 
Then for any $T\in\NN$, the corresponding sequence of batch BNNs $\{\tilde{F}[n]\}_{n=1}^\infty$ converges in distribution to the batch NNGP $\tilde{F}$ continuously.
\end{lemma}
\begin{proof}
For each $n\in\NN$, let $X[n] = \{x_t[n]\}_{t=1}^T \in (\RR^M)^T$ be a batch of inputs such that the sequence of batches $\{X[n]\}_{n=1}^\infty$ converges to some finite $X\in (\RR^M)^T$. 
We need to show that $\{\tilde{F}[n]\}_{n=1}^\infty$ converges in distribution to $\tilde{F}$ continuously, i.e. that the sequence of random variables $\{\tilde{F}(X[n])[n]\}_{n=1}^\infty$ converges in distribution to $\tilde{F}(X)$ and thus
\begin{equation} \label{eq-lemma-tildeF-1}
\{f^{(D+1)}(x_t[n])[n]\}_{t=1}^T \xrightarrow{D} \{f^{(D+1)}(x_t)\}_{t=1}^T.
\end{equation}
We will do so by establishing $\{f^{(\mu)}(x_t[n])[n]\}_{t=1}^T \xrightarrow{D} \{f^{(\mu)}(x_t)\}_{t=1}^T$ inductively for every $\mu\in\{1,\ldots,D+1\}$.

For the case $\mu=1$, let $\alpha\in\RR^{T\times|\NN|}$ with finite support $\{1,\ldots,T\}\times I$. 

By definition (Eq.~\eqref{eq-f1}), it is straightforward to verify that
\begin{align*}
f^{(1)}(X[n], \alpha)[n] &\sim \mathcal{N}(0, \alpha^\top K^{(1)}(X[n], X[n]) \alpha) \\
f^{(1)}(X, \alpha) &\sim \mathcal{N}(0, \alpha^\top K^{(1)}(X, X) \alpha).
\end{align*}
Let $c$ be a continuity point of $f^{(1)}(X, \alpha)$ so that $c\neq 0$ if $\alpha^\top K^{(1)}(X, X)\alpha = 0$. 
Extend the CDF $\Phi$ of the standard normal distribution by setting $\Phi(-\infty)=0$ and $\Phi(\infty)=1$. 
Then the map
\begin{equation} \label{eq-lemma-tildeF-2}
z \mapsto \Phi\left(\frac{c}{\sqrt{z}}\right)
\end{equation}
 is continuous on $(0, \infty)$ and is right-continuous at $z=0$ if $c\neq 0$. 
Now since the kernel is a continuous function (Lemma~\ref{lemma-cont}) and since $X[n]\rightarrow X$, then $\alpha^\top K^{(1)}(X[n], X[n])\alpha \rightarrow \alpha^\top K^{(1)}(X, X)\alpha$. 
Moreover, since we just established that the map given by Eq.~\eqref{eq-lemma-tildeF-2} is continuous, then it follows that
\[ \Phi\left(\frac{c}{\sqrt{\alpha^\top K^{(1)}(X[n], X[n])\alpha}}\right) \rightarrow \Phi\left(\frac{c}{\sqrt{\alpha^\top K^{(1)}(X, X)\alpha}}\right) \mbox{ as } n\rightarrow\infty, \]
and hence $f^{(1)}(X[n], \alpha)[n] \xrightarrow{D} f^{(1)}(X, \alpha)$. 
By the Cram\'{e}r-Wold Device, we deduce that $f^{(1)}_I(X[n])[n] \xrightarrow{D} f^{(1)}_I(X)$.

Now suppose that $f^{(\mu)}_I(X[n])[n] \xrightarrow{D} f^{(\mu)}_I(X)$ for every finite subset $I\subseteq\NN$ and for some $\mu\in\{1,\ldots,D\}$. 
We then want to show that this same convergence holds for $\mu+1$. 
Let $\alpha\in\RR^{T\times|\NN|}$ with finite support $\{1,\ldots,T\}\times I$. 
We view $\alpha$ as a block vector where $\alpha_{ti}$ is the $t$-th element in the $i$-th block. 
By Lemmas~\ref{lemma-condition-variance}-\ref{lemma-condition-3}, the sequence of summands $\{\gamma^{(\mu+1)}_j(X[n], \alpha)\}_{j=1}^{H_{\mu}[n]}$ for $n\in\NN$ satisfies the conditions of Thm.~\ref{thm-blum-clt}; condition 1 is immediate since the weights $w^{(\mu+1)}_{i1}$ and $w^{(\mu+1)}_{j2}$ are independent and have mean $0$. 
Theorem~\ref{thm-blum-clt} then tells us that the projections $S^{(\mu+1)}(X[n], \alpha)[n] \xrightarrow{D} \calN(0, \sigma^2(X, \alpha))$, where the limiting variance is given by $\sigma^2_{(\mu+1)}(X, \alpha)$ (Eq.~\eqref{eq-sigma-mu}). 
By the Cram\'{e}r-Wold Device, this implies
\[ f^{(\mu+1)}_I(X[n])[n] - b^{(\mu+1)}_I\otimes \mathbf{1}_T \xrightarrow{D} \calN(0, v_w L^{(\mu)}_{II}(X, X)), \]
which in turn implies
\[ F^{(\mu+1)}_I(X[n])[n] \xrightarrow{D} F^{(\mu+1)}_I(X) \sim \calN(0, K^{(\mu+1)}_{II}(X, X)). \]
Equation \eqref{eq-lemma-tildeF-1} then follows by induction, thus establishing continuous distributional convergence.
\end{proof}

Next is the proof of the bottleneck NNGP theorem, which is the main theorem of our paper.

\begin{proof}[Proof of Thm.~\ref{thm-bottleneck}]
We proceed by induction on $d\in\{1,\ldots,D\}$. 
The case $d=1$ is given to us by Thm.~\ref{thm-matthews} (i.e., no hidden bottlenecks). 
Now suppose the claim holds for some $d\in\{1,\ldots,D-1\}$. 
We will prove the claim for the case $d+1$.

Let $X = \{x_t\}_{t=1}^T$ be a finite subset of $\mathcal{X}$. 
Define the random variables
\begin{align*}
Z[n] &= \{(F^{(d)}[n]\circ\cdots\circ F^{(1)}[n])(x_t)\}_{t=1}^T, \\
Z &= \{(F^{(d)}\circ\cdots\circ F^{(1)})(x_t)\}_{t=1}^T.
\end{align*}
Let $\tilde{F}^{(d+1)}[n]$ (resp. $\tilde{F}^{(d+1)}$) be the batch BNN (resp. batch NNGP) corresponding to $F^{(d+1)}[n]$ (resp. $F^{(d+1)}$), and observe that
\begin{align*}
\tilde{F}^{(d+1)}(Z[n])[n] &= \{(F^{(d+1)}[n]\circ\cdots\circ F^{(1)}[n])(x_t)\}_{t=1}^T, \\
\tilde{F}^{(d+1)}(Z) &= \{(F^{(d+1)}\circ\cdots\circ F^{(1)})(x_t)\}_{t=1}^T.
\end{align*}
We proceed to establish the four conditions of Lemma~\ref{lemma-Fn-Xn} in order to prove
\begin{equation} \label{eq-thm-bottleneck-1}
\tilde{F}^{(d+1)}(Z[n])[n] \xrightarrow{D} \tilde{F}^{(d+1)}(Z).
\end{equation}
By the inductive hypothesis, $F^{(d)}[n]\circ\cdots\circ F^{(1)}[n] \xrightarrow{D} F^{(d)}\circ\cdots\circ F^{(1)}$ in $((\RR^L)^{\infty}, \mathcal{A})$ and thus in particular $Z[n] \xrightarrow{D} Z$, establishing condition~\ref{lemma-Fn-Xn:a}. 
Observe that
\[ \operatorname{E}[\|Z[n]\|^2] = \sum_{t=1}^T \operatorname{E}[\|(f^{(d)}[n]\circ\cdots\circ F^{(1)}[n])(x_t)\|^2]. \]
Since a composition of BNNs is still a BNN (with some hidden layers having linear activation), then we can apply Lemma~\ref{lemma-unif-bound-nn-var} to each expectation in the sum to get the bound
\[ \operatorname{E}[\|Z[n]\|^2] \leq \sum_{t=1}^T (A_t + B_t \|x_t\|^2), \]
for some constants $A_t,B_t > 0$. 
In other words, the sequence of second moments of $\{Z[n]\}_{n=1}^\infty$ is bounded, establishing condition~\ref{lemma-Fn-Xn:b}. 
Lemma~\ref{lemma-cont} gives us condition~\ref{lemma-Fn-Xn:c}. 
Finally, Lemma~\ref{lemma-tildeFn} tells us that $\tilde{F}^{(d+1)}[n] \xrightarrow{CD} \tilde{F}^{(d+1)}$. 
By Prop.~\ref{prop-uc}, we immediately have $\tilde{F}^{(d+1)}[n] \xrightarrow{UD} \tilde{F}^{(d+1)}$ on every compact subset of $\RR^{T\times B_d}$, establishing condition~\ref{lemma-Fn-Xn:d}.

Having verified its four conditions, Lemma~\ref{lemma-Fn-Xn} implies Eq.~\eqref{eq-thm-bottleneck-1} and hence
\[ \{(F^{(d+1)}[n]\circ\cdots\circ F^{(1)}[n])(x_t)\}_{t=1}^T \xrightarrow{D} \{(F^{(d+1)}\circ\cdots\circ F^{(1)})(x_t)\}_{t=1}^T. \]
Since this holds for any $T$ inputs in $\mathcal{X}$, then by Thm.~\ref{thm-marginal} the desired convergence in $((\RR^L)^{\infty}, \mathcal{A})$ follows.
\end{proof}

\begin{remark}[Nonlinear bottleneck] \label{rem-nonlinbottleneck-appendix}
    \label{rem-nonlin-bottleneck}
Theorem~\ref{thm-bottleneck} holds even if we replace $F^{(d)}[n]$ and $F^{(d)}$ with $F^{(d)}[n]\circ\left(\frac{1}{\sqrt{B_{d-1}}}\phi\right)$ and $F^{(d)}\circ\left(\frac{1}{\sqrt{B_{d-1}}}\phi\right)$ respectively for $d\in\{2,\ldots,D\}$. 
The proof is nearly identical, making the necessary replacements where appropriate. 
The only additional step needed is to verify condition~\ref{lemma-Fn-Xn:c} of Lemma~\ref{lemma-Fn-Xn} for $\tilde{F}^{(d+1)}[n]\circ\left(\frac{1}{\sqrt{B_d}}\phi\right)$ in the inductive step; 
by Lemma~\ref{lemma-Fn-Xn}, $\tilde{F}^{(d+1)}[n]\xrightarrow{UD}\tilde{F}^{(d+1)}$ and hence $\tilde{F}^{(d+1)}[n]\xrightarrow{CD}\tilde{F}^{(d+1)}$ by Prop.~\ref{prop-uc}. 
Now since $x\mapsto \frac{1}{\sqrt{B_d}}\phi(x)$ is (sequentially) continuous, then $\tilde{F}^{(d+1)}\left(\frac{1}{\sqrt{B_d}}\phi(x_n)\right)\rightarrow \tilde{F}^{(d+1)}\left(\frac{1}{\sqrt{B_d}}\phi(x)\right)$ whenever $x_n\rightarrow x$. 
By Prop.~\ref{prop-uc}, $\tilde{F}^{(d+1)}[n]\circ\left(\frac{1}{\sqrt{B_d}}\phi\right)\xrightarrow{UD}\tilde{F}^{(d+1)}\circ\left(\frac{1}{\sqrt{B_d}}\phi\right)$, establishing condition~\ref{lemma-Fn-Xn:c}.
\end{remark}

\begin{remark}[Discontinuous nonlinearity] \label{rem-discont-appendix}
Theorem~\ref{thm-bottleneck} holds even if the nonlinearity $\phi:\RR\mapsto\RR$ is continuous only almost everywhere (AE), as long as $\phi$ is continuous at $0$ or $v_b > 0$. 
If $\phi$ is continuous AE, then the pointwise convergence in Eq.~\eqref{eq-lemma-cont-pointwise} holds AE, which is still sufficient for the Dominated Convergence Theorem. 
Moreover, the Continuous Mapping Theorem used in Lemmas~\ref{lemma-condition-variance}-\ref{lemma-condition-3} is still applicable as long as the set of discontinuities of $\phi$ has measure $0$ with respect to the distribution of the NNGP preactivation $f^{(\mu)}_i(x)$. 
If $v_b=0$, then it becomes possible for the distribution of $f^{(\mu)}_i(x)$ to degenerate to a delta distribution concentrated at $0$; 
if $\phi$ is also discontinuous at $0$, then its set of discontinuities will have measure $1$ with respect to the delta distribution, 
hence the requirement that $v_b > 0$ if $\phi$ is discontinuous at $0$.
\end{remark}

\subsection{Verifying the conditions of the CLT for exchangeable processes}
\label{appendix-section-verifying}

The results in this section serve to support the proof of Lemma~\ref{lemma-tildeFn}. 
Since Lemma~\ref{lemma-tildeFn} is similar to Lemma 12 in \citetalias{matthews2018gaussian}, then the results in this section are also similar to results in \citetalias{matthews2018gaussian}. 
The approach to proving Lemma~\ref{lemma-tildeFn} is to show that in the (no-bottleneck) NNGP limit, if the preactivations into one hidden layer converge in distribution continuously to a GP, then so do the preactivations into the next hidden layer. 
This is done using a special central limit theorem. 
The challenge is that the preactivations into any hidden layer after the first hidden layer are independent only in the wide limit. 
Moreover, the distribution of each preactivation changes as the preceding hidden layer grows in width. 
The following is a central limit theorem adapted specifically for this case; 
it is a restatement of Lemma 10 in \citetalias{matthews2018gaussian}, which is in turn an adaptation of a central limit theorem for exchangeable processes by \citet{blum1958central}.

\begin{theorem}[CLT for sequences of exchangeable sequences~\citepalias{matthews2018gaussian}] \label{thm-blum-clt}
For each positive integer $n$, let $\{X_i[n]\}_{i=1}^\infty$ be an exchangeable sequence of random variables with mean $0$, variance $\sigma^2[n]$, and finite absolute third moment. 
Suppose also that the variances converge to the limit
%\begin{equation} \label{eq-clt-sigma2}
$
\limn \sigma^2[n] = \sigma^2.
$
%\end{equation}
If
\begin{enumerate}[label={(\alph*)}]
\item $E\{X_1[n]X_2[n]\} = 0$, \label{thm-blum-clt:a}
\item $\limn E\{X_1[n]^2X_2[n]^2\} = \sigma^4$, and \label{thm-blum-clt:b}
\item $E\{|X_1[n]|^3\} = o(\sqrt{n})$ \label{thm-blum-clt:c}
\end{enumerate}
then for any strictly increasing sequence $H$, the sequence of standardized partial sums $\{S[n]\}_{n=1}^\infty$ with
\begin{equation*}
S[n] = \frac{1}{\sqrt{H[n]}} \sum_{i=1}^{H[n]} X_i[n]
\end{equation*}
converges in distribution to $\calN(0, \sigma^2)$ , 
where $\mathcal{N}(0, 0)$ is interpreted as the constant $0$.
\end{theorem}

We will apply Thm.~\ref{thm-blum-clt} to the summands $\gamma^{(\mu)}_j(X[n], \alpha)[n]$ (Eq.~\eqref{eq-summand}) to show that the projection $S^{(\mu)}(X[n], \alpha)[n]$ (Eq.~\eqref{eq-projection}) converges to a GP. 
This requires us to verify the conditions of Thm.~\ref{thm-blum-clt}. 
We verify
%Eq.~\eqref{eq-clt-sigma2}
the existence of the limit $\limn \sigma^2[n] = \sigma^2$ first. 
The following lemma is analogous to Lemma 11 in \citetalias{matthews2018gaussian}. 
The main difference is that the batch input $X$ is replaced with a convergent sequence of input batches $\{X[n]\}_{n=1}^\infty$.
We maintain the notation introduced in Sec.~\ref{appendix-section-notation}.

\begin{lemma} \label{lemma-condition-variance}
Suppose that $f^{(\mu)}_I(X[n])[n] \xrightarrow{D} f^{(\mu)}_I(X)$ for some $\mu\in\{1,\ldots,D\}$, and for every finite set $I\subset\NN$. 
Then
\[ \limn \sigma^2_{\mu+1}(X[n], \alpha)[n] = \sigma^2_{\mu+1}(X, \alpha), \]
where these variances are defined in Eqs.~\eqref{eq-sigma}-\eqref{eq-sigma-mu}.
\end{lemma}
\begin{proof}
It is clear that $\operatorname{E}[\gamma^{(\mu+1)}_j(X[n], \alpha)[n]] = 0$ since the weights $w^{(\mu+1)}_{ij}$ have $0$ mean. 
We therefore have
\begin{align*}
\sigma^2_{\mu+1}(X[n], \alpha)[n]
&= \operatorname{E}[\gamma^{(\mu+1)}_1(X[n], \alpha)[n]^2] \\
&= \operatorname{E}\left[\left(\sqrt{H_{\mu}[n]} \sum_{t=1}^T\sum_{i\in I} \alpha_{ti} w^{(\mu)}_{i1} g^{(\mu)}_1(x_t[n])[n]\right)^2\right] \\
&= H_{\mu}[n] \sum_{t,u=1}^T\sum_{i,j\in I} \alpha_{ti}\alpha_{uj}\operatorname{E}[w^{(\mu)}_{i1} w^{(\mu)}_{j1}] \operatorname{E}[g^{(\mu)}_1(x_t[n])[n] g^{(\mu)}_1(x_u[n])[n]] \\
&= v_w \sum_{t,u=1}^T\sum_{i,j\in I} \alpha_{ti}\alpha_{uj} \delta_{ij} \operatorname{E}[g^{(\mu)}_1(x_t[n])[n] g^{(\mu)}_1(x_u[n])[n]].
\end{align*}
Theorem 3.5 in \citet{billingsley1999convergence} tells us that a limit can be moved inside an expectation operator if the sequence inside the expectation converges in distribution and is uniformly integrable. 
Since the preactivations $f^{(\mu)}_1(X[n])[n]$ converge in distribution and since the nonlinearity $\phi$ and multiplication mapping $\RR^2$ to $\RR$ are continuous functions, then the Continuous Mapping Theorem implies that the products of activations in the above expectations also converge in distribution. 
Uniform integrability holds by Cor.~\ref{cor-unif-int-stochastic-2}. 
We therefore have the limit
\begin{align*}
\limn \sigma^2_{\mu+1}(X[n], \alpha)[n]
&= v_w \sum_{t,u=1}^T\sum_{i,j\in I} \alpha_{ti}\alpha_{uj} \delta_{ij} \operatorname{E}[g^{(\mu)}_1(x_t) g^{(\mu)}_1(x_u)] \\
&= v_w \sum_{t,u=1}^T\sum_{i,j\in I} \alpha_{ti}\alpha_{uj} L^{(\mu)}_{11}(x_t, x_u) \\
&= v_w \alpha^\top L^{(\mu)}(X, X) \alpha \\
&= \sigma^2_{(\mu+1)}(X, \alpha),
\end{align*}
completing the proof.
\end{proof}

Condition~\ref{thm-blum-clt:a} of Thm.~\ref{thm-blum-clt} is easily verified directly in the proof of Lemma~\ref{lemma-tildeFn}. 
We thus move to condition~\ref{thm-blum-clt:b}. 
The following lemma is analogous to Lemma 15 in \citetalias{matthews2018gaussian}. 

\begin{lemma} \label{lemma-condition-2}
Suppose that $f^{(\mu)}_I(X[n])[n] \xrightarrow{D} f^{(\mu)}_I(X)$ for some $\mu\in\{1,\ldots,D\}$. and for every finite set $I\subset\NN$. 
Then
\[ \limn \operatorname{E}[\gamma^{(\mu+1)}_1(X[n], \alpha)[n]^2 \gamma^{(\mu+1)}_2(X[n], \alpha)[n]^2] = \sigma^4_{\mu}(X, \alpha). \]
\end{lemma}
\begin{proof}
We proceed in direct analogy to the proof of Lemma~\ref{lemma-condition-variance}. 
We have
\begin{align*}
\begin{aligned}
& \operatorname{E}[\gamma^{(\mu+1)}_1(X[n], \alpha)[n]^2 \gamma^{(\mu+1)}_2(X[n], \alpha)[n]^2] \\
&= \operatorname{E}\left[\left(\sqrt{H_{\mu}[n]} \sum_{r=1}^T\sum_{i\in I} w^{(\mu)}_{i1} g^{(\mu)}_1(x_r[n])[n]\right)^2 \left(\sqrt{H_{\mu}[n]} \sum_{t=1}^T\sum_{k\in I} w^{(\mu)}_{k2} g^{(\mu)}_2(x_t[n])[n]\right)^2\right] \\
&= H_{\mu}^2[n] \sum_{r,s,t,u=1}^T\sum_{i,j,k,\ell\in I} \left(\alpha_{ri}\alpha_{sj}\alpha_{tk}\alpha_{u\ell} 
\cdot
\operatorname{E}[w^{(\mu)}_{i1}w^{(\mu)}_{j1}] \cdot
\operatorname{E}[w^{(\mu)}_{k2} w^{(\mu)}_{\ell2}] \right. \\
& \left. \cdot\operatorname{E}[g^{(\mu)}_1(x_r[n])[n] g^{(\mu)}_1(x_s[n])[n] g^{(\mu)}_2(x_t[n])[n] g^{(\mu)}_2(x_u[n])[n]]
\right) \\
&= v_w^2 \sum_{r,s,t,u=1}^T\sum_{i,j,k,\ell\in I}
\Big( \alpha_{ri}\alpha_{sj}\alpha_{tk}\alpha_{u\ell} \delta_{ij}\delta_{k\ell} \\
& \left. \cdot\operatorname{E}[g^{(\mu)}_1(x_r[n])[n] g^{(\mu)}_1(x_s[n])[n] g^{(\mu)}_2(x_t[n])[n] g^{(\mu)}_2(x_u[n])[n]]\right) .
\end{aligned}
\end{align*}
Since the preactivations $f^{(\mu)}_I(X[n])[n]$ converge in distribution for $I = \{1, 2\}$, and since the nonlinearity $\phi$ and multiplication from $\RR^4$ to $\RR$ are continuous functions, then the Continuous Mapping Theorem implies the four-way products of activations in each expectation above converge in distribution as well. 
Corollary~\ref{cor-unif-int-stochastic} also tells us that the set of these four-way products of activations is uniformly integrable. 
By Theorem 3.5 in \citet{billingsley1999convergence}, we have the limit
\begin{align} \label{eq-lemma-condition-2}
\begin{aligned}
& \limn \operatorname{E}[\gamma^{(\mu+1)}_1(X[n], \alpha)[n]^2 \gamma^{(\mu+1)}_2(X[n], \alpha)[n]^2] \\
& = v_w^2 \sum_{r,s,t,u=1}^T\sum_{i,j,k,\ell\in I} \alpha_{ri}\alpha_{sj}\alpha_{tk}\alpha_{u\ell} \delta_{ij}\delta_{k\ell} \operatorname{E}[g^{(\mu)}_1(x_r) g^{(\mu)}_1(x_s) g^{(\mu)}_2(x_t, \alpha) g^{(\mu)}_2(x_u, \alpha)].
\end{aligned}
\end{align}
Since parallel activations in a layer decorrelate in an NNGP, then we have
\begin{align*}
\begin{aligned}
& \limn \operatorname{E}[\gamma^{(\mu+1)}_1(X[n], \alpha)[n]^2 \gamma^{(\mu+1)}_2(X[n], \alpha)[n]^2] \\
&= v_w^2 \sum_{r,s,t,u=1}^T\sum_{i,j,k,\ell\in I} \alpha_{ri}\alpha_{sj}\alpha_{tk}\alpha_{u\ell} \delta_{ij}\delta_{k\ell} \operatorname{E}[g^{(\mu)}_1(x_r) g^{(\mu)}_1(x_s)] \operatorname{E}[g^{(\mu)}_2(x_t) g^{(\mu)}_2(x_u)] \\
&= \left( v_w \sum_{r,s=1}^T\sum_{i,j\in I} \alpha_{ri}\alpha_{sj}\delta_{ij} \operatorname{E}[g^{(\mu)}_1(x_r) g^{(\mu)}_1(x_s)]\right) \\
&\cdot \left(v_w \sum_{t,u=1}^T\sum_{k,\ell\in I} \alpha_{tk}\alpha_{u\ell}\delta_{k\ell} \operatorname{E}[g^{(\mu)}_2(x_t) g^{(\mu)}_2(x_u)]\right) \\
&= v_w \alpha^\top L^{(\mu)}(X, X)\alpha v_w \alpha^\top L^{(\mu)}(X, X)\alpha \\
&= \sigma^4_{(\mu+1)}(X, \alpha),
\end{aligned}
\end{align*}
completing the proof.
\end{proof}

Finally, we verify condition~\ref{thm-blum-clt:c} of Thm.~\ref{thm-blum-clt}. 
The following lemma is analogous to Lemma 16 in \citetalias{matthews2018gaussian}. 

\begin{lemma} \label{lemma-condition-3}
Suppose that $f^{(\mu)}_I(X[n])[n] \xrightarrow{D} f^{(\mu)}_I(X)$ for some $\mu\in\{1,\ldots,D\}$, and for every finite set $I\subset\NN$. 
Then
\[ \operatorname{E}[|\gamma^{(\mu+1)}_1(X[n], \alpha)[n]|^3] = o(\sqrt{n}). \]
\end{lemma}
\begin{proof}
We will prove the stronger result that the third absolute moment of $\gamma^{(\mu+1)}_1(X[n], \alpha)[n]$ is bounded over $n$. 
By H\"{o}lder's inequality,
\[ \operatorname{E}[|\gamma^{(\mu+1)}_1(X[n], \alpha)[n]|^3] \leq \operatorname{E}[\gamma^{(\mu+1)}_1(X[n], \alpha)[n]^4]^{\frac{3}{4}}. \]
Thus, to bound the left side independently of $n$, it is sufficient to do the same for the fourth moment of $\gamma^{(\mu+1)}_1(X[n], \alpha)[n]$. 
Observe that
\[ \operatorname{E}[\gamma^{(\mu+1)}_1(X[n], \alpha)[n]^4]^{\frac{3}{4}} = \operatorname{E}[\gamma^{(\mu+1)}_1(X[n], \alpha)[n]^2 \gamma^{(\mu+1)}_1(X[n], \alpha)[n]^2], \]
where the right-hand side is similar to the quantity discussed in Lemma~\ref{lemma-condition-2}. 
Therefore, calculations proceed in direct analogy to the proof of Lemma~\ref{lemma-condition-2} up to and including Eq.~\eqref{eq-lemma-condition-2}. 
Thus, we have
\begin{align*}
\begin{aligned}
& \limn \operatorname{E}[\gamma^{(\mu+1)}_1(X[n], \alpha)[n]^4]^{\frac{3}{4}} \\
&= v_w^2 \sum_{r,s,t,u=1}^T\sum_{i,j,k,\ell\in I} \alpha_{ri}\alpha_{sj}\alpha_{tk}\alpha_{u\ell} \delta_{ij}\delta_{k\ell}
\operatorname{E}[g^{(\mu)}_1(x_r, \alpha) g^{(\mu)}_1(x_s, \alpha) g^{(\mu)}_1(x_t, \alpha) g^{(\mu)}_1(x_u, \alpha)].
\end{aligned}
\end{align*}
The right-hand side can be shown to be finite by applying Lemma~\ref{lemma-4way-product} to bound the expectation of the four-way product by a product of eighth moments, applying the linear envelope property to obtain bounds in terms of preactivations, and finally noting that the eighth moment of a normal distribution is finite; this gives us the desired bound on the fourth and hence third absolute moment.
\end{proof}

\subsection{Establishing uniform integrability}
\label{appendix-section-establishing}

The results in this section serve to support the proofs in Appendix~\ref{appendix-section-verifying}. 
As in Appendix~\ref{appendix-section-verifying}, the results in this appendix are stronger versions of results appearing in \citetalias{matthews2018gaussian}.
The key results in this section are Lemma~\ref{lemma-4way-product} and Cors.~\ref{cor-unif-int-stochastic} and~\ref{cor-unif-int-stochastic-2} and are the only ones referenced outside of this section.

\begin{lemma} \label{lemma-4th-moment}
Let $X$ be a random variable. 
Then $E[X^4] \leq E[X^8]^{\frac{1}{2}}$.
\end{lemma}
\begin{proof}
By H\"{o}lder's Inequality, we have
\begin{equation*}
E[X^4] = E[X^4\cdot 1] \leq E[(X^4)^2]^{\frac{1}{2}} E[1^2]^{\frac{1}{2}} = E[X^8]^{\frac{1}{2}}.
\end{equation*}
\end{proof}

The following lemma is a stronger version of Lemma 18 in \citetalias{matthews2018gaussian}.
\citetalias{matthews2018gaussian} proves that the expectation $E\left[\prod_{i=1}^4 |X_i|^{p_i}\right]$ is uniformly bounded
by a polynomial in the eighth moments $E[X_i^8] < \infty$ for $i\in\{1,2,3,4\}$ without specifying the polynomial.
Lemma~\ref{lemma-4way-product} below provides the explicit bound $\prod_{i=1}^4 E[X_i^8]^{\frac{p_i}{8}}$, which is 
a polynomial in the eighth moments.
This bound is important when proving uniform convergence with respect to the inputs of a random 
neural network, since the coefficients and exponents in the bound are independent of the network's input.

\begin{lemma} \label{lemma-4way-product}
Let $X_i$ be random variables with $E[X_i^8] < \infty$ for $i\in\{1,2,3,4\}$.
Then for any choice of $p_i\in \{0,1,2\}$ it holds that
\[ E\left[\prod_{i=1}^4 |X_i|^{p_i}\right] \leq \prod_{i=1}^4 E[X_i^8]^{\frac{p_i}{8}}. \]
\end{lemma}
\begin{proof}
Using H\"{o}lder's inequality twice, we have
\begin{align*}
E[|X_1|^{p_1}|X_2|^{p_2}|X_3|^{p_3}|X_4|^{p_4}]
&\leq E[(|X_1|^{p_1}|X_2|^{p_2})^2]^{\frac{1}{2}} E[(|X_3|^{p_3}|X_4|^{p_4})^2]^{\frac{1}{2}} \\
&= E[X_1^{2p_1}X_2^{2p_2}]^{\frac{1}{2}} E[X_3^{2p_3}X_4^{2p_4}]^{\frac{1}{2}} \\
&\leq \left(E[(X_1^{2p_1})^2]^{\frac{1}{2}} E[(X_2^{2p_2})^2]^{\frac{1}{2}}\right)^{\frac{1}{2}} \left(E[(X_3^{2p_3})^2]^{\frac{1}{2}} E[(X_4^{2p_4})^2]^{\frac{1}{2}}\right)^{\frac{1}{2}} \\
&= E[X_1^{4p_1}]^{\frac{1}{4}} E[X_2^{4p_2}]^{\frac{1}{4}} E[X_3^{4p_3}]^{\frac{1}{4}} E[X_4^{4p_4}]^{\frac{1}{4}} \\
&= \prod_{i=1}^4 E[X_i^{4p_i}]^{\frac{1}{4}}.
\end{align*}
If $p_i=0$, then $E[X_i^{4p_i}]^{\frac{1}{4}} = E[1]^{\frac{1}{4}} = 1$, which can be written as $E[X_i^8]^0$. 
If $p_i=1$, then by Lemma~\ref{lemma-4th-moment}, $E[X_i^{4p_i}]^{\frac{1}{4}} \leq E[X_i^8]^{\frac{1}{8}}$. 
If $p_i=2$, then we simply have $E[X_i^{4p_i}]^{\frac{1}{4}} = E[X_i^8]^{\frac{2}{8}}$. 
We therefore see that for any $p_i\in\{0,1,2\}$, $E[X_i^{4p_i}]^{\frac{1}{4}} = E[X_i^8]^{\frac{p_i}{8}}$. 
Substituting this into the above product yields the desired bound.
\end{proof}

The following lemma extends Lemma 20 in \citetalias{matthews2018gaussian} to stochastic processes in the sense that the input into the BNN is now a variable. 
We can achieve a uniform bound if we assume that the input space is compact.

\begin{lemma} \label{lemma-8th-moment-stochastic}
Let $\mathcal{X}\subset \RR^M$ be a compact input space. 
Then for each $\mu\in\{1,\ldots,D+1\}$, the eighth moments
of the normally distributed random variables $f_i^{(\mu)}(x)[n]$ defined by equation \eqref{eq-f1}
are uniformly bounded over all $i\in\{1,\ldots,h_{\mu}(n)\}$, $n\in\NN$ and $x\in\mathcal{X}$.
\end{lemma}
\begin{proof}
We proceed by induction on $\mu$. 
The case $\mu=1$ is trivial; the random variables $f_i^{(1)}(x)[n]$ are IID over $i$ and follow 
the normal distribution $\calN(0, v_b^{(1)}+v_w^{(1)}\|x\|^2)$. 
The eighth moments are therefore
\[ E[f_i^{(1)}(x)[n]^8] = 105 (v_b^{(1)}+v_w^{(1)}\|x\|^2)^4. \]
Clearly the eighth moment is independent of $i$ and $n$. 
Moreover, since $\mathcal{X}$ is compact, then 
$\sup_{x\in\mathcal{X}} E[f_i^{(1)}(x)[n]^8] < \infty$. 
The eighth moments are therefore uniformly bounded over $i$, $n$ and 
$x$.

Now assume that the eighth moments of $f_i^{(\mu)}(x)[n]$ are uniformly bounded over $i$, $n$ and $x$ for all $\mu\in\{1,\ldots,t-1\}$ and for some $t\in \{2,\ldots,D+1\}$. 
We wish to prove that the eighth moments of $f_i^{(t)}(x)[n]$ are uniformly bounded over $i$, $n$ and $x$.
% https://en.wikipedia.org/wiki/Minkowski_inequality
Using the inequality $|u(x)+v(x)|^p\le 2^{p-1}(|u(x)|^p+|v(x)|^p)$ for elements $u$ and $v$ of the $L^p$ space for $p\ge 1$,
which follows from the convexity of $h(x):=x^p$ for $p>1$, the bound 
\[ E[f_i^{(t)}(x)[n]^8] \leq 2^{7}E\left[(b_i^{(t)})^8 + \left(\sum_{j=1}^{h_{t-1}(n)} w_{ij}^{(t)} g_j^{(t)}(x)[n]\right)^8\right]\]
is first established.
The term $E[(b_i^{(t)})^8]$ is bounded since the biases are normally distributed. 
Moreover, the biases are IID over $i$ and are independent of $n$. 
Therefore, to achieve the desired uniform bound, we only need to show that the term
\[ S_i(x)[n] := E\left[\left(\sum_{j=1}^{h_{t-1}(n)} w_{ij}^{(t)} g_j^{(t-1)}(x)[n]\right)^8\right] \]
is uniformly bounded over $i$, $n$ and $x$. 
By Lemma 20 in \citetalias{matthews2018gaussian},
\[ S_i(x)[n] \leq \frac{1}{h_{t-1}(n)^4} E\left[\left(\sum_{i=1}^{h_{t-1}(n)} (c^2+2cm |f_i^{(t-1)}(x)[n]| + m^2 |f_i^{(t-1)}(x)[n]|^2)\right)^4\right], \]
where $c,m > 0$ are constants from the linear envelope property of the activation function. 
Letting $a=\operatorname{max}\{c^2, 2cm, m^2\}$ and multiplying out the quantity in the above expectation, we have
\begin{align*}
S_i(x)[n]
% &\leq \frac{1}{h_{t-1}(n)^4} E\left[\left(\sum_{i=1}^{h_{t-1}(n)} (a+a |f_i^{(t-1)}(x)[n]| + a |f_i^{(t-1)}(x)[n]|^2)\right)^4\right] \\
% &\leq \frac{a^4}{h_{t-1}(n)^4} E\left[\left(\sum_{i=1}^{h_{t-1}(n)} (1+ |f_i^{(t-1)}(x)[n]| + |f_i^{(t-1)}(x)[n]|^2)\right)^4\right] \\
& \leq \frac{a^4}{h_{t-1}(n)^4} E\left[\sum_{i,j,k,\ell=1}^{h_{t-1}(n)} \sum_{p,q,r,s=0}^2 |f_i^{(t-1)}(x)[n]|^p \cdot |f_j^{(t-1)}(x)[n]|^q \right. \\
&  \cdot |f_k^{(t-1)}(x)[n]|^r \cdot |f_{\ell}^{(t-1)}(x)[n]|^s\Bigg] \\
&= \frac{a^4}{h_{t-1}(n)^4} \sum_{i,j,k,\ell=1}^{h_{t-1}(n)} \sum_{p,q,r,s=0}^2 E\left[|f_i^{(t-1)}(x)[n]|^p \cdot |f_j^{(t-1)}(x)[n]|^q \right. \\
& \left. \cdot |f_k^{(t-1)}(x)[n]|^r \cdot |f_{\ell}^{(t-1)}(x)[n]|^s\right].
\end{align*}
Using Lemma~\ref{lemma-4way-product} and the fact that the moments of $f_i^{(t-1)}(x)[n]$ are independent of $i$ by exchangeability, we have
\begin{align*}
S_i(x)[n]
&\leq \frac{a^4}{h_{t-1}(n)^4} \sum_{i,j,k,\ell=1}^{h_{t-1}(n)} \sum_{p,q,r,s=0}^2 \left( E[f_i^{(t-1)}(x)[n]^8]^{\frac{p}{8}} \cdot E[f_j^{(t-1)}(x)[n]^8]^{\frac{q}{8}} \right. \\
&\left. \cdot E[f_k^{(t-1)}(x)[n]^8]^{\frac{r}{8}} \cdot E[f_{\ell}^{(t-1)}(x)[n]^8]^{\frac{s}{8}}\right) \\
% &= \frac{a^4}{h_{t-1}(n)^4} \sum_{i,j,k,\ell=1}^{h_{t-1}(n)} \sum_{p,q,r,s=0}^2 E[f_1^{(t-1)}(x)[n]^8]^{\frac{p}{8}} E[f_1^{(t-1)}(x)[n]^8]^{\frac{q}{8}} E[f_1^{(t-1)}(x)[n]^8]^{\frac{r}{8}} E[f_1^{(t-1)}(x)[n]^8]^{\frac{s}{8}} \\
&= \frac{a^4}{h_{t-1}(n)^4} \sum_{i,j,k,\ell=1}^{h_{t-1}(n)} \sum_{p,q,r,s=0}^2 E[f_1^{(t-1)}(x)[n]^8]^{\frac{p+q+r+s}{8}} \\
&= \frac{a^4}{h_{t-1}(n)^4}\cdot h_{t-1}(n)^4 \sum_{p,q,r,s=0}^2 E[f_1^{(t-1)}(x)[n]^8]^{\frac{p+q+r+s}{8}} \\
&= a^4 \sum_{p,q,r,s=0}^2 E[f_1^{(t-1)}(x)[n]^8]^{\frac{p+q+r+s}{8}} \\
&= a^4 \sum_{j=1}^{81} E[f_1^{(t-1)}(x)[n]^8]^{m_j},
\end{align*}
where each $m_j$ is a rational number between $0$ and $1$. 
Define the function
\[ \psi(z) = a^4 \sum_{j=1}^{81} z^{m_j}, \]
and note that $a$ and the $m_j$ are independent of the hidden width index $n$, the hidden neuron index $i$, and the input $x$. 
Moreover, $\psi$ is increasing on the interval $(0, \infty)$. 
Since we assumed as our inductive hypothesis that $E[f_i^{(t-1)}(x)[n]^8] < \infty$, then it follows that
\[ \sup_{i,n,x} S_i(x)[n] \leq \psi(\sup_{i,n,x} E[f_1^{(t-1)}(x)[n]^8]) < \infty, \]
implying that $E[f_i^{(t)}(x)[n]^8] < \infty$ uniformly over $i$ $n$, and $x$, thereby completing the proof.
\end{proof}

The following lemma extends Lemma 21 in \citetalias{matthews2018gaussian} to stochastic processes in the same sense as Lemma~\ref{lemma-8th-moment-stochastic} above.

\begin{lemma} \label{lemma-unif-int-stochastic}
Let $\mathcal{X}\subset \RR^M$ be a compact input space. 
Then for any
$\mu\in\{1,\ldots,D+1\}$ and indices $i,j,k,\ell\in\NN$, the set of random variables
\[ S := \{g_i^{(\mu)}(x_1)[n] g_j^{(\mu)}(x_2)[n] g_k^{(\mu)}(x_3)[n] g_{\ell}^{(\mu)}(x_4)[n]: n\in\NN \mbox{ and } x_1,x_2,x_3,x_4\in\mathcal{X}\} \]
is uniformly integrable.
\end{lemma}
\begin{proof}
By the de la Vall\'{e}e-Poussin Theorem~\citep[p.19, Theorem T22]{meyer1966probability}, $S$ is uniformly integrable if
\[ \sup_{n, x_1, x_2, x_3, x_4} E[|g_i^{(\mu)}(x_1)[n] g_j^{(\mu)}(x_2)[n] g_k^{(\mu)}(x_3)[n] g_{\ell}^{(\mu)}(x_4)[n]|^{1+\eps}] < \infty \mbox{ for some } \eps > 0. \]
We consider $\eps=1$. 
By Lemma~\ref{lemma-4way-product},
\begin{align*}
\begin{aligned}
& E[|g_i^{(\mu)}(x_1)[n] g_j^{(\mu)}(x_2)[n] g_k^{(\mu)}(x_3)[n] g_{\ell}^{(\mu)}(x_4)[n]|^2] \\
% &= E[g_i^{(\mu)}(x_1)[n]^2 g_j^{(\mu)}(x_2)[n]^2 g_k^{(\mu)}(x_3)[n]^2 g_{\ell}^{(\mu)}(x_4)[n]^2] \\
& \leq E[g_i^{(\mu)}(x_1)[n]^8]^{\frac{1}{4}} E[g_j^{(\mu)}(x_2)[n]^8]^{\frac{1}{4}} E[g_k^{(\mu)}(x_3)[n]^8]^{\frac{1}{4}} E[g_{\ell}^{(\mu)}(x_4)[n]^8]^{\frac{1}{4}}\\
& = \prod_{q=1}^4 E[g_1^{(\mu)}(x_q)[n]^8]^{\frac{1}{4}},\\
\end{aligned}
\end{align*}
where we obtained the last line by exchangeability over the indices $i,j,k,\ell$. 
We therefore have
\begin{align*}
\sup_{n,x_1,x_2,x_3,x_4} E[|g_i^{(\mu)}(x_1)[n] g_j^{(\mu)}(x_2)[n] g_k^{(\mu)}(x_3)[n] g_{\ell}^{(\mu)}(x_4)[n]|^2]
&\leq \prod_{q=1}^4 \sup_{n,x_q} E[g_1^{(\mu)}(x_q)[n]^8]^{\frac{1}{4}} \\
&= \prod_{q=1}^4 \sup_{n,x} E[g_1^{(\mu)}(x)[n]^8]^{\frac{1}{4}} \\
%&= \prod_{q=1}^4 \left(\sup_{n,x} E[g_1^{(\mu)}(x)[n]^8]\right)^{\frac{1}{4}} \\
&= \sup_{n,x} E[g_1^{(\mu)}(x)[n]^8].
\end{align*}
It thus suffices to show that the supremum in the last line is finite. 
By the linear envelope property of the activation function,
\[ E[g_1^{(\mu)}(x)[n]^8] \leq 2^{7} \left(c^8+m^8 E[f_1^{(\mu)}(x)[n]^8]\right). \]
By Lemma~\ref{lemma-8th-moment-stochastic}, the right-hand side is uniformly bounded over all $n\in\NN$ and $x\in\mathcal{X}$, completing the proof.
\end{proof}

Cors.~\ref{cor-unif-int-stochastic} and~\ref{cor-unif-int-stochastic-2} of
Lemma~\ref{lemma-unif-int-stochastic}, below, are used in Lemmas~\ref{lemma-condition-2} and~\ref{lemma-condition-variance}, respectively.

\begin{corollary} \label{cor-unif-int-stochastic}
Let $\{x_q[n]\in\RR^M\}_{n=1}^\infty$ for $q\in\{1,\ldots,4\}$ be four convergent sequences with finite limits. 
Then for any $\mu\in\{1,\ldots,D+1\}$ and indices $i,j,k,\ell\in\NN$, the set of random variables
\[ S = \{g_i^{(\mu)}(x_1[n])[n] g_j^{(\mu)}(x_2[n])[n] g_k^{(\mu)}(x_3[n])[n] g_{\ell}^{(\mu)}(x_4[n])[n]: n\in\NN\} \]
is uniformly integrable.
\end{corollary}
\begin{proof}
Since the sequences $\{x_q[n]\}_{n=1}^\infty$ for $q\in\{1,\ldots,4\}$, converge to finite limits, then there exists a compact set $\mathcal{X}\subset\RR^M$ that contains $x_q[n]$ for all $n$ and $q$. 
By (the proof of) Lemma~\ref{lemma-unif-int-stochastic}, we have
\[ \sup_{n,x_1,x_2,x_3,x_4} E[|g_i^{(\mu)}(x_1)[n] g_j^{(\mu)}(x_2)[n] g_k^{(\mu)}(x_3)[n] g_{\ell}^{(\mu)}(x_4)[n]|^2] < \infty, \]
where the $x_q$ are elements of the compact set $\mathcal{X}$. 
It then holds in particular that
\[ \sup_n E[|g_i^{(\mu)}(x_1[n])[n] g_j^{(\mu)}(x_2[n])[n] g_k^{(\mu)}(x_3[n])[n] g_{\ell}^{(\mu)}(x_4[n])[n]|^2] < \infty. \]
Uniform integrability then follows by the de la Vallee-Poussin Theorem.
\end{proof}

\begin{corollary} \label{cor-unif-int-stochastic-2}
Let $\{x_q[n]\in\RR^M\}_{n=1}^\infty$ for $q\in\{1,2\}$, be two convergent sequences with finite limits. 
Then for any $\mu\in\{1,\ldots,D+1\}$ and indices $i,j\in\NN$, the set of random variables
\[ S = \{g_i^{(\mu)}(x_1[n])[n] g_j^{(\mu)}(x_2[n])[n]: n\in\NN\} \]
is uniformly integrable.
\end{corollary}
\begin{proof}
By the de la Vall\'{e}e-Poussin Theorem, $S$ is uniformly integrable if
\[ \sup_{n,x_1,x_2,x_3,x_4} E[|g_i^{(\mu)}(x_1)[n] g_j^{(\mu)}(x_2)[n]|^{1+\eps}] < \infty \mbox{ for some } \eps > 0. \]
We consider $\eps=3$. 
We have
\begin{align*}
\begin{aligned}
& \sup_{n,x_1,x_2,x_3,x_4} E[|g_i^{(\mu)}(x_1)[n] g_j^{(\mu)}(x_2)[n]|^4] \\
% &= \sup_{n,x_1,x_2,x_3,x_4} E[g_i^{(\mu)}(x_1)[n]^4 g_j^{(\mu)}(x_2)[n]^4] \\
&= \sup_{n,x_1,x_2,x_3,x_4} E[g_i^{(\mu)}(x_1)[n]^2 g_i^{(\mu)}(x_1)[n]^2 g_j^{(\mu)}(x_2)[n]^2 g_j^{(\mu)}(x_2)[n]^2],
\end{aligned}
\end{align*}
which is finite by (the proof of) Cor.~\ref{cor-unif-int-stochastic}. 
The claim then follows.
\end{proof}

\section{Correspondence to the no-bottleneck NNGP}
\label{appendix-section-correspondence}

The following is our proof of the Wide Bottleneck Correspondence Theorem for the case of a single-bottleneck NNGP.

\begin{proof}[Proof of Thm.~\ref{thm-correspondence}]
First we prove statement~\ref{thm-correspondence:b}. 
We will do so for $L=1$; the case $L>1$ proceeds similarly. 
Let $X = \{x_t\}_{t=1}^T$ be a finite set of inputs. 
Let $p_H:\RR^T\mapsto\RR$ and $p:\RR^T\mapsto\RR$ be the PDFs of $F^{(H)}$ and $F$, respectively. 
Let $k^{(D_1)}:\RR^M\times\RR^M\mapsto\RR$ and $k^{(D_2)}:\RR^H\times\RR^H\mapsto\RR$ be the NNGP kernels of the pre-bottleneck and post-bottleneck components (with respective depths $D_1$ and $D_2$) of $F^{(H)}$; 
note that the kernels are independent of $H$. 
Then the PDF $p_H$ is given by
\begin{align} \label{eq-pH}
\begin{aligned}
&p_H(y) 
=\int_{(\RR^T)^H} \Bigg(\calN\left(y; 0, k^{(D_2)}\left(\frac{1}{\sqrt{H}}\phi(\{h_i\}_{i=1}^H), \frac{1}{\sqrt{H}}\phi(\{h_j\}_{j=1}^H)\right) + v_n \mathbf{I}_T\right) \\
& \left.\cdot \prod_{k=1}^H \calN(h_k; 0, k^{(D_1)}(X, X))\right)\,\mathrm{d}h_k,
\end{aligned}
\end{align}
where $h_i\in\RR^T$ is the vector of preactivations into the $i$-th hidden neuron in the bottleneck, 
and where we use the notation $\calN(z; \mu, \Sigma)$ to mean the normal PDF in the variable $z$ with mean $\mu$ and covariance $\Sigma$. 
Observing that the NNGP kernel in Eqs.~\eqref{eq-K1}-\eqref{eq-Kmu} depends on its inputs only through their Gram matrix and writing the kernel of the first layer explicitly, 
we can define a function $\tilde{k}^{(D_2)}$ on the space of symmetric positive semidefinite matrices such that
\begin{equation*}
\tilde{k}^{(D_2)}(v_b + v_w A) = k^{(D_2)}(B, B),
\quad A = BB^\top.
\end{equation*}
Defining the random $T\times T$ matrix
\begin{equation} \label{eq-ZH}
Z_H = \frac{1}{H}\sum_{i=1}^H \phi(h_i)\phi(h_i)^\top,
\quad h_i\sim \calN(0, k^{(D_1)}(X, X)) \mbox{ IID},
\end{equation}
and letting $\mu_H$ denote the probability measure associated with $Z_H$, the PDF in Eq.~\eqref{eq-pH} can be written as
\begin{equation*}
p_H(y) = \int_{\RR^{T\times T}} \calN(y; 0, \tilde{k}^{(D_2)}(v_b + v_w z)+v_n \mathbf{I}_T)\,\mathrm{d}\mu_H(z).
\end{equation*}
Here $z$ is a dummy variable.
Now since the $h_i$ in Eq.~\eqref{eq-ZH} are IID, then so are the matrices $\phi(h_i)\phi(h_i)^\top$. 
Therefore, $Z_H$ is an empirical average of $H$ IID random matrices.
By the Law of Large Numbers, we have
\begin{equation*}
\{Z_H\}_{H=1}^\infty \xrightarrow{P} Z = \operatorname{E}_{h\sim\calN(0, k^{(D_1)}(X, X)}[\phi(h)\phi(h)^\top],
\end{equation*}
where the convergence is in probability. 
In particular, $\{Z_H\}_{H=1}^\infty\xrightarrow{D} Z$ so that the sequence of measures $\{\mu_H\}_{H=1}^\infty$ weakly converges to the probability measure $\mu$ associated with $Z$. 
Note that $\mu$ is a delta distribution concentrated at $Z$. 
Furthermore, thanks to the Gaussian noise, the function $z\rightarrow \calN(y; 0, \tilde{k}^{(D_2)}(z)+v_n\mathbf{I}_T)$ is bounded over $\RR^{T\times T}$; 
it is continuous as well, as the matrix inversion and determination operations and the NNGP kernel (Lemma~\ref{lemma-cont}) are all continuous. 
By the weak convergence of measures and the delta distribution $\mu$, we have
\begin{align*}
\lim_{H\rightarrow\infty} p_H(y)
&= \lim_{H\rightarrow\infty} \int_{\RR^{T\times T}} \calN(y; 0, \tilde{k}^{(D_2)}(v_b + v_w z)+v_n \mathbf{I}_T)\,\mathrm{d}\mu_H(z) \\
&= \int_{\RR^{T\times T}} \calN(y; 0, \tilde{k}^{(D_2)}(v_b + v_w z)+v_n \mathbf{I}_T)\,\mathrm{d}\mu(z) \\
&= \calN\left(y; 0, \tilde{k}^{(D_2)}\left(v_b + v_w \operatorname{E}_{h\sim\calN(0, k^{(D_1)}(X, X)}[\phi(h)\phi(h)^\top]\right)+v_n \mathbf{I}_T\right) \\
&= \calN\left(y; 0, \tilde{k}^{(D_2)}\left(k^{(D_1+1)}(X, X)\right)+v_n \mathbf{I}_T\right) \\
&= \calN(y; 0, k^{(D_1+D_2+1)}(X, X)+v_n \mathbf{I}_T) \\
&= p(y),
\end{align*}
which is the PDF of an NNGP with $D_1+D_2+1$ hidden layers.

To prove statement~\ref{thm-correspondence:a} of the theorem, we first note that the pointwise convergence $p_H\rightarrow p$ ensures the convegence in distribution $F^{(H)}(X)\xrightarrow{D} F(X)$ according to Scheff\'{e}'s Lemma.
Since this holds for any finite set of inputs $X$ and in particular any finite subset of a countable set $\mathcal{X}\subset \RR^M$, then by Thm.~\ref{thm-marginal}, we have that $\{F^{(H)}\}_{H=1}^\infty\xrightarrow{D} F$ in $((\RR^L)^\infty, \mathcal{A})$ for inputs restricted to $\mathcal{X}$ as claimed.
\end{proof}

\section{Bottleneck layers induce correlation}
\label{appendix-section-cor}

Recall the single-bottleneck NNGP $F$ defined in Sec.~\ref{subsection-cor-exact-formula}. 
Each output $(F_i(x_1), F_i(x_2))$ conditional on the activations of the bottleneck layer follow the two-dimensional normal distribution $\mathcal{N}(0, K)$ where
\[ K = \begin{bmatrix} k_{11} & k_{12} \\ k_{21} & k_{22} \end{bmatrix}. \]
It can be shown that the diagonal entries of $K$ are given by
\begin{align*}
k_{aa} &= b_D + \frac{w_D}{H} \sum_{i=1}^H \phi(h_i^a)^2, \\
b_D &= v_n + v_b \sum_{d=0}^{D-1} v_w^d, \\ %\label{eq-bD} \\
w_D &= v_w^D. %\label{eq-wD}
\end{align*}
The expression for the off-diagonals $k_{12}, k_{21}$ will not be important.

The proof of Prop.~\ref{prop-quad-cor} regarding the quadratic correlation between bottleneck NNGP outputs follows.

\begin{proof}[Proof of Prop.~\ref{prop-quad-cor}]
We have
\begin{align} \label{eq-E-Y1a2}
\begin{aligned}
\operatorname{E}[F_1(x_a)^2]
&= \int_{(\RR^2)^H} \int_{(\RR^2)^2} (y_1^a)^2 \calN(y_1; 0, K) \calN(y_2; 0, K) \prod_{m=1}^H \calN(h_m; 0, C)\,\mathrm{d}y\,\mathrm{d}h \\
&= \int_{(\RR^2)^H} k_{aa} \prod_{m=1}^H \calN(h_m; 0, C)\,\mathrm{d}h \\
&= \int_{(\RR^2)^H} \left(b_D + \frac{w_D}{H}\sum_{i=1}^H \phi(h_i^a)^2\right) \prod_{m=1}^H \calN(h_m; 0, C)\,\mathrm{d}h \\
&= b_D + \frac{w_D}{H} \sum_{i=1}^H \int_{(\RR^2)^H} \phi(h_i^a)^2 \prod_{m=1}^H \calN(h_m; 0, C)\,\mathrm{d}h \\
    &= b_D + \frac{w_D}{H} \sum_{i=1}^H \int_{(\RR^2)^H} \phi(h_i^a)^2 \calN(h_i; 0, C)\,\mathrm{d}h_i \\
&= b_D + \frac{w_D}{H} \sum_{i=1}^H E_{z\sim\calN(0, c_{aa})}[\phi(z)^2] \\
&= b_D + w_D E_{z\sim\calN(0, c_{aa})}[\phi(z)^2].
\end{aligned}
\end{align}
We similarly have
\begin{equation}
 \label{eq-E-Y2b2}
\operatorname{E}[F_2(x_b)^2]
= b_D + w_D E_{z\sim\calN(0, c_{bb})}[\phi(z)^2].
\end{equation}
Combining  Eqs.~\eqref{eq-E-Y1a2} and  \eqref{eq-E-Y2b2} yields
\begin{align*}
\begin{aligned}
&\operatorname{E}[F_1(x_a)^2] \operatorname{E}[F_2(x_b)^2]\\
&= b_D^2 + b_D w_D E_{z\sim\calN(0, c_{aa})}[\phi(z)^2] + b_D w_D E_{z\sim\calN(0, c_{bb})}[\phi(z)^2]\\
&+ E_{z\sim\calN(0, c_{aa})}[\phi(z)^2] E_{z\sim\calN(0, c_{bb})}[\phi(z)^2].
\end{aligned}
\end{align*}

We also have
\begin{align*}
\begin{aligned}
&\operatorname{E}[F_1(x_a)^2 F_2(x_b)^2]\\
% &= \int_{(\RR^2)^H} \int_{(\RR^2)^2} (y_1^a)^2 (y_2^b)^2 \calN(y_1; 0, K) \calN(y_2; 0, K) \prod_{m=1}^H \calN(h_m; 0, C)\,\mathrm{d}y\,\mathrm{d}h \\
&= \int_{(\RR^2)^H} \int_{(\RR^2)^2} \prod_{y=y_1^a,y_2^b}y^2 \calN(y; 0, K) \prod_{m=1}^H \calN(h_m; 0, C)\,\mathrm{d}y\,\mathrm{d}h \\
&= \int_{(\RR^2)^H} k_{aa} k_{bb} \prod_{m=1}^H \calN(h_m; 0, C)\,\mathrm{d}h \\
% &= \int_{(\RR^2)^H} \left(b_D + \frac{w_D}{H}\sum_{i=1}^H \phi(h_i^a)^2\right) \left(b_D + \frac{w_D}{H} \sum_{j=1}^H \phi(h_j^b)^2\right) \prod_{m=1}^H \calN(h_m; 0, C)\,\mathrm{d}h \\
&= \int_{(\RR^2)^H} \prod_{h=h_i^a,h_j^b} \left(b_D + \frac{w_D}{H} \sum_{j=1}^H \phi(h)^2\right) \prod_{m=1}^H \calN(h_m; 0, C)\,\mathrm{d}h \\
% &= \int_{(\RR^2)^H} \left(b_D + \frac{w_D}{H}\sum_{i=1}^H \phi(h_i^a)^2\right) \left(b_D + \frac{w_D}{H} \sum_{j=1}^H ]\phi(h_j^b)^2\right) \prod_{m=1}^H \calN(h_m; 0, C)\,\mathrm{d}h \\
% &= b_D^2 + \frac{b_D w_D}{H} \sum_{i=1}^H \int_{(\RR^2)^H} \phi(h_i^a)^2 + \frac{b_D w_D}{H} \sum_{j=1}^H \int_{(\RR^2)^H} ]\phi(h_j^b)^2 \prod_{m=1}^H \calN(h_m; 0, C)\,\mathrm{d}h \\
% &\quad + \frac{w_D^2}{H^2} \sum_{i,j=1}^H \int_{(\RR^2)^H} \phi(h_i^a)^2 ]\phi(h_j^b)^2 \prod_{m=1}^H \calN(h_m; 0, C)\,\mathrm{d}h \\
&= b_D^2 + b_D w_D E_{z\sim\calN(0, c_{aa})}[\phi(z)^2] + b_D w_D E_{z\sim\calN(0, c_{bb})}[\phi(z)^2] \\
&+ \frac{w_D^2}{H^2} \sum_{i,j=1}^H \int_{(\RR^2)^2} \phi(h_i^a)^2 \phi(h_j^b)^2 \calN(h_i; 0, C)\calN(h_j; 0, C)\,\mathrm{d}h_i\,\mathrm{d}h_j \\
&= \operatorname{E}[F_1(x_a)^2] \operatorname{E}[F_2(x_b)^2] - E_{z\sim\calN(0, c_{aa})}[\phi(z)^2] E_{z\sim\calN(0, c_{bb})}[\phi(z)^2] \\
&+ \frac{w_D^2}{H^2} \sum_{i=1}^H \int_{\RR^2} \phi(h_i^a)^2 \phi(h_i^b)^2 \calN(h_i; 0, C)\,\mathrm{d}h_i \\
&+ \frac{w_D^2}{H^2} \sum_{i\neq j=1}^H \int_{(\RR^2)^2} \phi(h_i^a)^2 \phi(h_j^b)^2 \calN(h_i; 0, C)\calN(h_j; 0, C)\,\mathrm{d}h_i\,\mathrm{d}h_j \\
% &= \operatorname{E}[F_1(x_a)^2] \operatorname{E}[F_2(x_b)^2] - E_{z\sim\calN(0, c_{aa})}[\phi(z)^2] E_{z\sim\calN(0, c_{bb})}[\phi(z)^2] \\
% &+ \frac{w_D^2}{H^2} \sum_{i=1}^H E_{(z,z^\prime)\sim\calN(0, C)}[\phi(z)^2\phi(z^\prime)^2] \\
% &+ \frac{w_D^2}{H^2} \sum_{i\neq j=1}^H E_{z\sim\calN(0, c_{aa})}[\phi(z)^2] E_{z\sim\calN(0, c_{bb})}[\phi(z)^2] \\
&= \operatorname{E}[F_1(x_a)^2] \operatorname{E}[F_2(x_b)^2] - E_{z\sim\calN(0, c_{aa})}[\phi(z)^2] E_{z\sim\calN(0, c_{bb})}[\phi(z)^2] \\
&+ \frac{w_D^2}{H} E_{(z,z^\prime)\sim\calN(0, C)}[\phi(z)^2\phi(z^\prime)^2] \\
&+ w_D^2\left(1-\frac{1}{H}\right) E_{z\sim\calN(0, c_{aa})}[\phi(z)^2] E_{z\sim\calN(0, c_{bb})}[\phi(z)^2] \\
&= \operatorname{E}[F_1(x_a)^2] \operatorname{E}[F_2(x_b)^2]
+ \frac{w_D^2}{H} E_{(z,z^\prime)\sim\calN(0, C)}[\phi(z)^2\phi(z^\prime)^2] \\
&- \frac{w_D^2}{H} E_{z\sim\calN(0, c_{aa})}[\phi(z)^2] E_{z\sim\calN(0, c_{bb})}[\phi(z)^2] \\
&= \operatorname{E}[F_1(x_a)^2] \operatorname{E}[F_2(x_b)^2] + \frac{w_D^2}{H} \cov_{(z,z^\prime)\sim\calN(0, C)}[\phi(z)^2, \phi(z^\prime)^2].
\end{aligned}
\end{align*}
We therefore have
\begin{equation*}
\cov[F_1(x_a)^2, F_2(x_b)^2] = \frac{w_D^2}{H} \cov_{(z,z^\prime)\sim\calN(0, C)}[\phi(z)^2, \phi(z^\prime)^2].
\end{equation*}
On the right-hand side, we have the covariance of two rectified quadratic units with respect to the Gaussian measure $\calN(0, C)$. 
By the work of \cite{cho2009kernel} and by adjusting for differences in normalization, we have the expectations
\begin{align}
E_{z\sim\calN(0, c_{aa})}[\phi(z)^2] &= c_{aa} \label{eq-E-z}, \\
E_{(z,z^\prime)\sim\calN(0, C)}[\phi(z)^2\phi(z^\prime)^2] &= \frac{2}{\pi}c_{aa}c_{bb}J_2(\beta),
\end{align}
where $\beta = \cos^{-1}\left(\frac{c_{ab}}{\sqrt{c_{aa}c_{bb}}}\right)$ and 
$J_2(\beta) = 3\sin\beta\cos\beta + (\pi-\beta)(1 + 2\cos^2\beta)$. 
Using also the fact that $w_D = v_w^D$, we obtain
\begin{align*}
\cov[F_1(x_a)^2, F_2(x_b)^2]
&= \frac{(v_w^D)^2}{H}\left(\frac{2}{\pi}c_{aa}c_{bb}J_2(\beta)-c_{aa}c_{bb}\right) \\
&= \frac{v_w^{2D}c_{aa}c_{bb}}{H}\left(\frac{2}{\pi}J_2(\beta)-1\right),
\end{align*}
establishing Eq.~\eqref{eq-square-cov}.

The corresponding correlation is defined as
\begin{equation} \label{eq-cor-1}
q^{\times}_{ab} = \frac{\cov[F_1(x_a)^2, F_2(x_b)^2]}{\sqrt{\operatorname{V}[F_1(x_a)^2] \operatorname{V}[F_2(x_b)^2]}}.
\end{equation}
We already know the numerator on the right-hand side, but we need to calculate the variances in the denominator. 
Using the fact that $F_1(x_a)$ and $F_2(x_b)$ are identically (but not independently) distributed, we have
\begin{align*}
\begin{aligned}
&\operatorname{V}[F_1(x_a)^2]\\
&= \operatorname{E}[F_1(x_a)^4] - \operatorname{E}[F_1(x_a)^2]^2 \\
&= \int_{(\RR^2)^H} \int_{(\RR^2)^2} (y_1^a)^4 \prod_{i=1}^{2}\calN(y_i; 0, K) \prod_{m=1}^H \calN(h_m; 0, C)\,\mathrm{d}y\,\mathrm{d}h - \operatorname{E}[F_1(x_a)^2]^2 \\
% &= \int_{(\RR^2)^H} \int_{(\RR^2)^2} (y_1^a)^4 \calN(y_1; 0, K) \calN(y_2; 0, K) \prod_{m=1}^H \calN(h_m; 0, C)\,\mathrm{d}y\,\mathrm{d}h - \operatorname{E}[F_1(x_a)^2]^2 \\
% &= \int_{(\RR^2)^H} 3(k_{aa})^2 \prod_{m=1}^H \calN(h_m; 0, C)\,\mathrm{d}h - \operatorname{E}[F_1(x_a)^2]^2 \\
&= 3\int_{(\RR^2)^H} k_{aa}k_{aa} \prod_{m=1}^H \calN(h_m; 0, C)\,\mathrm{d}h - \operatorname{E}[F_1(x_a)^2]^2 \\
&= 3\operatorname{E}[F_1(x_a)^2, F_2(x_a)^2] - \operatorname{E}[F_1(x_a)^2] \operatorname{E}[F_2(x_a)^2] \\
&= 3\cov[F_1(x_a)^2, F_2(x_a)^2] + 3\operatorname{E}[F_1(x_a)^2] \operatorname{E}[F_2(x_a)^2] - \operatorname{E}[F_1(x_a)^2] \operatorname{E}[F_2(x_a)^2] \\
&= 3\cov[F_1(x_a)^2, F_2(x_a)^2] + 2\operatorname{E}[F_1(x_a)^2] \operatorname{E}[F_2(x_a)^2] \\
&= 3\cov[F_1(x_a)^2, F_2(x_a)^2] + 2\operatorname{E}[F_1(x_a)^2]^2.
\end{aligned}
\end{align*}
By Eqs.~\eqref{eq-E-Y1a2}, \eqref{eq-E-z}, and \eqref{eq-square-cov}, we have
\begin{align}
\label{eq-prop-quad-cor-square-var}
\begin{aligned}
\operatorname{V}[F_1(x_a)^2]
&= \frac{3w_D^2 c_{aa}c_{aa}}{H}\left(\frac{2}{\pi}J_2(0)-1\right) + 2\left(b_D + w_D c_{aa}\right)^2 \\
&= \frac{3w_D^2(c_{aa})^2}{H}\left(\frac{2\cdot 3\pi}{\pi}-1\right) + 2\left(b_D + w_D c_{aa}\right)^2 \\
&= \frac{15w_D^2(c_{aa})^2}{H} + 2\left(b_D + w_D c_{aa}\right)^2.
\end{aligned}
\end{align}
We similarly have
\begin{equation*}
\operatorname{V}[F_2(x_b)^2] = \frac{15w_D^2(c_{bb})^2}{H} + 2\left(b_D + w_D c_{bb}\right)^2.
\end{equation*}
Substituting these variances into Eq.~\eqref{eq-cor-1}, we obtain
\begin{align*}
q^{\times}_{ab}
% &= \corr[F_1(x_a)^2, F_2(x_b)^2] \\
% &= \frac{\frac{w_D^2 c_{aa}c_{bb}}{H}\left(\frac{2}{\pi}J_2(\beta)-1\right)}{\sqrt{\left[\frac{15w_D^2(c_{aa})^2}{H} + 2\left(b_D + w_D c_{aa}\right)^2\right]\left[\frac{15w_D^2(c_{bb})^2}{H} + 2\left(b_D + w_D c_{bb}\right)^2\right]}} \\
% &= \frac{\left(\frac{2}{\pi}J_2(\beta)-1\right)}{\sqrt{\frac{H}{w_D (c_{aa})^2}\left[\frac{15w_D^2(c_{aa})^2}{H} + 2\left(b_D + w_D c_{aa}\right)^2\right]\frac{H}{w_D(c_{bb})^2}\left[\frac{15w_D^2(c_{bb})^2}{H} + 2\left(b_D + w_D c_{bb}\right)^2\right]}} \\
% &= \frac{\left(\frac{2}{\pi}J_2(\beta)-1\right)}{\sqrt{\left[15 + 2H\left(\frac{r_D}{c_{aa}} + 1\right)^2\right]\left[15 + 2H\left(\frac{r_D}{c_{bb}} + 1\right)^2\right]}},
&= \frac{\frac{w_D^2 c_{aa}c_{bb}}{H}\left(\frac{2}{\pi}J_2(\beta)-1\right)}{\displaystyle\prod_{c=c_{aa},c_{bb}}\sqrt{\frac{15w_D^2c^2}{H} + 2\left(b_D + w_D c\right)^2}} \\
&= \frac{\left(\frac{2}{\pi}J_2(\beta)-1\right)}{\displaystyle\prod_{c=c_{aa},c_{bb}}\sqrt{\frac{H}{w_D c^2}\left(\frac{15w_D^2c^2}{H} + 2\left(b_D + w_D c\right)^2\right)}} \\
&= \frac{\left(\frac{2}{\pi}J_2(\beta)-1\right)}{\displaystyle\prod_{c=c_{aa},c_{bb}}\sqrt{15 + 2H\left(\frac{r_D}{c} + 1\right)^2}},
\end{align*}
where
\begin{align*}
r_D
&= \frac{b_D}{w_D} \\
&= \frac{v_n}{v_w^D} + \frac{v_b}{v_w^D}\sum_{d=0}^{D-1} v_w^d \\
&= \frac{v_n}{v_w^D} + v_b \sum_{d=1}^D \frac{1}{v_w^d} \\
&= \begin{cases}
v_n + D v_b & \mbox{ if } v_w = 1 \\
\frac{v_n}{v_w^D} + \frac{v_b}{1-v_w}\left(\frac{1}{v_w^D}-1\right) & \mbox{ otherwise},
\end{cases}
\end{align*}
establishing Eqs.~\eqref{eq-square-cor} and \eqref{eq-rD}.
\end{proof}

\begin{proof}[Proof of Prop.~\ref{prop-quad-cor-lim}]
For part~\ref{prop-quad-cor-lim:a}, Eq.~\eqref{eq-square-cor-lim} follows by substitution of Eq.~\eqref{eq-r_infty} into Eq.~\eqref{eq-square-cor}.

For part~\ref{prop-quad-cor-lim:b}, the map $G\mapsto Q^{\times(\infty)}$ is a composition of the maps $G\mapsto C$ and $C\mapsto Q^{\times(\infty)}$, 
and thus it suffices to show that these two maps are invertible. 
The map $G\mapsto C$ sends the input Gram matrix to the NNGP kernel at the bottleneck layer. 
Inverting Eq.~\eqref{eq-Kmu-relu} for the case $i=j$, we obtain the recursion for the backward propagation of the NNGP kernel:
\begin{equation*}
K^{(\mu-1)}_{ii}(x_1, x_2) = \frac{1}{v_w}\left(\prod_{a=1,2}\sqrt{K^{(\mu)}_{ii}(x_a, x_a)-v_b}\right)\cos J_1^{-1}\left(\pi \frac{K^{(\mu)}_{ii}(x_1, x_2)-v_b}{\prod_{a=1,2}\sqrt{K^{(\mu)}_{ii}(x_a, x_a)-v_b}}\right),
\end{equation*}
where we note $J_1$ is strictly decreasing on $[0, \pi]$. 
Applying this recursion to $C$ $d$ times (where $d$ is the depth of the pre-bottleneck NNGP) gives $K^{(1)}_{ii}(x_1, x_2)$, 
and by solving Eq.~\eqref{eq-K1} we obtain $G$. 
Thus, $G\mapsto C$ is invertible.

To show $C\mapsto Q^{\times(\infty)}$ is invertible, we inspect Eq.~\eqref{eq-square-cor-lim} for the case $v_w>1$ in Prop.~\ref{prop-quad-cor-lim} and observe that $q^{\times(\infty)}_{aa}$ depends only on $c_{aa}$ 
(the bottleneck angle $\beta$ is $0$ when the two inputs are identical). 
We may then solve for $c_{aa}$. 
Substituting $c_{aa}$ for $a\in\{1,2\}$ into Eq.~\eqref{eq-square-cor-lim} and noting that $J_2$ is strictly decreasing, we may solve for the bottleneck angle $\beta$ from $q^{\times(\infty)}_{12}$ and thus obtain $c_{12}$, recovering $C$. 
\end{proof}

\begin{proof}[Proof of Prop.~\ref{prop-depth-scale}]
	\label{proof:prop-depth-scale}
	Since $\|x_1\|=\|x_2\|$, then $c_{11}=c_{22}$. 
	Letting $c=c_{11}=c_{22}$, we have by Eq.~\eqref{eq-square-cor} that
	\[ q^{\times(D)}_{ab} = \frac{\left(\frac{2}{\pi}J_2(\beta)-1\right)}{15+2H\left(\frac{r_D}{c}+1\right)^2}. \]
	We will find a $\lambda > 0$ such that the limit $L$ in Eq.~\eqref{eq-asymp-exp-L} is finite and non-zero. 
	Note that while evaluating the limit, we will drop (non-zero) constants of proportionality. 
	Observing that both the numerator and denominator inside the limit $L$ in Eq.~\eqref{eq-asymp-exp-L} tend to $0$ as $D\rightarrow\infty$ and using L'H\^{o}pital's rule, we have
	\begin{align}
	\label{eq-asymp-exp-2}
	\begin{aligned}
	L 
	&= \limD \frac{\frac{\mathrm{d}}{\mathrm{d}D} q^{\times(D)}_{ab} - \frac{\mathrm{d}}{\mathrm{d}D}q^{\times(\infty)}_{ab}}{\frac{\mathrm{d}}{\mathrm{d}D}e^{-\frac{D}{\lambda}}} \\
	&= \limD \frac{\frac{\mathrm{d}}{\mathrm{d}D} \frac{\left(\frac{2}{\pi}J_2(\beta)-1\right)}{15+2H\left(\frac{r_D}{c}+1\right)^2} - 0}{-\frac{1}{\lambda}e^{-\frac{D}{\lambda}}} \\
	&\propto -\limD e^{\frac{D}{\lambda}}\cdot\frac{\mathrm{d}}{\mathrm{d}D}\frac{1}{15+2H\left(\frac{r_D}{c}+1\right)^2} \\
	&= \limD e^{\frac{D}{\lambda}}\cdot \frac{4H\left(\frac{r_D}{c}+1\right)\frac{1}{c}\cdot\left(v_n+\frac{v_b}{1-v_w}\right)\cdot\frac{1}{v_w^D}\ln\left(\frac{1}{v_w}\right)}{\left[15+2H\left(\frac{r_D}{c}+1\right)^2\right]^2} \\
	&\propto \limD \frac{e^{\frac{D}{\lambda}}}{v_w^D}\cdot \frac{\left(\frac{r_D}{c}+1\right)}{\left[15+2H\left(\frac{r_D}{c}+1\right)^2\right]^2}.
	\end{aligned}
	\end{align}
	
	In the case $v_w > 1$, $r_D$ tends to a finite positive limit as $D\rightarrow\infty$, so that the second fraction in the limit in Eq.~\eqref{eq-asymp-exp-2} tends to a finite positive limit as well. We therefore have
	\[ L
	\propto \limD \frac{e^{\frac{D}{\lambda}}}{v_w^D}, \]
	which is finite and non-zero (and exists) if and only if $\frac{e^{\frac{1}{\lambda}}}{v_w} = 1$, implying $\lambda = \ln(v_w)^{-1}$.
	
	In the case $v_w < 1$, $r_D\rightarrow\infty$ as $D\rightarrow\infty$, so that Eq.~\eqref{eq-asymp-exp-2} simplifies to
	\begin{align}
	L
	&\propto \limD \frac{e^{\frac{D}{\lambda}}}{v_w^D}\cdot \frac{\left(\frac{r_D}{c}\right)}{\left[2H\left(\frac{r_D}{c}\right)^2\right]^2} \nonumber \\
	&\propto \limD \frac{e^{\frac{D}{\lambda}}}{v_w^D}\cdot \frac{1}{r_D^3} \label{eq-asymp-exp-3} \\
	&= \limD \frac{e^{\frac{D}{\lambda}}}{v_w^D}\cdot \frac{1}{\left[\frac{v_n}{v_w^D}+\frac{v_b}{1-v_w}\left(\frac{1}{v_w^D}-1\right)\right]^3} \nonumber \\
	&= \limD \frac{e^{\frac{D}{\lambda}}}{v_w^D}\cdot \frac{1}{\left[\left(v_n+\frac{v_b}{1-v_w}\right)\frac{1}{v_w^D}-\frac{v_b}{1-v_w}\right]^3}. \nonumber
	\end{align}
	Since $\frac{1}{v_w^D}\rightarrow\infty$ as $D\rightarrow\infty$, then
	\[ L
	\propto \limD \frac{e^{\frac{D}{\lambda}}}{v_w^D}\cdot \frac{1}{\left[\left(v_n+\frac{v_b}{1-v_w}\right)\frac{1}{v_w^D}\right]^3} 
	\propto \limD v_w^{2D} e^{\frac{D}{\lambda}}, \]
	which is finite and non-zero (and exists) if and only if $v_w^2 e^{\frac{1}{\lambda}} = 1$, implying $\lambda = \ln\left(\frac{1}{v_w^2}\right)^{-1}$.
	
	In the case $v_w=1$, we again have $r_D\rightarrow\infty$ as $D\rightarrow\infty$ and thus still obtain Eq.~\eqref{eq-asymp-exp-3}. 
	Substituting $v_w=1$ and $r_D=v_n+v_b D$ into Eq.~\eqref{eq-asymp-exp-3} gives
	\[ L \propto \limD \frac{e^{\frac{D}{\lambda}}}{(v_n + v_b D)^3}, \]
	which is infinite for all finite $\lambda > 0$.
\end{proof}

\begin{proof}[Proof of Prop.~\ref{prop-q}]
	\label{proof:prop-q}
	For part~\ref{prop-q:a}, let
	\[ \mathrm{d}\mu(h) = \prod_{m=1}^H \calN(h_m; 0, C)\,\mathrm{d}h. \]
	on $(\RR^2)^H$. 
	Note that $\mu$ is a non-degenerate normal distribution. 
	Then we have
	\begin{align*}
	\cov[F_1(x_a)^2, F_1(x_b)^2]
	&= \operatorname{E}[F_1(x_a)^2 F_2(x_b)^2] - \operatorname{E}[F_1(x_a)^2] \operatorname{E}[F_1(x_b)^2] \\
	&= \int_{(\RR^2)^H}\int_{(\RR^2)^2} (y^a)^2 (y^b)^2 \calN(y; 0, K)\,\mathrm{d}y\,\mathrm{d}\mu(h) - \operatorname{E}[F_1(x_a)^2] \operatorname{E}[F_1(x_b)^2] \\
	&= \int_{(\RR^2)^H} (2k_{ab}^2+k_{aa}k_{bb})\,\mathrm{d}\mu(h) - \operatorname{E}[F_1(x_a)^2] \operatorname{E}[F_1(x_b)^2] \\
	&= 2\int_{(\RR^2)^H} k_{ab}^2\,\mathrm{d}\mu(h) + \int_{(\RR^2)^H} k_{aa}k_{bb}\,\mathrm{d}\mu(h) - \operatorname{E}[F_1(x_a)^2] \operatorname{E}[F_1(x_b)^2].
	\end{align*}
	Noting that $F_1(x_b)$ and $F_2(x_b)$ are identically distributed and recalling the proof of Prop.~\ref{prop-quad-cor}, we have
	\begin{align*}
	\cov[F_1(x_a)^2, F_1(x_b)^2]
	&= 2\int_{(\RR^2)^H} k_{ab}^2\,\mathrm{d}\mu(h) + \int_{(\RR^2)^H} k_{aa}k_{bb}\,\mathrm{d}\mu(h) - \operatorname{E}[F_1(x_a)^2] \operatorname{E}[F_2(x_b)^2] \\
	&= 2\int_{(\RR^2)^H} k_{ab}^2\,\mathrm{d}\mu(h) + \cov[F_1(x_a)^2, F_2(x_b)^2].
	\end{align*}
	Using again the fact that $F_1(x_b)$ and $F_2(x_b)$ are identically distributed, we have the correlation
	\begin{align*}
	q^{(D)}_{ab}
	&= \frac{\cov[F_1(x_a)^2, F_1(x_b)^2]}{\sqrt{\operatorname{V}[F_1(x_a)^2]\operatorname{V}[F_1(x_b)^2]}} \\
	&= 2\int_{(\RR^2)^H} \frac{k_{ab}^2}{\sqrt{\operatorname{V}[F_1(x_a)^2]\operatorname{V}[F_1(x_b)^2]}}\,\mathrm{d}\mu(h) + \frac{\cov[F_1(x_a)^2, F_2(x_b)^2]}{\sqrt{\operatorname{V}[F_1(x_a)^2]\operatorname{V}[F_1(x_b)^2]}} \\
	&= 2\int_{(\RR^2)^H} \frac{k_{ab}^2}{\sqrt{\operatorname{V}[F_1(x_a)^2]\operatorname{V}[F_1(x_b)^2]}}\,\mathrm{d}\mu(h) + q^{\times (D)}_{ab}.
	\end{align*}
	To make the dependence on the post-bottleneck depth $D$ more explicit, we write
	\[ q^{(D)}_{ab} = 2\int_{(\RR^2)^H} \frac{(k^{(D)}_{ab})^2}{\sqrt{\operatorname{V}[f^{(D)}_1(x_a)^2]\operatorname{V}[f^{(D)}_1(x_b)^2]}}\,\mathrm{d}\mu(h) + q^{\times (D)}_{ab}. \]
	By Eq.~\eqref{eq-prop-quad-cor-square-var}, this becomes
	\begin{align*}
	q^{(D)}_{ab}
	% &= 2\int_{(\RR^2)^H} \frac{(k^{(D)}_{ab})^2}{\sqrt{\left(\frac{15w_D^2(c_{aa})^2}{H} + 2\left(b_D + w_D c_{aa}\right)^2\right)\left(\frac{15w_D^2(c_{bb})^2}{H} + 2\left(b_D + w_D c_{bb}\right)^2\right)}}\,\mathrm{d}\mu(h) + q^{\times (D)}_{ab} \\
	%&= 2\int_{(\RR^2)^H} \frac{(k^{(D)}_{ab})^2}{\sqrt{\displaystyle\prod_{c=c_{aa},c_{bb}}\left(\frac{15w_D^2c^2}{H} + 2\left(b_D + w_D c\right)^2\right)}}\,\mathrm{d}\mu(h) + q^{\times (D)}_{ab} \\
	&= 2\int_{(\RR^2)^H} \frac{(k^{(D)}_{ab})^2}{w_D^2 c_{aa}c_{bb}\sqrt{\left(\frac{15}{H} + 2\left(\frac{r_D}{c_{aa}} + 1\right)^2\right)\left(\frac{15}{H} + 2\left(\frac{r_D}{c_{bb}} + 1\right)^2\right)}}\,\mathrm{d}\mu(h) + q^{\times (D)}_{ab}.
	\end{align*}
	The correlation at infinite depth is then
	\[ q^{(\infty)}_{ab} = \limD 2\int_{(\RR^2)^H} \frac{(k^{(D)}_{ab})^2}{w_D^2 c_{aa}c_{bb}\sqrt{\left(\frac{15}{H} + 2\left(\frac{r_D}{c_{aa}} + 1\right)^2\right)\left(\frac{15}{H} + 2\left(\frac{r_D}{c_{bb}} + 1\right)^2\right)}}\,\mathrm{d}\mu(h) + q^{\times (\infty)}_{ab}. \]
	The limit can be moved inside the integral. 
	To justify this, observe that the integrand (as a function of $h\in(\RR^2)^H$) can be expressed as
	\begin{align*}
	I_D(h)
	&:= \frac{(k^{(D)}_{ab})^2}{w_D^2 c_{aa}c_{bb}\sqrt{\left(\frac{15}{H} + 2\left(\frac{r_D}{c_{aa}} + 1\right)^2\right)\left(\frac{15}{H} + 2\left(\frac{r_D}{c_{bb}} + 1\right)^2\right)}} \\
	&\leq \frac{k^{(D)}_{aa}k^{(D)}_{bb}}{w_D^2 c_{aa}c_{bb}\sqrt{\left(\frac{15}{H} + 2\left(\frac{r_D}{c_{aa}} + 1\right)^2\right)\left(\frac{15}{H} + 2\left(\frac{r_D}{c_{bb}} + 1\right)^2\right)}} \\
	&= \frac{\left(b_D+\frac{w_D}{H}\sum_i \phi(h^a_i)^2\right)\left(b_D+\frac{w_D}{H}\sum_i \phi(h^b_i)^2\right)}{w_D^2 c_{aa}c_{bb}\sqrt{\left(\frac{15}{H} + 2\left(\frac{r_D}{c_{aa}} + 1\right)^2\right)\left(\frac{15}{H} + 2\left(\frac{r_D}{c_{bb}} + 1\right)^2\right)}} \\
	&\leq \frac{\left(b_D+\frac{w_D}{H}\sum_i \phi(\operatorname{max}(h))^2\right)\left(b_D+\frac{w_D}{H}\sum_i \phi(\operatorname{max}(h))^2\right)}{w_D^2 c_{aa}c_{bb}\sqrt{\left(\frac{15}{H} + 2\left(\frac{r_D}{c_{aa}} + 1\right)^2\right)\left(\frac{15}{H} + 2\left(\frac{r_D}{c_{bb}} + 1\right)^2\right)}} \\
	&= \frac{\left(b_D+w_D \phi(\operatorname{max}(h))^2\right)^2}{w_D^2 c_{aa}c_{bb}\sqrt{\left(\frac{15}{H} + 2\left(\frac{r_D}{c_{aa}} + 1\right)^2\right)\left(\frac{15}{H} + 2\left(\frac{r_D}{c_{bb}} + 1\right)^2\right)}} \\
	&= \frac{\left(r_D+\phi(\operatorname{max}(h))^2\right)^2}{c_{aa}c_{bb}\sqrt{\left(\frac{15}{H} + 2\left(\frac{r_D}{c_{aa}} + 1\right)^2\right)\left(\frac{15}{H} + 2\left(\frac{r_D}{c_{bb}} + 1\right)^2\right)}} \\
	% &\leq \frac{2\left(r_D^2+\phi(\operatorname{max}(h))^4\right)}{c_{aa}c_{bb}\sqrt{\left(\frac{15}{H} + 2\left(\frac{r_D}{c_{aa}} + 1\right)^2\right)\left(\frac{15}{H} + 2\left(\frac{r_D}{c_{bb}} + 1\right)^2\right)}} \\
	&\leq \frac{2r_D^2}{c_{aa}c_{bb}\sqrt{\left(\frac{15}{H} + 2\left(\frac{r_D}{c_{aa}} + 1\right)^2\right)\left(\frac{15}{H} + 2\left(\frac{r_D}{c_{bb}} + 1\right)^2\right)}} \\
	&+ \frac{2\phi(\operatorname{max}(h))^4}{c_{aa}c_{bb}\sqrt{\left(\frac{15}{H} + 2\left(\frac{r_D}{c_{aa}} + 1\right)^2\right)\left(\frac{15}{H} + 2\left(\frac{r_D}{c_{bb}} + 1\right)^2\right)}}.
	\end{align*}
	Recall that $r_D\rightarrow \frac{v_b}{v_w-1}$ if $v_w>1$ and $r_D\rightarrow\infty$ otherwise. 
	In either case, it is easy to verify that the first term and the denominator of the second term converge to non-negative numbers independent of $D$. 
	Therefore, there exist positive constants $A$ and $B$ such that
	\[ I_D(h) < A + B \phi(\operatorname{max}(h))^4 \mbox{ for sufficiently large } D. \]
	Note the right-hand side is integrable with respect to the non-degenerate Gaussian measure $\mu$ since it is a piecewise polynomial in $h$ (with finitely many pieces). 
	We can therefore use the Dominated Convergence Theorem. 
	Recall also that the NNGP kernel (post-bottleneck) degenerates to a constant-element kernel corresponding to a correlation matrix of $1$'s given any fixed input $h$ from the bottleneck layer. 
	Using the Dominated Convergence Theorem twice, we therefore have
	\begin{align*}
	q^{(\infty)}_{ab}
	% &= 2\int_{(\RR^2)^H} \limD \frac{(k^{(D)}_{ab})^2}{w_D^2 c_{aa}c_{bb}\sqrt{\left(\frac{15}{H} + 2\left(\frac{r_D}{c_{aa}} + 1\right)^2\right)\left(\frac{15}{H} + 2\left(\frac{r_D}{c_{bb}} + 1\right)^2\right)}}\,\mathrm{d}\mu(h) + q^{\times (\infty)}_{ab} \\
	&= 2\int_{(\RR^2)^H} \limD \frac{(k^{(D)}_{ab})^2}{w_D^2 c_{aa}c_{bb}\displaystyle\prod_{c=c_{aa},c_{bb}}\sqrt{\frac{15}{H} + 2\left(\frac{r_D}{c} + 1\right)^2}}\,\mathrm{d}\mu(h) + q^{\times (\infty)}_{ab} \\
	% &= 2\int_{(\RR^2)^H} \limD \frac{(k^{(D)}_{ab})^2}{k^{(D)}_{aa}k^{(D)}_{bb}} \frac{k^{(D)}_{aa}k^{(D)}_{bb}}{w_D^2 c_{aa}c_{bb}\sqrt{\left(\frac{15}{H} + 2\left(\frac{r_D}{c_{aa}} + 1\right)^2\right)\left(\frac{15}{H} + 2\left(\frac{r_D}{c_{bb}} + 1\right)^2\right)}}\,\mathrm{d}\mu(h) + q^{\times (\infty)}_{ab} \\
	&= 2\int_{(\RR^2)^H} \limD \frac{(k^{(D)}_{ab})^2}{k^{(D)}_{aa}k^{(D)}_{bb}} \frac{k^{(D)}_{aa}k^{(D)}_{bb}}{w_D^2 c_{aa}c_{bb}\displaystyle\prod_{c=c_{aa},c_{bb}}\sqrt{\frac{15}{H} + 2\left(\frac{r_D}{c} + 1\right)^2}}\,\mathrm{d}\mu(h) + q^{\times (\infty)}_{ab} \\
	% &= 2\int_{(\RR^2)^H} \limD \frac{k^{(D)}_{aa}k^{(D)}_{bb}}{w_D^2 c_{aa}c_{bb}\sqrt{\left(\frac{15}{H} + 2\left(\frac{r_D}{c_{aa}} + 1\right)^2\right)\left(\frac{15}{H} + 2\left(\frac{r_D}{c_{bb}} + 1\right)^2\right)}}\,\mathrm{d}\mu(h) + q^{\times (\infty)}_{ab} \\
	&= 2\int_{(\RR^2)^H} \limD \frac{k^{(D)}_{aa}k^{(D)}_{bb}}{w_D^2 c_{aa}c_{bb}\displaystyle\prod_{c=c_{aa},c_{bb}}\sqrt{\frac{15}{H} + 2\left(\frac{r_D}{c} + 1\right)^2}}\,\mathrm{d}\mu(h) + q^{\times (\infty)}_{ab} \\
	% &= \limD \frac{2}{w_D^2 c_{aa}c_{bb}\sqrt{\left(\frac{15}{H} + 2\left(\frac{r_D}{c_{aa}} + 1\right)^2\right)\left(\frac{15}{H} + 2\left(\frac{r_D}{c_{bb}} + 1\right)^2\right)}} \int_{(\RR^2)^H} k^{(D)}_{aa}k^{(D)}_{bb}\,\mathrm{d}\mu(h) + q^{\times (\infty)}_{ab} \\
	&= \limD \frac{2}{w_D^2 c_{aa}c_{bb}\displaystyle\prod_{c=c_{aa},c_{bb}}\sqrt{\frac{15}{H} + 2\left(\frac{r_D}{c} + 1\right)^2}} \int_{(\RR^2)^H} k^{(D)}_{aa}k^{(D)}_{bb}\,\mathrm{d}\mu(h) + q^{\times (\infty)}_{ab} \\
	% &= \limD \frac{2}{w_D^2 c_{aa}c_{bb}\sqrt{\left(\frac{15}{H} + 2\left(\frac{r_D}{c_{aa}} + 1\right)^2\right)\left(\frac{15}{H} + 2\left(\frac{r_D}{c_{bb}} + 1\right)^2\right)}} \left((b_D+w_D c_{aa})(b_D+w_D c_{bb}) + \cov[f_1^{(D)}(x_a)^2, f_2^{(D)}(x_b)^2]\right) + q^{\times (\infty)}_{ab} \\
	&= \limD \frac{2}{w_D^2 c_{aa}c_{bb}\displaystyle\prod_{c=c_{aa},c_{bb}}\sqrt{\frac{15}{H} + 2\left(\frac{r_D}{c} + 1\right)^2}} \\
	& \cdot\left((b_D+w_D c_{aa})(b_D+w_D c_{bb}) + \cov[f_1^{(D)}(x_a)^2, f_2^{(D)}(x_b)^2]\right) + q^{\times (\infty)}_{ab} \\
	% &= \limD \frac{2(b_D+w_D c_{aa})(b_D+w_D c_{bb})}{w_D^2 c_{aa}c_{bb}\sqrt{\left(\frac{15}{H} + 2\left(\frac{r_D}{c_{aa}} + 1\right)^2\right)\left(\frac{15}{H} + 2\left(\frac{r_D}{c_{bb}} + 1\right)^2\right)}} + 3q^{\times (\infty)}_{ab} \\
	&= \limD \frac{2(b_D+w_D c_{aa})(b_D+w_D c_{bb})}{w_D^2 c_{aa}c_{bb}\displaystyle\prod_{c=c_{aa},c_{bb}}\sqrt{\frac{15}{H} + 2\left(\frac{r_D}{c} + 1\right)^2}} + 3q^{\times (\infty)}_{ab} \\
	% &= \limD \frac{\left(\frac{r_D}{c_{aa}}+1\right)\left(\frac{r_D}{c_{bb}}\right)}{\sqrt{\left(\frac{15}{2H} + \left(\frac{r_D}{c_{aa}} + 1\right)^2\right)\left(\frac{15}{2H} + \left(\frac{r_D}{c_{bb}} + 1\right)^2\right)}} + 3q^{\times (\infty)}_{ab}.
	&= \limD \frac{\left(\frac{r_D}{c_{aa}}+1\right)\left(\frac{r_D}{c_{bb}}\right)}{\displaystyle\prod_{c=c_{aa},c_{bb}}\sqrt{\frac{15}{2H} + \left(\frac{r_D}{c} + 1\right)^2}} + 3q^{\times (\infty)}_{ab}.
	\end{align*}
	Evaluating the limit by recalling the limit of $r_D$, we obtain Eq.~\eqref{eq-prop-q} as desired.

	For part~\ref{prop-q:b}, given $(Q^{(\infty)}, \operatorname{diag}(G))$, we will show that we recover $G$. 
	We can obtain $\operatorname{diag}(C)$ from $\operatorname{diag}(G)$ by applying the NNGP kernel propagation defined in Eq.~\eqref{eq-Kmu-relu}. 
	Given $\operatorname{diag}(C)$, we can solve for $Q^{\times(\infty)}$ using Eq.~\eqref{eq-prop-q} for the case $v_w>1$ in Prop.~\ref{prop-q}. 
	We can then obtain $G$ by Prop.~\ref{prop-quad-cor-lim}~\ref{prop-quad-cor-lim:b}.
\end{proof}

\subsection{Other nonlinearities}
\label{appendix-section-other-nonlin}

The following is the proof for the proposition linking the sinusoidal nonlinearity in Eq.~\eqref{eq-sinusoidal} to the RBF kernel.

\begin{proof}[Proof of Prop.~\ref{prop-rbf-nngp}]
First we prove part~\ref{prop-rbf-nngp:a}. 
The case $\mu=0$ is trivial. 
So, consider $\mu \geq 1$. 
Define the function $F_{\phi}:\RR^3\mapsto\RR$ by
\begin{equation} \label{eq-rbf-nngp-Fphi}
F_{\phi}(a, b, c) = \int_{\RR^2} \phi(z_1)\phi(z_2) \calN\left(z; 0, \lmat a & b\\ b & c\rmat\right)\,\mathrm{d}z,
\end{equation}
where $\phi$ is the sinusoidal nonlinearity in Eq.~\eqref{eq-sinusoidal} and
\[ z = \lmat z_1\\ z_2\rmat. \]
The NNGP kernel recursion can then be written as
\[ K^{(\mu+1)}(x, x^\prime) = v_b + v_w F_{\phi}[K^{(\mu)}(x, x), K^{(\mu)}(x, x^\prime), K^{(\mu)}(x^\prime, x^\prime)]. \]
All we need to show is that
\begin{equation} \label{eq-rbf-nngp-Fphi_abc}
F_{\phi}(a, b, c) = e^{-\frac{1}{2}(a+c-2b)}.
\end{equation}
Let $X$ be a $2\times 2$ matrix with columns $x_1,x_2\in\RR^2$ such that
\begin{equation} \label{eq-rbf-nngp-XTX}
X^\top X = \lmat a & b\\ b & c\rmat.
\end{equation}
Such a matrix $X$ exists since the matrix on the right side is symmetric positive semidefinite. 
Performing the change of variables $z = X^\top w$, Eq.~\eqref{eq-rbf-nngp-Fphi} becomes
\begin{align*}
F_{phi}(a, b, c)
&= \int_{\RR^2} \phi(x_1^\top w)\phi(x_2^\top w) \calN(w; 0, I)\,\mathrm{d}w \\
&= \frac{1}{2\pi}\int_{\RR^2} \phi(w\cdot x_1)\phi(w\cdot x_2) e^{-\frac{1}{2}\Vert w\Vert^2}\,\mathrm{d}w.
\end{align*}
The product of activations in the integrand can be rewritten as
\begin{align*}
\phi(w\cdot x_1)\phi(w\cdot x_2)
&= [\cos(w\cdot x_1)+\sin(w\cdot x_1)] [\cos(w\cdot x_2)+\sin(w\cdot x_2)] \\
&= \cos(w\cdot x_1)\cos(w\cdot x_2) + \sin(w\cdot x_1)\sin(w\cdot x_2) \\
&+ \cos(w\cdot x_1)\sin(w\cdot x_2) + \sin(w\cdot x_1)\cos(w\cdot x_2) \\
&= \cos[w\cdot (x_1-x_2)] + \sin[w\cdot (x_1+x_2)].
\end{align*}
We therefore have
\begin{align*}
F_{\phi}(a, b, c)
&= \frac{1}{2\pi}\int_{\RR^2} \cos[w\cdot (x_1-x_2)] e^{-\frac{1}{2}\Vert w\Vert^2}\,\mathrm{d}w \\
&+ \frac{1}{2\pi}\int_{\RR^2} \sin[w\cdot (x_1+x_2)] e^{-\frac{1}{2}\Vert w\Vert^2}\,\mathrm{d}w.
\end{align*}
The integrand of the second integral on the right side is odd in $w$ for all $x_1$ and $x_2$, 
and thus this integral is zero. 
We can therefore replace $x_1+x_2$ with $x_1-x_2$ in the second integral and multiply the integral by the imaginary unit $i$ without changing its value:
\begin{align*}
F_{\phi}(a, b, c)
&= \frac{1}{2\pi}\int_{\RR^2} \cos[w\cdot (x_1-x_2)] e^{-\frac{1}{2}\Vert w\Vert^2}\,\mathrm{d}w \\
&+ i \frac{1}{2\pi}\int_{\RR^2} \sin[w\cdot (x_1-x_2)] e^{-\frac{1}{2}\Vert w\Vert^2}\,\mathrm{d}w \\
&= \frac{1}{2\pi}\int_{\RR^2} \left(\cos[w\cdot (x_1-x_2)] + i\sin[w\cdot (x_1-x_2)]\right) e^{-\frac{1}{2}\Vert w\Vert^2}\,\mathrm{d}w \\
&= \frac{1}{2\pi}\int_{\RR^2} e^{iw\cdot (x_1-x_2)} e^{-\frac{1}{2}\Vert w\Vert^2}\,\mathrm{d}w \\
&= e^{-\frac{1}{2}\Vert x_1-x_2\Vert^2},
\end{align*}
where the last line holds because the Gaussian is an eigenfunction of the Fourier transform. 
Using Eq.~\eqref{eq-rbf-nngp-XTX}, this becomes
\begin{align*}
F_{\phi}(a, b, c)
&= e^{-\frac{1}{2}(\Vert x_1\Vert^2+\Vert x_2\Vert^2-2x_1\cdot x_2)} \\
&= e^{-\frac{1}{2}(a+c-2b)},
\end{align*}
giving us Eq.~\eqref{eq-rbf-nngp-Fphi_abc} as desired.

For part~\ref{prop-rbf-nngp:b}, consider two distinct inputs $x,x^\prime\in\RR^M$ and define
\[ v^{(\mu)} = K^{(\mu)}(x, x^\prime) 
\mbox{ and }
c^{(\mu)} = \frac{K^{(\mu)}(x, x^\prime)}{\sqrt{K^{(\mu)}(x, x) K^{(\mu)}(x^\prime, x^\prime)}}. \]
All we need to show is that $(v^{(\mu)}, c^{(\mu)})$ has a globally attractive fixed point $(v_*, c_*)$ of the form given in the statement of the proposition. 
The dynamics of $v^{(\mu)}$ is given by $v^{(\mu+1)} = f_v(v^{(\mu)})$ where $f_v:(0, \infty)\mapsto (0, \infty)$ is given by
\begin{align*}
f_v(v)
&= v_b + v_w F_{\phi}(v, v, v) \\
&= v_b + v_w.
\end{align*}
It is thus trivial that the global fixed point of $f_v$ is $v_*(v_b, v_w) = v_b + v_w$. 
The dynamics of $c^{(\mu)}$ is given by $c^{(\mu+1)} = f_c(c^{(\mu)})$ where $f_c:[-1, 1]\mapsto [-1, 1]$ is given by
\[ f_c(c) = v_b + v_w F_{\phi}(v_*, v_* c, v_*), \]
where we approximate $v^{(\mu)}$ with $v_*$; 
this approximation becomes exact in the deep limit. 
Substituting in $v_*=v_b+v_w$, we get
\[ f_c(c) = \frac{v_b}{v_b+v_w} + \frac{v_w}{v_b+v_w} e^{(v_b+v_w)(c-1)}. \]
We make use of the fact that if a function $f:\RR\mapsto\RR$ has strictly positive first and second derivatives, then $f$ has either no fixed points, one fixed point that is neither stable nor unstable, or two fixed points $c_1 < c_2$ where $c_1$ is stable with basin of attraction $(-\infty, c_2)$ and $c_2$ is unstable. 
Observe that $f_c(1) = 1$ and $f^\prime_c(1)=v_w$. 
If $v_w < 1$, then $c=1$ is a stable fixed point and is thus globally attractive on $[-1, 1]$, 
thus establishing $c_*(v_b, v_w) = 1$ for $v_w < 1$. 
On the other hand, if $v_w>1$, then $c=1$ is an unstable fixed point. 
Since $f_c(0) > 0$ and $f_c(1)=1$, then by the Intermediate Value Theorem, $f_c$ has a stable fixed point $c^\prime\in (0, 1)$ that is globally attractive on $[-1, 1)$. 
Given the initial condition $c^{(0)} = \frac{x^\top x^\prime}{\Vert x\Vert \Vert x^\prime\Vert} < 1$ (since $x\neq x^\prime$) and that $c=1$ is an unstable fixed point, 
then we must have $c_*(v_b, v_w) < 1$ for $v_w>1$ and thus in particular $c_*(v_b, v_w)=c^\prime$, 
concluding the proof.
\end{proof}

\begin{proof}[Proof of Prop.~\ref{prop-rbf-nngp-bottleneck}]
Let $H$ be the bottleneck width, 
and let $h_1,h_2\in\RR^H$ be the bottleneck preactivations---i.e., the outputs of the pre-bottleneck NNGP component---given network inputs $x_1,x_2\in\RR^M$. 
Let $\mu$ be the (Gaussian) probability measure associated to $(h_1, h_2)$. 
Let $z_1,z_2\in\RR^H$ be the corresponding bottleneck activations that are fed into the post-bottleneck NNGP component, 
where $z_a = \frac{1}{\sqrt{H}}\phi(h_a)$ for $a=1,2$. 
Finally, let $K^{(D)}:\RR^H\times\RR^H\mapsto\RR$ be the kernel of the post-bottleneck NNGP component assuming post-bottleneck depth $D$. 
Then the PDF of the bottleneck NNGP outputs is given by
\[ p^{(D)}(y) = \int_{(\RR^H)^2} \calN(y; 0, K^{(D)}(Z, Z)+v_n I)\,\mathrm{d}\mu(h_1, h_2), \]
where $K^{(D)}(Z, Z)$ is a $2\times 2$ matrix with entries $K^{(D)}(z_a, z_b)$ for $a,b=1,2$. 
Using the fact that for any $2\times 2$ positive semidefinite matrix $A$ with eigenvalues $\lambda_1,\lambda_2\geq 0$,
\[ \operatorname{det}(A+v_n I) = (\lambda_1+v_n)(\lambda_2+v_n) \geq v_n^2, \]
we have the bound
\begin{align*}
\calN(y; 0, K^{(D)}(Z, Z)+v_n I)
&\leq \frac{1}{2\pi}\operatorname{det}[K^{(D)}(Z, Z)+v_n I]^{-\frac{1}{2}} \\
&\leq \frac{1}{2\pi} (v_n^2)^{-\frac{1}{2}} \\
&= \frac{v_n}{2\pi}.
\end{align*}
Since the bound is clearly an integrable function with respect to $\mu$ and since the bound holds for all $D$, then we may apply the Bounded Convergence Theorem. 
By the continuity of the matrix determinant and inversion operations, the PDF converges to
\begin{align}
p^{(\infty)}(y)
&= \limD p^{(D)}(y) \nonumber \\
&= \int_{(\RR^H)^2} \calN(y; 0, \limD K^{(D)}(Z, Z)+v_n I)\,\mathrm{d}\mu(h_1, h_2) \nonumber \\
&= \int_{(\RR^H)^2} \calN(y; 0, K^{(\infty)}(Z, Z)+v_n I)\,\mathrm{d}\mu(h_1, h_2). \label{eq-rbf-nngp-pinfty-temp}
\end{align}
Define the set
\begin{align*}
S
&= \left\{(h_1, h_2)\in (\RR^H)^2: z_1=z_2\right\} \\
&= \left\{(h_1, h_2)\in (\RR^H)^2: \frac{1}{\sqrt{H}}\phi(h_1) = \frac{1}{\sqrt{H}}\phi(h_2)\right\} \\
&= \{(h_1, h_2)\in (\RR^H)^2: \phi(h_1) = \phi(h_2)\} \\
&= \{(h_1, h_2)\in (\RR^H)^2: \exists n\in\ZZ^H\mid h_1 = h_2+2n\pi\} \\
&= \bigcup_{n\in\ZZ^H} \{(h, h+2n\pi): h\in\RR^H\}.
\end{align*}
We see that $S$ is a countable disjoint union of $H$-dimensional planes embedded in a $2h$-dimensional space. 
Since the network inputs $x_1$ and $x_2$ are distinct, then $\mu$ is a non-degenerate Gaussian distribution on $(\RR^H)^2$ so that $\mu(S) = 0$. 
We can therefore remove $S$ from the region of integration in Eq.~\eqref{eq-rbf-nngp-pinfty-temp}, 
so that $z_1$ and $z_2$ are distinct inputs into $K^{(\infty)}$ under the integral. 
We thus evaluate the covariance of the integrand in Eq.~\eqref{eq-rbf-nngp-pinfty-temp} using Prop.~\ref{prop-rbf-nngp}~\ref{prop-rbf-nngp:b} and Eq.~\eqref{eq-rbf-nngp-Kinfty} and get
\begin{align*}
p^{(\infty)}(y)
&= \int_{(\RR^H)^2\setminus S} \calN(y; 0, K^{(\infty)}(Z, Z)+v_n I)\,\mathrm{d}\mu(h_1, h_2) \\
&= \int_{(\RR^H)^2\setminus S} \calN\left(y; 0, v_*(v_b, v_w) \lmat 
1 & c_*(v_b, v_w) \\
c_*(v_b, v_w) & 1
\rmat + v_n I\right)\,\mathrm{d}\mu(h_1, h_2) \\
&= \calN\left(y; 0, v_*(v_b, v_w) \lmat 
1 & c_*(v_b, v_w) \\
c_*(v_b, v_w) & 1
\rmat + v_n I\right),
\end{align*}
which according to Prop.~\ref{prop-rbf-nngp}~\ref{prop-rbf-nngp:b} is precisely the deep limit with no bottleneck given two distinct inputs, 
thus establishing Eq.~\eqref{eq-rbf-nngp-pinfty}.
\end{proof}

\vskip 0.2in
\bibliography{main}

\end{document}